%% file: ms.tex
\documentclass{article}

\PassOptionsToPackage{numbers}{natbib}

\usepackage[final]{neurips_2020}

\usepackage[utf8]{inputenc} 
\usepackage[T1]{fontenc}    
\usepackage{hyperref}       
\usepackage{url}            
\usepackage{booktabs}       
\usepackage{amsfonts}       
\usepackage{nicefrac}       
\usepackage{microtype}      

\hypersetup{%
    pdfborder = {0 0 0},
    colorlinks,
    citecolor=blue,
    linkcolor=blue,
}

\usepackage{enumitem}
\usepackage{wrapfig}
\usepackage{amsthm}

\usepackage{setspace}

\usepackage{bmpsize}
\usepackage[pdftex]{graphicx}

\usepackage{amsthm}
\usepackage{wrapfig}

\usepackage{amssymb}

\usepackage{color}
\usepackage{mathtools}

\usepackage{multirow}

\usepackage{algorithm}
\usepackage{algorithmic}

\usepackage{multirow}

\usepackage{thmtools} 
\usepackage{thm-restate}

\newcommand{\E}{\mathop{\mathbb{E}}} 
\newcommand{\Var}{\mathop{\mathbb{V}}} 
\newcommand{\Cov}{\mathop{\mathbb{C}}} 

\newcommand{\jact}{\frac{d\,\mathcal{T}_w(\repsilon)}{d\,w}}
\newcommand{\Tr}{\mathcal{T}_w(\repsilon)}
\newcommand{\T}{\mathcal{T}_w(\epsilon)}
\newcommand{\hess}{\nabla^2 f(\mu)}

\DeclareMathAlphabet{\mathbfsf}{\encodingdefault}{\sfdefault}{bx}{n}
\newcommand{\upgreektemplate}[2]{#2{
\renewcommand{\alpha}{\upalpha}
\renewcommand{\beta}{\upbeta}
\renewcommand{\theta}{\uptheta}
\renewcommand{\gamma}{\upgamma}
\renewcommand{\lambda}{\uplambda}
\renewcommand{\xi}{\upxi}
\renewcommand{\epsilon}{\upepsilon}
\renewcommand{\delta}{\updelta}
\renewcommand{\phi}{\upphi}
\renewcommand{\zeta}{\upzeta}
\renewcommand{\Lambda}{\Uplambda}
\renewcommand{\Gamma}{\Upgamma}
\renewcommand{\Delta}{\Updelta}
\renewcommand{\Theta}{\Uptheta}
#1
}}
\newcommand{\upgreek}[1]{\upgreektemplate{#1}{\mathsf}}

\usepackage{upgreek}

\newcommand{\rr}[1]{{\upgreek #1}}
\newcommand{\repsilon}{\rr{\epsilon}}
\newcommand{\rz}{\rr{z}}

\newtheorem{theorem}{Theorem}[section]

\newtheorem{lemma}[theorem]{Lemma}

\title{Approximation Based Variance Reduction for Reparameterization Gradients}

\author{
  Tomas Geffner \\
  College of Information and Computer Science\\
  University of Massachusetts, Amherst\\
  \texttt{tgeffner@cs.umass.edu} \\
  \And
  Justin Domke \\
  College of Information and Computer Science\\
  University of Massachusetts, Amherst\\
  \texttt{domke@cs.umass.edu} \\
}

\begin{document}

\maketitle

\begin{abstract}
\input{./sections/abstract.tex}
\end{abstract}

\input{./sections/introduction.tex}

\input{./sections/preliminaries.tex}

\input{./sections/newcv.tex}


\input{./sections/experiments_50.tex}

\newpage
\clearpage

\input{./sections/impact.tex}

\bibliography{control.bib}
\bibliographystyle{apalike}
\clearpage
\newpage

\input{./sections/appendix.tex}

\end{document}

%% file: sections/abstract.tex

Flexible variational distributions improve variational inference but are harder to optimize. In this work we present a control variate that is applicable for any reparameterizable distribution with known mean and covariance matrix, e.g. Gaussians with any covariance structure. The control variate is based on a quadratic approximation of the model, and its parameters are set using a double-descent scheme by minimizing the gradient estimator's variance. We empirically show that this control variate leads to large improvements in gradient variance and optimization convergence for inference with non-factorized variational distributions.

%% file: sections/introduction.tex

\section{Introduction}

This paper concerns estimating the gradient of $\E_{q_w(\rz)} f(\rz)$ with respect to $w$. This is a ubiquitous problem in machine learning, needed to perform stochastic optimization in variational inference (VI), reinforcement learning, and experimental design \cite{mohamed2019monte, jordan1999introduction, sutton2018reinforcement, chaloner1995bayesian}. A popular technique is the ``reparameterization trick'' \cite{pflug2012optimization, vaes_welling, glasserman2013monte}. Here, one defines a mapping $\mathcal{T}_w$ that transforms some base density $q_0$ into $q_w$. Then, the gradient is estimated by drawing $\epsilon \sim q_0$ and evaluating $\nabla_w f(\mathcal{T}_w(\repsilon))$.

In any application using stochastic gradients, variance is a concern. Several variance reduction methods exist, with control variates representing a popular alternative \cite{mcbook}. A control variate is a random variable with expectation zero, which can be added to an estimator to cancel noise and decrease variance. Previous work has shown that control variates can significantly reduce the variance of reparameterization gradients, and thereby improve optimization performance \cite{cvs, traylorreducevariance_adam, stickingthelanding}.


Miller et al. \cite{traylorreducevariance_adam} proposed a Taylor-expansion based control variate for the case where $q_w$ is a fully-factorized Gaussian parameterized by its mean and scale. Their method works well for the gradient with respect to the mean parameters. However, for the scale parameters, computational issues force the use of further approximations. In a new analysis (Sec. \ref{sec:miller4}) we observe that this amounts to using a {\em constant} Taylor approximation (i.e. an approximation of order zero). As a consequence, for the scale parameters, the control variate has little effect. Still, the approach is very helpful with fully-factorized Gaussians, because in this case most variance is contributed by gradient with respect to the mean parameters.




The situation is different for non fully-factorized distributions: Often, most of the variance is contributed by the gradient with respect to the {\em scale} parameters. This renders Taylor-based control variates practically useless. Indeed, empirical results in Section \ref{sec:results} show that, with diagonal plus low rank Gaussians or Gaussians with arbitrary dense covariances, Taylor-based control variates yield almost no benefit over the no control variates baseline. (We generalize the Taylor approach to full-rank and diagonal plus low-rank distributions in Appendix \ref{app:millerdrawbacks}.)

For VI, fully factorized variational distributions are typically much less accurate than those representing interdependence \cite{diagLR, householderflows}. Thus, we seek a control variate that can aid the use of more powerful distributions, such as Gaussians with any covariance structure (full-rank, factorized as diagonal plus low rank \cite{diagLR}, Householder flows \cite{householderflows}), Student-t, and location-scale and elliptical families. This paper introduces such a control variate.

Our proposed method can be described in two steps. First, given any quadratic function $\hat{f}$ that approximates $f$, we define the control variate as $\E [\nabla_w \hat f(\mathcal{T}_w(\repsilon))] - \nabla_w \hat f(\mathcal{T}_w(\repsilon))$. We show that this control variate is tractable for any distribution with known mean and covariance. Intuitively, the more accurately $\hat{f}$ approximates $f$, the more this will decrease the variance of the original reparameterization estimator $\nabla_w f(\mathcal{T}_w(\repsilon))$. Second, we fit the parameters of $\hat{f}$ through a ``double descent'' procedure aimed at reducing the estimator's variance. 


We empirically show that the use of our control variate leads to reductions in variance several orders of magnitude larger than the state of the art method when diagonal plus low rank or full-rank Gaussians are used as variational distributions. Optimization speed and reliability is greatly improved as a consequence.

%% file: sections/preliminaries.tex

\section{Preliminaries}


\textbf{Stochastic Gradient Variational Inference (SGVI).} Take a model $p(x, z)$, where $x$ is observed data and $z$ latent variables. The posterior $p(z|x)$ is often intractable. VI finds the parameters $w$ to approximate the target $p(z|x)$ with the simpler distribution $q_w(z)$ \cite{jordan1999introduction, jaakkola2000bayesian, blei2017variational, zhang2017advances}. It does this by maximizing the "evidence lower bound"
\begin{equation}
\mbox{ELBO}(w) = E_{q_w(\rz)} \log \frac{p(x,\rz)}{q_w(\rz)},
\end{equation}
which is equivalent to minimizing the KL divergence from the approximating distribution $q_w(z)$ to the posterior $p(z|x)$. Using $f(z) = \log p(x, z)$ and letting $\mathcal{H}(w) = -\E_{q_w(z)} \log q_w(z)$ denote the entropy, we can express the ELBO's gradient as
\begin{equation}
\nabla_w \mathrm{ELBO}(w) = \nabla_w \E_{q_w(\rz)} f(\rz) + \nabla_w \mathcal{H}(w).\label{eq:elbo}
\end{equation}

SGVI's idea is that, while the first term from Eq.~\ref{eq:elbo} typically has no closed-form, there are many unbiased estimators that can be used with stochastic optimization algorithms to maximize the ELBO \cite{neuralVI_minh, viasstochastic_jordan, blackbox_blei, doublystochastic_titsias, VIforMCobjectives_mnih, stickingthelanding, generalreparam_blei, overdispersed_blei}. (We assume that the entropy term can be computed in closed form. If it cannot, one can ``absorb'' $\log q_w$ into $f$ and estimate its gradient alongside $f$.) These gradient estimators are usually based on the score function method \cite{williams_reinforce} or the reparameterization trick \cite{vaes_welling, doublystochastic_titsias, rezende2014stochastic}. Since the latter usually provides lower-variance gradients in practice, it is the method of choice whenever applicable. It requires a fixed distribution $q_0(\epsilon)$, and a transformation $\mathcal{T}_w(\epsilon)$ such that if $\repsilon \sim q_0(\epsilon)$, then $\mathcal{T}_w(\repsilon) \sim q_w(z)$. Then, an unbiased estimator for the first term in Eq. \ref{eq:elbo} is given by drawing $\epsilon \sim q_0(\epsilon)$ and evaluating
\begin{equation}
g(w, \epsilon) = \nabla_w f(\mathcal{T}_w(\epsilon)). \label{eq:gbase}
\end{equation}

\textbf{Control Variates.} A control variate is a zero-mean random variable used to reduce the variance of another random variable \cite{mcbook}. Control variates are widely used in SGVI to reduce a gradient estimator's variance \cite{traylorreducevariance_adam, neuralVI_minh, REBAR, backpropvoid, blackbox_blei, cvs, boustati2020amortized}. 
Let $g(w,\epsilon)$ define the base gradient estimator, using random variables $\epsilon$, and let the function $c(w, \epsilon)$ define the control variate, whose expectation over $\epsilon$ is zero. Then, for any scalar $\gamma$ we can get an unbiased gradient estimator as
\begin{equation}
g_\mathrm{cv}(w, \epsilon) = g(w, \epsilon) + \gamma c(w, \epsilon). \label{eq:gcv}
\end{equation}
The hope is that $c$ approximates and cancels the error in the gradient estimator $g$. It can be shown that the optimal weight is\footnote{Since $g$ and $c$ are vectors, the expressions for $\gamma$ and $\Var{g_\mathrm{cv}}$ should be interpreted using $\Var{X} = \E\Vert X \Vert^2 - \Vert\E X \Vert^2$, $\Cov[X,Y]=\E[(X-\E X)^\top (Y- \E Y)]$, and $\mathrm{Corr}[X,Y]=\mathrm{Cov}[X,Y]/\sqrt{\Var[X]\Var[Y]}$} $\gamma = -\Cov[c,g]/\Var[c]$, which results in a variance of $\Var[g_\mathrm{cv}] = \Var[g]\left(1 - \mathrm{Corr}[c,g]^2\right)$. Thus, a good control variate will have high correlation with the gradient estimator (while still being zero mean). In the extreme case that $c = \E[g]- g$, variance would be reduced to zero. In practice, $\gamma$ must be estimated. This can be done approximately using empirical estimates of $\E[c^\top g]$ and $\E[c^\top c]$ from recent evaluations \cite{cvs}.

%% file: sections/newcv.tex
\section{New Control Variate} \label{sec:newcv}


This section presents our control variate. The goal is to estimate the gradient $\nabla_w \E_{q_w(\rz)} f(\rz)$ with low variance. The core idea behind our method is simple: if $f$ is replaced with a simpler function $\hat f$, a closed-form for $\nabla_w \E_{q_w(\rz)} \hat f(\rz)$ may be available. Then, the control variate is defined as the difference between the term $\nabla_w \E_{q_w(\rz)} \hat f(\rz)$ computed exactly and estimated using reparameterization. 
Intuitively, if the approximation $\hat f$ is good, this control variate will yield large reductions in variance.

We use a quadratic function $\hat{f}$ as our approximation (Sec. \ref{sec:validity}). The resulting control variate is tractable as long as the mean and covariance of $q_w$ are known (Sec. \ref{sec:challenge1}). While this is valid for any quadratic function $\hat{f}$, the effectiveness of the control variate depends on the approximation's quality. We propose to find the parameters of $\hat f$ by minimizing the final gradient estimator's variance $\Var[g+c]$ or a proxy to it (Sec. \ref{sec:challenge2}). We do this via a double-descent scheme to simultaneously optimize the parameters of $\hat f$ alongside the parameters of $q_w$ (Sec. \ref{sec:finalalg}).

\subsection{Definition, Validity, and Motivation} \label{sec:validity}

Given a function $\hat f_v$ that approximates $f$, we define the control variate as
\begin{equation}
c_v(w, \epsilon) = \nabla_w \E_{q_w(\rz)} \left[ \hat f_v(\rz) \right] - \nabla_w \hat f_v(\mathcal{T}_w(\epsilon)). \label{eq:cv}
\end{equation}
Since the second term is an unbiased estimator of the first one, $c_v(w, \epsilon)$ has expectation zero and thus represents a valid control variate. To understand the motivation behind this control variate consider the final gradient estimator, 
\begin{equation}
    g_\mathrm{cv}(w, \epsilon) = g(w,\epsilon) + \gamma c_v(w,\epsilon) = \underbrace{\gamma \nabla_w \E_{q_w(\rz)} \left[ \hat f_v(\rz) \right]}_{\text{ deterministic term}} + \underbrace{\vphantom{\E_{q_w(\rz)} \left[ \hat f_v(\rz) \right]}\nabla_w \left( f(\mathcal{T}_w(\epsilon)) - \gamma \hat f_v(\mathcal{T}_w(\epsilon)) \right)}_{\text{stochastic term}}.
\end{equation}
Intuitively, making $\hat f_v$ a better approximation of $f$ will tend to make the stochastic term smaller, thus reducing the estimator's variance.  We propose to set the approximating function to be a quadratic parameterized by $v$ and $z_0$,
\begin{equation}
\hat{f}_v(z) = b_v^\top (z - z_0) + \frac{1}{2} (z - z_0)^\top B_v (z - z_0), \label{eq:fhat}
\end{equation}
where $b_v$ and $B_v$ are a vector and a square matrix parameterized by $v$, and $z_0$ is a vector. (We avoid including an additive constant in the quadratic since it would not affect the gradient.)

\subsection{Tractability of the Control Variate} \label{sec:challenge1}

We now consider computational issues associated with the control variate from Eq. \ref{eq:cv}. Our first result is that, given $b_v, B_v$ and $z_0$, the control variate is tractable for any distribution with known mean and covariance. We begin by giving a closed-form for the expectation in Eq. \ref{eq:cv} (proven in Appendix \ref{app:proof}).

\begin{lemma} \label{prop:closedE}
Let $\hat f_v$ be defined as in Eq. \ref{eq:fhat}. If $q_w$ has mean $\mu_w$ and covariance $\Sigma_w$, then
\begin{equation}
\E_{q_w(\rz)} \hat f_v(\rz) = b_v^\top(\mu_w - z_0) + \frac{1}{2} \mathrm{tr}(B_v \Sigma_w)
+ \frac{1}{2}\big(\mu_w^\top B_v \mu_w - z_0^\top B_v \mu_w - \mu_w^\top B_v z_0 + z_0^\top B_v z_0 \big). \label{eq:closedE}
\end{equation}
\end{lemma}
If we substitute this result into Eq. \ref{eq:cv}, we can easily use automatic differentiation tools to compute the gradient with respect to $w$, and thus compute the control variate. Therefore, our control variate can be easily used for any reparameterizable distribution $q_w$ with known mean and covariance matrix. These include fully-factorized Gaussians, Gaussians with arbitrary full-rank covariance, Gaussians with structured covariance (e.g. diagonal plus low rank \cite{diagLR}, Householder flows \cite{householderflows}), Student-t distributions, and, more generally, distributions in a location scale family or elliptical family.

\textbf{Computational cost.} The cost of computing the control variate depends on the cost of computing matrix-vector products (with matrix $B_v$) and the trace of $B_v \Sigma_w$ (see Eqs.~\ref{eq:fhat} and \ref{eq:closedE}). These costs depend on the structure of $B_v$ and $\Sigma_w$. We consider the case where $B_v$ and $\Sigma_w$ are parameterized as diagonal plus low rank matrices, with ranks $r_v$ and $r_w$, respectively. Then, computing the control variate has cost $\mathcal{O}(d \, (1 + r_v) \, (1 + r_w))$, where $d$ is the dimensionality of $z$.


Notice that the cost of evaluating the reparameterization estimator $g(w,\epsilon)$ is at least $\mathcal{O}(d (1 + r_w))$, since $\Sigma_w$ has $d (1 + r_w)$ parameters. However, constant factors here are usually significantly higher than for the control variate, since evaluating $f$ requries a pass through a dataset. Thus, as long as $r_v$ is ``small'', the control variate does not affect the algorithm's overall scalability.

These complexity results extend to cases where $B_v$ and/or $\Sigma_w$ are diagonal or full-rank matrices by replacing the corresponding rank, $r_v$ or $r_w$, by $0$ or $d$. For example, if $\Sigma_w$ is a full-rank matrix and $B_v$ is a diagonal plus rank-$r_v$, the control variate's cost is $\mathcal{O}(d^2 r_v)$. If both matrices are full-rank and $\Sigma_w$ is parameterized by its Cholesky factor $L$, the cost is $\mathcal{O}(d^3)$. This cubic cost comes entirely from instantiating $\Sigma_w = L L^\top$, all other costs are $\mathcal{O}(d^2)$.

\subsection{Constructing the Quadratic Approximation} \label{sec:challenge2}

The results in the previous section hold for any quadratic function $\hat f_v$. However, for the control variate to reduce variance, it is important that $\hat f_v$ is a good approximation of $f$. This section proposes two methods to find such an approximation.


A natural idea would be to use a Taylor approximation of $f$ \cite{traylorreducevariance_adam, viasstochastic_jordan}. However, as we discuss in Section \ref{sec:miller4}, this leads to serious computational challenges (and is suboptimal). Instead, we will directly seek parameters $v$ that minimize the variance of the final gradient estimator $g_\mathrm{cv}$. For a given set of parameters $w$, we set $z_0 = \mu_w$ and find the parameters $v$ by minimizing an objective $\mathcal{L}_w(v)$. We present two different objectives that can be used: 

\textbf{Method 1.} Find $v$ by minimizing the variance of the final gradient estimator (assuming $\gamma=1$),
\begin{equation}
\mathcal{L}_w(v) = \Var [g(w, \repsilon) + c_v(w, \repsilon)].\label{eq:L1}
\end{equation}
Using a sample $\epsilon \sim q_0(\epsilon)$ an unbiased estimate of $\nabla_v \mathcal{L}_w(v)$ can be obtained as
\begin{equation}
h_w(\epsilon, v) = \nabla_v \Vert g(w, \repsilon) + c_v(w, \repsilon) \Vert^2. \label{eq:estv1}
\end{equation}
\textbf{Method 2.} While the above method works well, it imposes a modest constant factor overhead, due to the need to differentiate through the control variate. As an alternative, we propose a simple proxy. The motivation is that the difference between the base gradient estimator and its approximation based on $\hat f_v$ is given by
\small
\begin{equation}
\nabla_w f(\mathcal{T}_w(\repsilon)) - \nabla_w \hat f_v(\mathcal{T}_w(\repsilon)) = \left( \frac{d\,\mathcal{T}_w(\repsilon)}{d\,w} \right)^\top \left( \nabla f(\mathcal{T}_w(\repsilon)) - \nabla \hat f_v(\mathcal{T}_w(\repsilon)) \right).
\end{equation}\normalsize
Thus, the closer $\nabla \hat f_v(z)$ is to $\nabla f(z)$, the better the control variate $c_v$ can approximate and cancel estimator $g$'s noise. Accordingly, we propose the proxy objective 
\begin{equation}
\mathcal{L}_w(v) = \frac{1}{2} \E_{q_0(\repsilon)} ||\nabla f(\Tr) - \nabla \hat f_v(\Tr)||^2. \label{eq:L2}
\end{equation}
Using a sample $\epsilon \sim q_0(\epsilon)$ an unbiased estimate of $\nabla_v \mathcal{L}_w(v)$ can be obtained as 
\begin{equation}
h_w(\epsilon, v) = \frac{1}{2} \nabla_v ||\nabla f(\Tr) - \nabla \hat f_v(\Tr)||^2. \label{eq:estv2}
\end{equation}
We observed that both methods lead to reductions in variance of similar magnitude (see Fig. \ref{fig:results_underover_s4} for a comparison). However, the second method introduces a smaller overhead. 


The idea of using a double-descent scheme to minimize gradient variance was explored in previous work. It has been done to set the parameters of a sampling distribution \cite{overdispersed_blei}, and to set the parameters of a control variate for discrete latent variable models \cite{backpropvoid, REBAR} using a continuous relaxation for discrete distributions \cite{gumbel1, gumbel2}.

\subsection{Final Algorithm} \label{sec:finalalg}

This section presents an efficient algorithm to use our control variate for SGVI. The approach involves maximizing the ELBO and finding a good quadratic approximation $\hat f_v$ simultaneously, via a double-descent scheme. We maximize the ELBO using stochastic gradient ascent with the gradient estimator from Eq. \ref{eq:gbase} and our control variate for variance reduction. Simultaneously, we find an approximation $\hat f_v$ by minimizing $\mathcal{L}_w(v)$ using stochastic gradient descent with the gradient estimators from Eq. \ref{eq:estv1} or \ref{eq:estv2}. Our procedure, summarized in Alg. \ref{alg:final}, involves alternating steps of each optimization process. Notably, optimizing $v$ as in Alg. \ref{alg:final} does not involve extra likelihood evaluations, since the model evaluations used to estimate the ELBO's gradient are re-used to estimate $\nabla_{v} \mathcal{L}_w(v)$.

Alg. \ref{alg:final} includes the control variate weight $\gamma$. This is useful in practice, specially at the beginning of training, when $v$ is far from optimal and $\hat f_v$ is a poor approximation of $f$. The (approximate) optimal weight can be obtained by keeping estimates of $\E[c(w,\epsilon)^\top g(w, \epsilon)]$ and $\E[c(w,\epsilon)^\top c(w,\epsilon)]$ as optimization proceeds \cite{cvs}.


\begin{algorithm}[t]
    \setstretch{1.1}
    \begin{small}
    \caption{SGVI with the proposed control variate.}
    \label{alg:final}
    \begin{algorithmic}
    \REQUIRE Learning rates $\alpha^{(w)}, \alpha^{(v)}$.
    \STATE Initialize $w_0$, $v_0$ and control variate weight $\gamma = 0$.
    \FOR{$k = 1, 2, \cdots$}
    \STATE Sample $\epsilon \sim q_0$ and compute $z = \mathcal{T}_{w_k}(\epsilon)$.
    \STATE Compute estimator and control variate $g = g(w_k, \epsilon)$, $c = c_{v_k}(w_k, \epsilon)$. \hfill (Eqs. \ref{eq:gbase} and \ref{eq:cv})
    \STATE Take primary step as $w_{k+1} \leftarrow w_k + \alpha^{(w)} (g+ \gamma c)$.
    \STATE Update $\gamma$ to minimize empirical $\Var[g + \gamma c]$. \hfill (Sec. \ref{sec:finalalg})
    \STATE Compute control variate gradient estimator $h = h_w(\epsilon, v_k)$. \hfill (Eq. \ref{eq:estv1} or \ref{eq:estv2})
    \STATE Take dual step as $v_{k+1} \leftarrow v_k - \alpha^{(v)} h$.
    \ENDFOR
    \end{algorithmic}
    \end{small}
\end{algorithm}

\section{Comparison of Approximations} \label{sec:miller4}




\textbf{Taylor-Based Approximations.} There is closely related work exploring Taylor-expansion based control variates for reparameterization gradients \cite{traylorreducevariance_adam}. These control variates can be expressed as
\small
\begin{equation}
\small c(w, \epsilon) = \E_{q_{0}(\repsilon)} \left[\left(\frac{d\,\mathcal{T}_w(\repsilon)}{d\,w}\right)^\top \nabla \hat{f}(\Tr) \right] - \left(\frac{d\,\mathcal{T}_w(\repsilon)}{d\,w}\right)^\top \nabla \hat{f}(\Tr). \label{eq:mcv-m}
\end{equation}\normalsize
This is similar to Eq. \ref{eq:cv}. The difference is that, here, the approximation $\hat f$ is set to be a Taylor expansion of $f$. In general this leads to an intractable control variate: the expectation may not be known, or the Taylor approximation may be intractable (e.g. requires computing Hessians). However, in some cases, it can be computed efficiently. For this discussion we focus on Gaussian variational distributions, where the parameters $w$ are the mean and scale. 


For the gradient with respect to the mean parameters, $\hat{f}(z)$ can be set to be a second-order Taylor expansion of $f(z)$ around the current mean. This might appear to be problematic, since computing the Hessian of $f$ will be intractable in general. However, it turns out that, for the mean parameters, this leads to a control variate that can be computed using only Hessian-vector products. This was first observed by Miller et al. \cite{traylorreducevariance_adam} for diagonal Gaussians.

For the scale parameters, even with a diagonal Gaussian, using a second-order Taylor expansion requires the diagonal of the Hessian, which is intractable in general. For this reason, Miller et al. \cite{traylorreducevariance_adam} propose an approach equivalent\footnote{The original paper \cite{traylorreducevariance_adam} describes the control variate for the scale parameters as using a second-order Taylor expansion, and then applies an additional approximation based on a minibatch to deal with intractable Hessian computations. In Appendix \ref{app:miller} we show these formulations are exactly equivalent.} to setting  $\hat f$ to a {\em first}-order Taylor expansion, so that $\nabla \hat{f}$ is constant. 


The biggest drawback of Taylor-based control variates is that the crude first-order Taylor approximation used for the scale parameters provides almost no variance reduction. Interestingly, this seems to pose very little problem with diagonal Gaussians. This is because, in this case, the gradient with respect to the mean parameters typically contribute almost all the variance. However, this approach may be useless in some other situations: With non-diagonal distributions, the scale parameters often contribute the majority of the variance (see Fig.~\ref{fig:results_underover_s4}).

A second drawback is that even a second-order Taylor expansion is not optimal. A Taylor expansion provides a good \textit{local} approximation, which may be poor for distributions $q_w$ with large variance.

\textbf{Demonstration.} Fig.~\ref{fig:results_underover_s4} compares four gradient estimators on a Bayesian logistic regression model (see Sec. \ref{sec:results}): plain reparameterization, reparameterization with a Taylor-based control variate, and reparameterization with our control variate (minimizing Eq.~\ref{eq:L1} or Eq.~\ref{eq:L2}, using a diagonal plus rank-$10$ matrix $B_v$). The variational distribution is either a diagonal Gaussian or a Gaussian with arbitrary full-rank covariance. We set the mean \small$\mu_w=0$\normalsize\ and covariance \small$\Sigma_w = \sigma^2 I$\normalsize. We measure each estimator's variance for different values of $\sigma$. For transparency, in all cases we use a fixed weight $\gamma=1$.

There are four key observations: (i) Our variance reduction for the mean parameters is somewhat better than a Taylor approximation (which even increases variance in some cases). This is not surprising, since a Taylor expansion was never claimed to be optimal; (ii) Our control variate is vastly better for the scale parameters; (iii) the variance for fully-factorized distributions is dominated by the mean, while the variance for full-covariance distributions it is dominated by the scale; (iv) the proxy for the gradient variance (Eq.~\ref{eq:L1}) performs extremely similarly to the true gradient variance (Eq.~\ref{eq:L2}).

\begin{figure}[t]
    \begin{center}
    \includegraphics[scale=0.3, trim={0cm, 1.9cm, 0cm, 0cm}, clip]{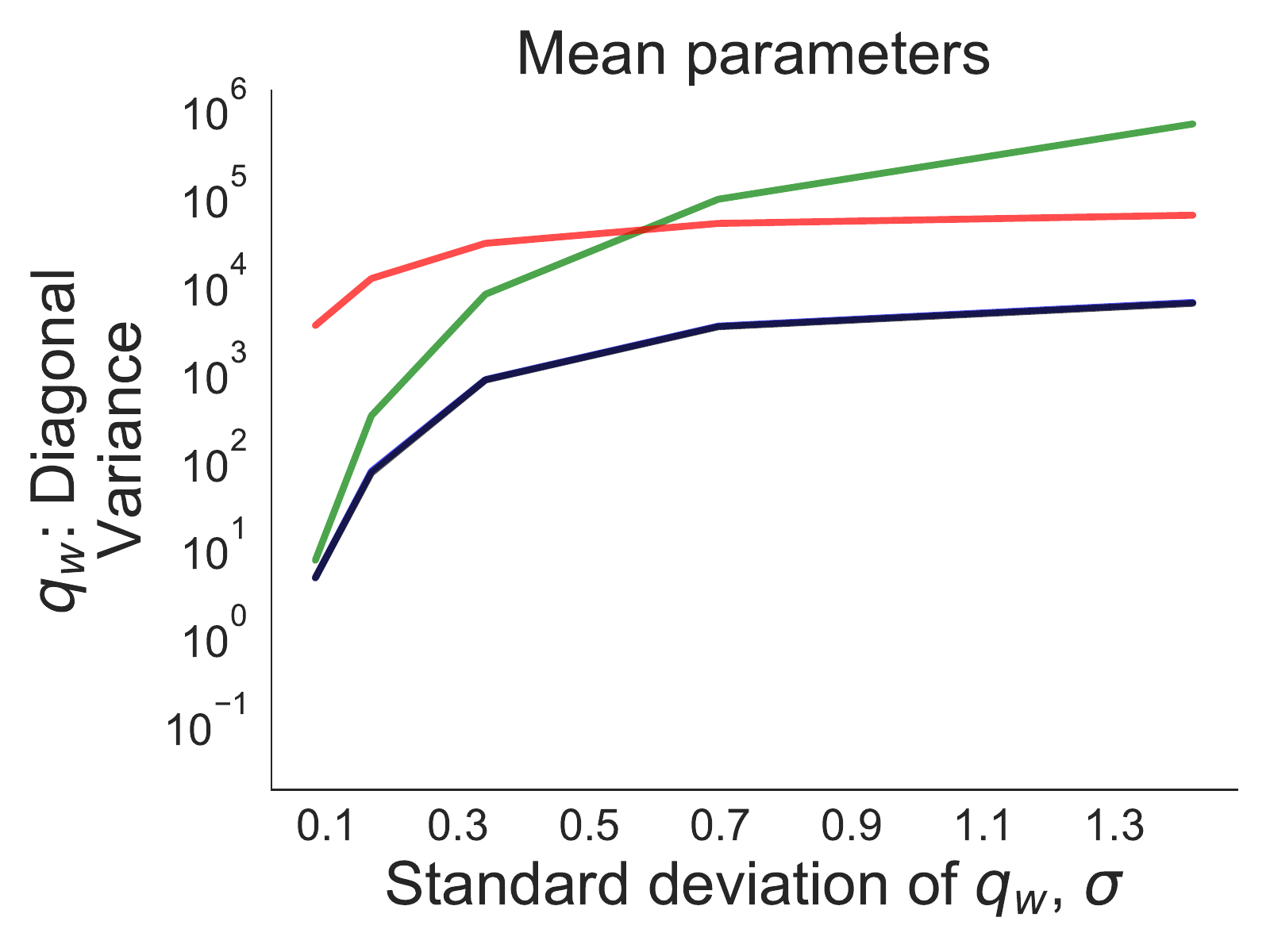}
    \includegraphics[scale=0.3, trim={3.3cm, 1.9cm, 0cm, 0cm}, clip]{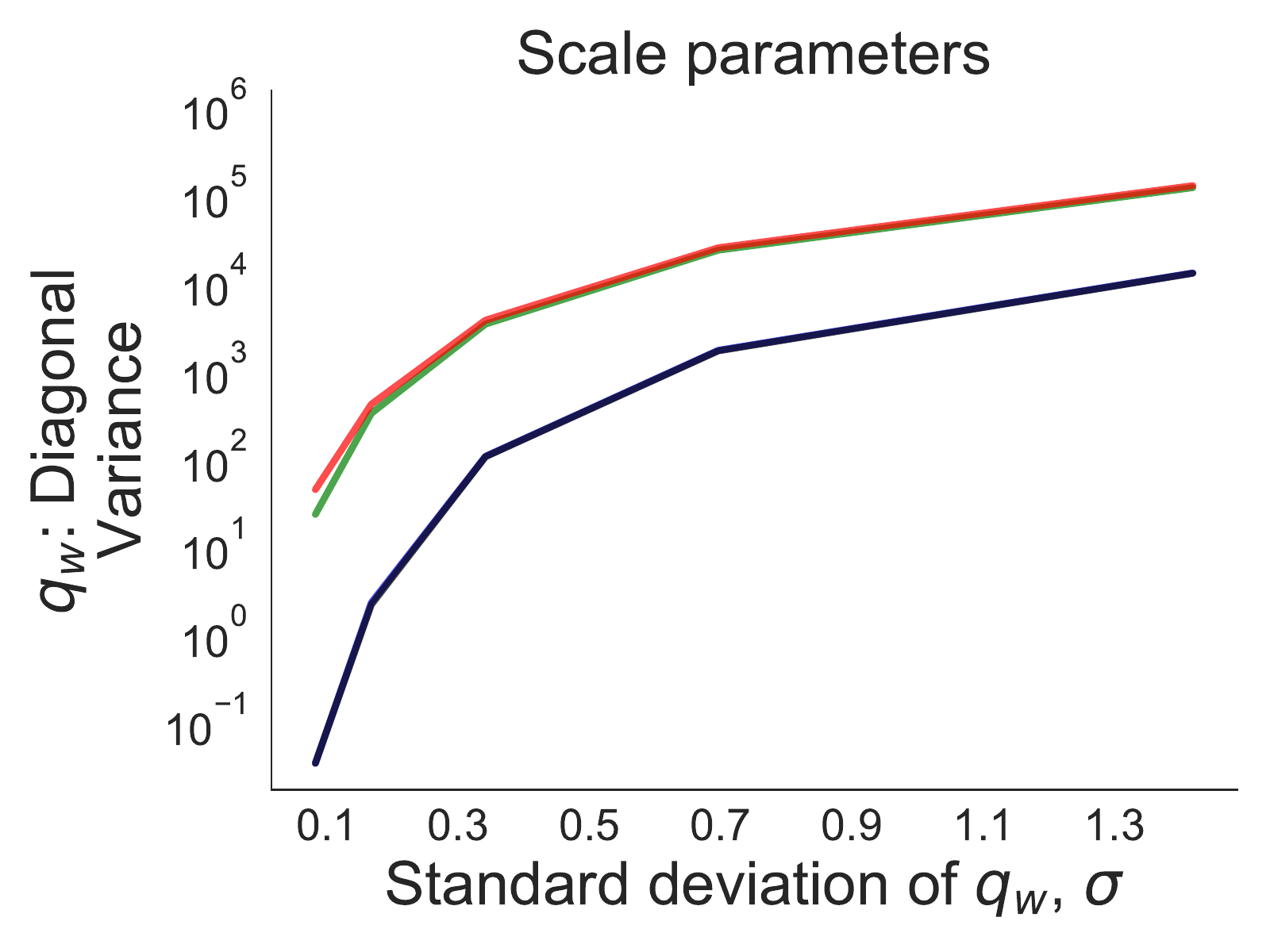}
    \includegraphics[scale=0.3, trim={3.3cm, 1.9cm, 0cm, 0cm}, clip]{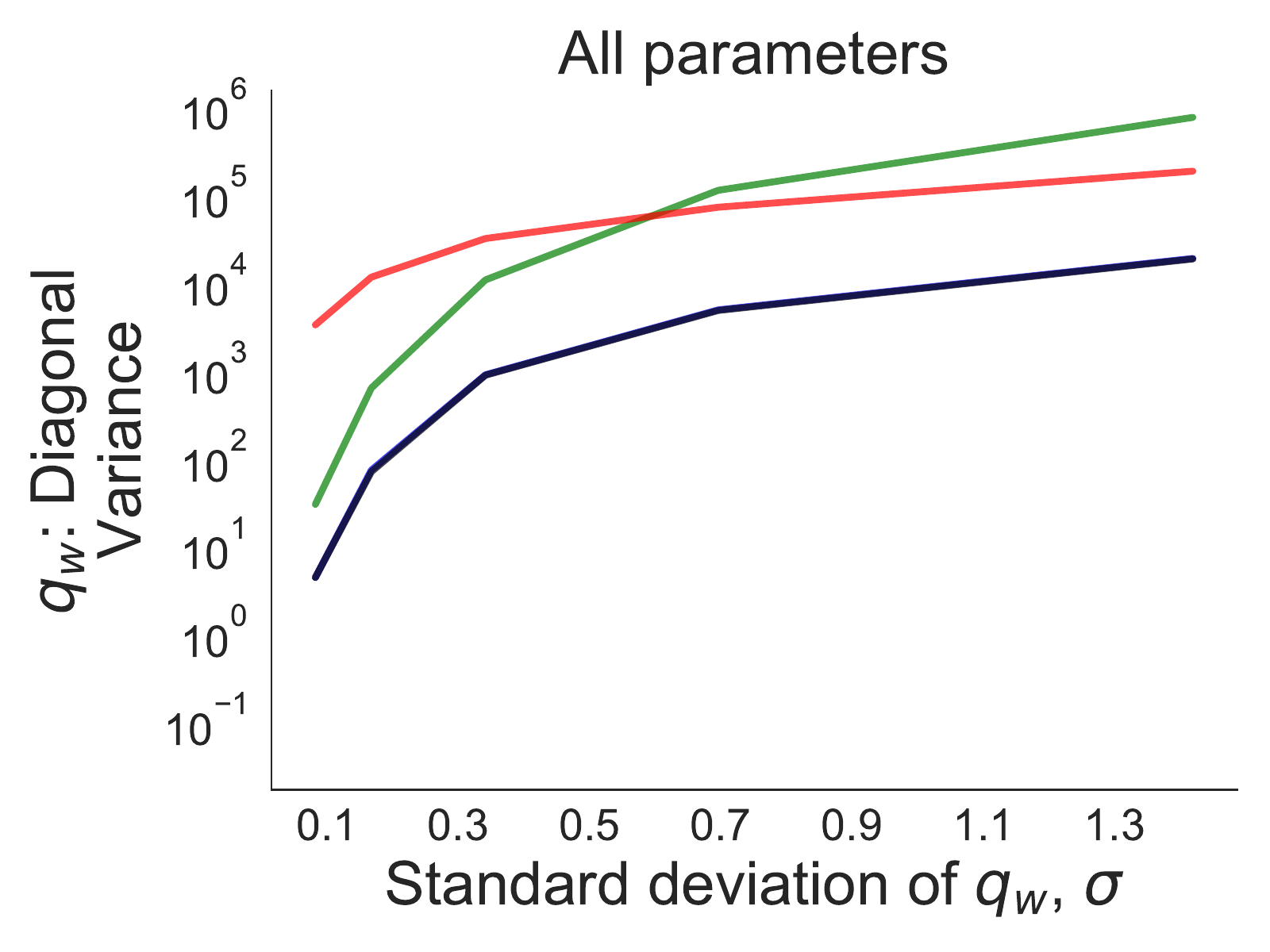}
    
    \hspace{0.13cm}
    \includegraphics[scale=0.3, trim={0cm, 0cm, 0cm, 1.1cm}, clip]{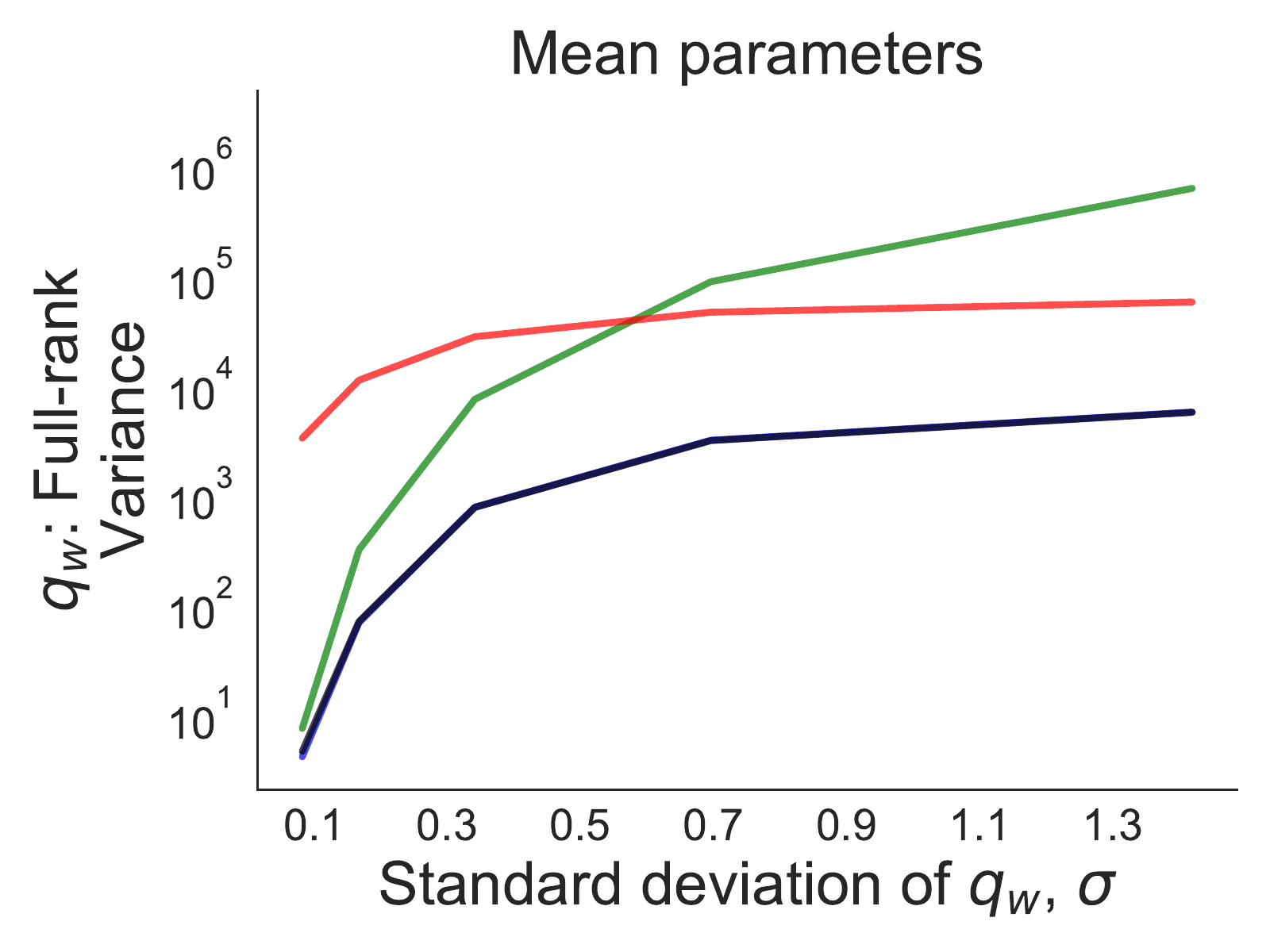}
    \includegraphics[scale=0.3, trim={3.2cm, 0cm, 0cm, 1.1cm}, clip]{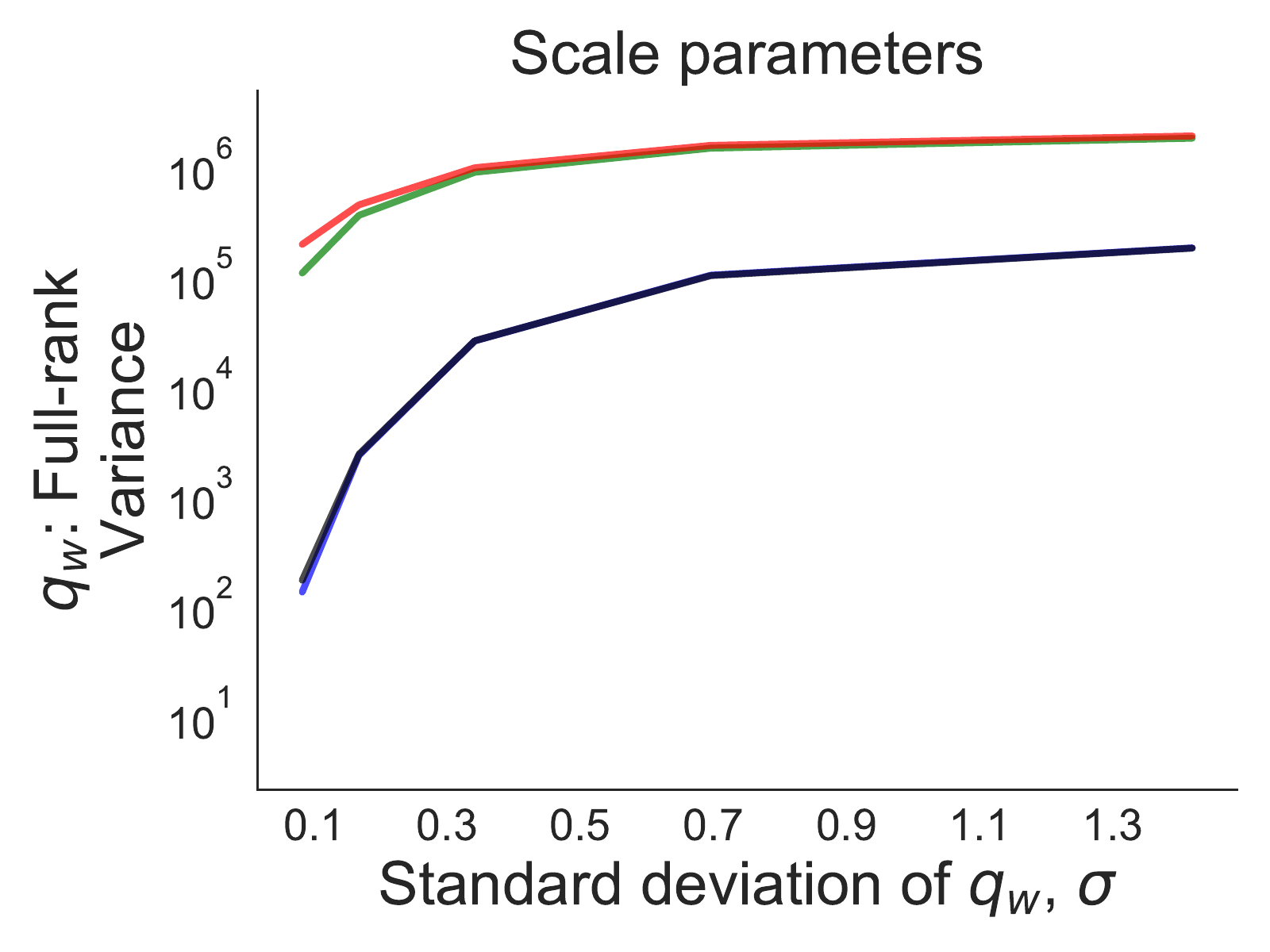}
    \includegraphics[scale=0.3, trim={3.2cm, 0cm, 0cm, 1.1cm}, clip]{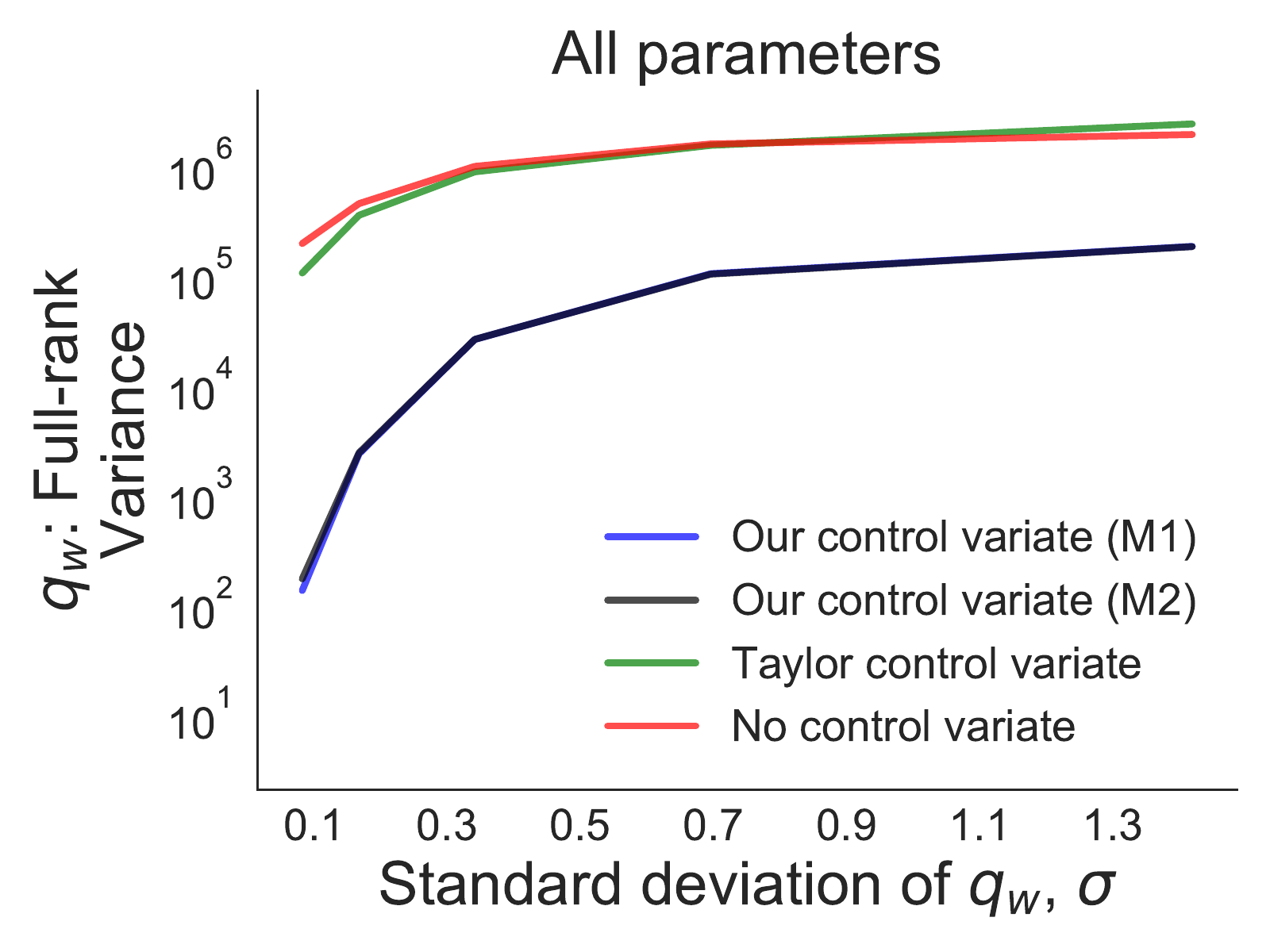}
    \caption{\textbf{The new control variate improves variance, particularly for scale parameters.} Variance of different gradient estimators on a Bayesian logistic regression model for a variational distribution with mean zero and covariance $\sigma^2 I$ for varying $\sigma$.
    The mean parameters (where Miller's approach often works well) dominate the variance for fully-factorized distributions, while the scale parameters (where Miller's approach does little) dominate for full-rank Gaussians. Method 1 (M1) and Method 2 (M2) to find the parameters of our control variate perform extremely similarly.}
    \label{fig:results_underover_s4}
    \end{center}
\end{figure}

It should be emphasized that, for this analysis, the parameters $v$ are trained to completion for each value of $\sigma$. This does not exactly reflect what would be expected in practice, where the dual-descent scheme "tracks" the optimal $v$ as $w$ changes. Experiments in the next section consider this practical setting.

%% file: sections/experiments_50.tex

\section{Experiments and Results} \label{sec:results}

We present results that empirically validate the the proposed control variate and algorithm. We perform SGVI on several probabilistic models using different variational distributions. We maximize the ELBO using the reparameterization estimator with the proposed control variate to reduce its variance (Alg.~\ref{alg:final}). We compare against optimizing using the reparameterization estimator without any control variates, and against optimizing using a Taylor-based control variates for variance reduction.

\subsection{Experimental details} \label{sec:exp_details}

\textbf{Tasks and datasets:} We use three different models: Logistic regression with the \textit{a1a} dataset, hierarchical regression with the \textit{frisk} dataset \cite{frisk}, and a Bayesian neural network with the \textit{red wine} dataset. The latter two are the ones used by Miller et al. \cite{traylorreducevariance_adam}. (Details for each model in App.~\ref{app:models}.)

\textbf{Variational distribution:} We consider diagonal Gaussians parameterized by the log-scale parameters, and diagonal plus rank-$10$ Gaussians, whose covariance is parameterized by a diagonal component $D$ and a factor $F$ of shape $d \times 10$ (i.e. \small$\Sigma_w = D + F F^\top$\normalsize) \cite{diagLR}. For the simpler models, logistic regression and hierarchical regression, we also consider full-rank Gaussians parameterized by the Cholesky factor of the covariance.

\textbf{Algorithmic details:} We use Adam \cite{adam} to optimize the parameters $w$ of the variational distribution $q_w$ (with step sizes between $10^{-5}$ and $10^{-2}$). We use Adam with a step size of $0.01$ to optimize the parameters $v$ of the control variate, by minimizing the proxy to the variance from Eq. \ref{eq:L2}. We parameterize $B_v$ as a diagonal plus rank-$r_v$. We set $r_v = 10$ when diagonal or diagonal plus low rank variational distributions are used, and $r_v = 20$ when a full-rank variational distribution is used. (We show results using other ranks in Appendix~\ref{sec:otherranks}.)

\textbf{Baselines considered:} We compare against optimization using the base reparameterization estimator (Eq. \ref{eq:gbase}). We also compare against using Taylor-based control variates. (We generalize the Taylor approach to full-rank and diagonal plus low-rank distributions in Appendix \ref{app:millerdrawbacks}.) For all control variates we find the (approximate) optimal weight using the method from Geffner and Domke \cite{cvs} (fixing the weight to 1 lead to strictly worse results). We use $M = 10$ and $M = 50$ samples from $q_w$ to estimate gradients.

We show results in terms of iterations and wall-clock time. Table \ref{tab:costs} shows the per iteration time-cost of each method in our experiments. Our method's overhead is around 50\%, while the Taylor approach has an overhead of around 150\%. These numbers depend on the implementation and platform, but should give a rough estimate of the overhead in practice (we use PyTorch 1.1.0 on an Intel i5 2.3GHz).

\begin{table}[]
\centering
\begin{tabular}{lllllllll}
\toprule
\multirow{2}{*}{\#Samples} & \multirow{2}{*}{Model} & \multicolumn{3}{c}{$q_w$: Diag plus low rank} & & \multicolumn{3}{c}{$q_w$: full-rank covariance} \\ \cmidrule{3-5} \cmidrule{7-9}
                          &          & Base      & Our CV & Taylor  && Base   & Our CV & Taylor   \\ \midrule  
\multirow{3}{*}{$M = 10$} & Hierarchical  & $4.4$     & $6.4$  & $10.8$  && $3.9$  & $6.0$  & $10.2$   \\
                          & Logistic & $3.8$     & $6.3$  & $7.7$   && $4.9$  & $8.1$  & $9.7$    \\ 
                          & BNN      & $11.1$    & $16.2$ & $31.2$  && $-$    & $-$    & $-$      \\ \midrule  
\multirow{3}{*}{$M = 50$} & Hierarchical  & $5.8$     & $8.3$  &  $12.8$ && $4.9$  & $7.4$  & $11.7$   \\
                          & Logistic & $8.1$     & $11$   &  $16.5$ && $14.2$ & $20.1$ & $32.1$    \\ 
                          & BNN      & $17.3$    & $25.6$ &  $48.4$ && $-$    & $-$    & $-$      \\ \bottomrule
\end{tabular}
\caption{Cost (milliseconds) of performing one optimization step using no control variates (Base), a Taylor-based control variate (Taylor), and our control variate (Our CV). For the latter, one step involves computing the gradient, control variate, and updating the parameters $v$. For reference, computing the Hessian of $f$ takes $131, 146 \mbox{ and } 2883$ milliseconds for the hierarchical regression, logistic regression and Bayesian neural network models. As expected, because of these high costs, using a second order Taylor-based control variate for the scale parameters is not practical.}
\label{tab:costs}
\end{table}

\subsection{Results}

Fig.~\ref{fig:opt1} shows optimization results for the diagonal plus low rank Gaussian variational distribution. The two leftmost columns show ELBO vs. iteration plots for two specific learning rates. The third column shows, for each method and iteration, the ELBO for the best learning rate chosen retrospectively. In all cases, our method improves over competing approaches. In fact, our method with $M = 10$ samples to estimate the gradients performs better than competing approaches with $M = 50$. On the other hand, Taylor-based control variates give practically no improvement over using the base estimator alone. This is because most of the gradient variance comes from estimating the gradient with respect to the scale parameters, for which Taylor-based control variates do little.

\begin{figure}[ht!]
    \begin{center}
    \includegraphics[scale=0.29,trim={0 0 0 0},clip]{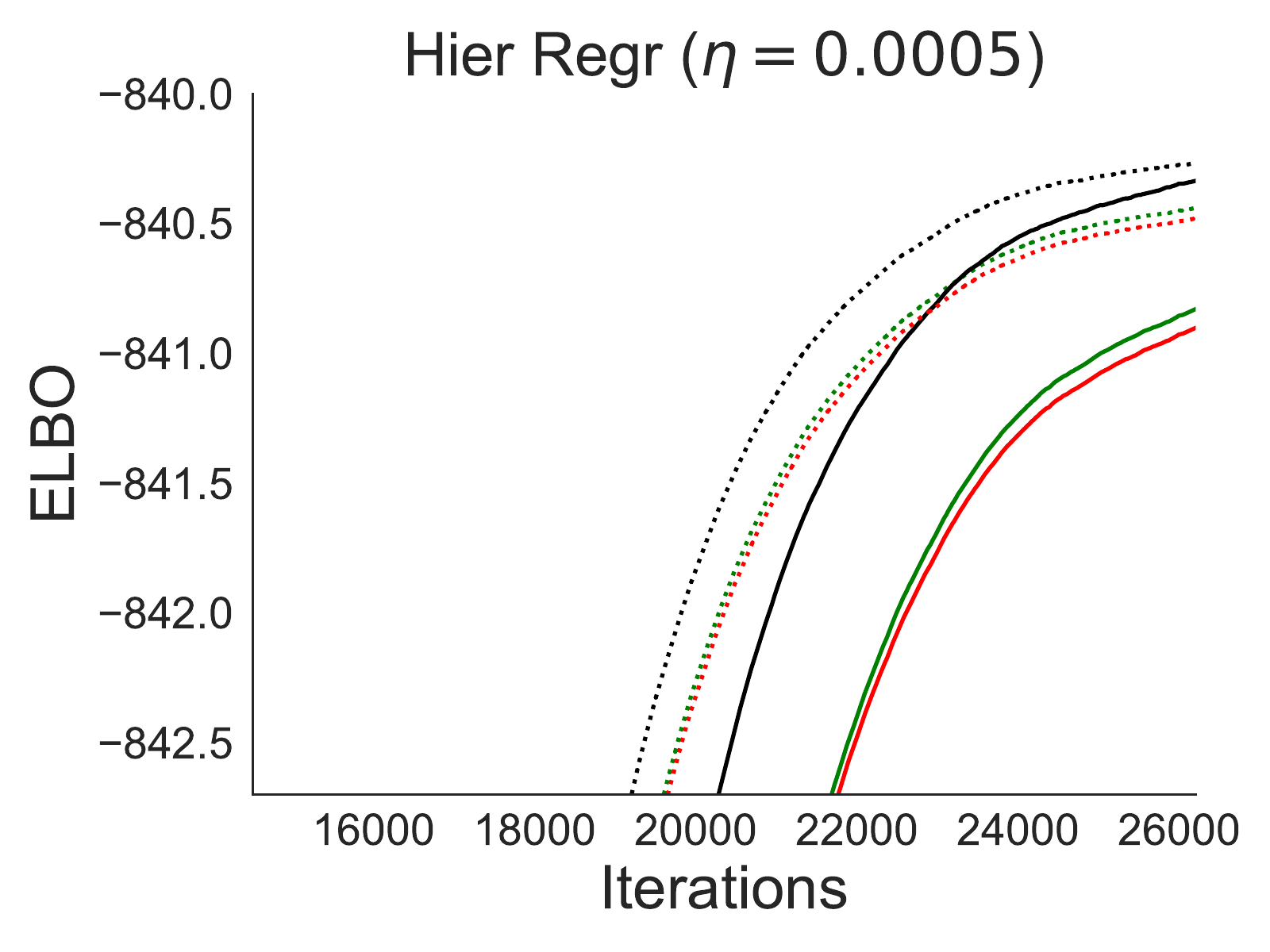}
    \includegraphics[scale=0.29,trim={3cm 0 0 0},clip]{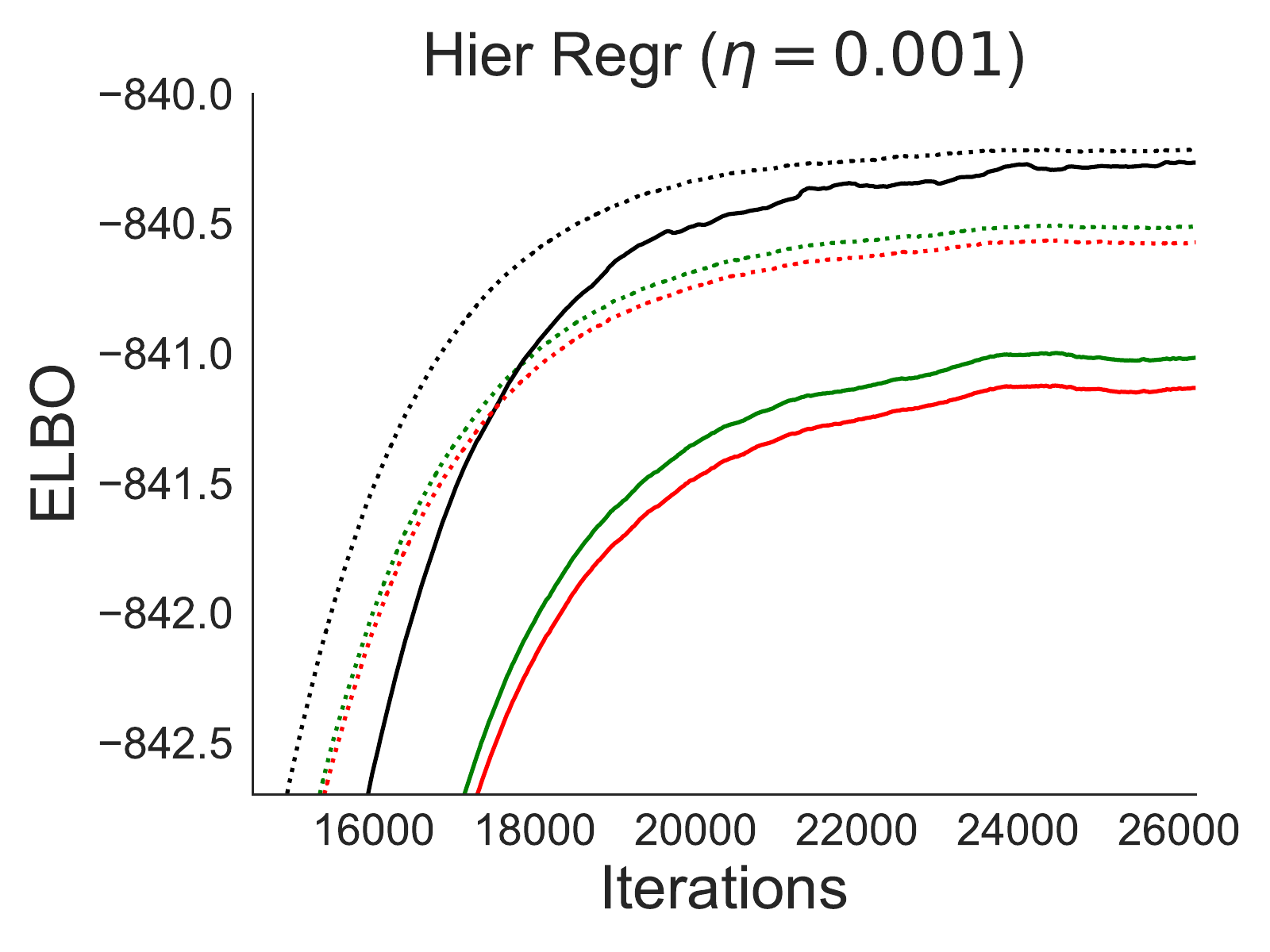}
    \includegraphics[scale=0.29,trim={3cm 0 0 0},clip]{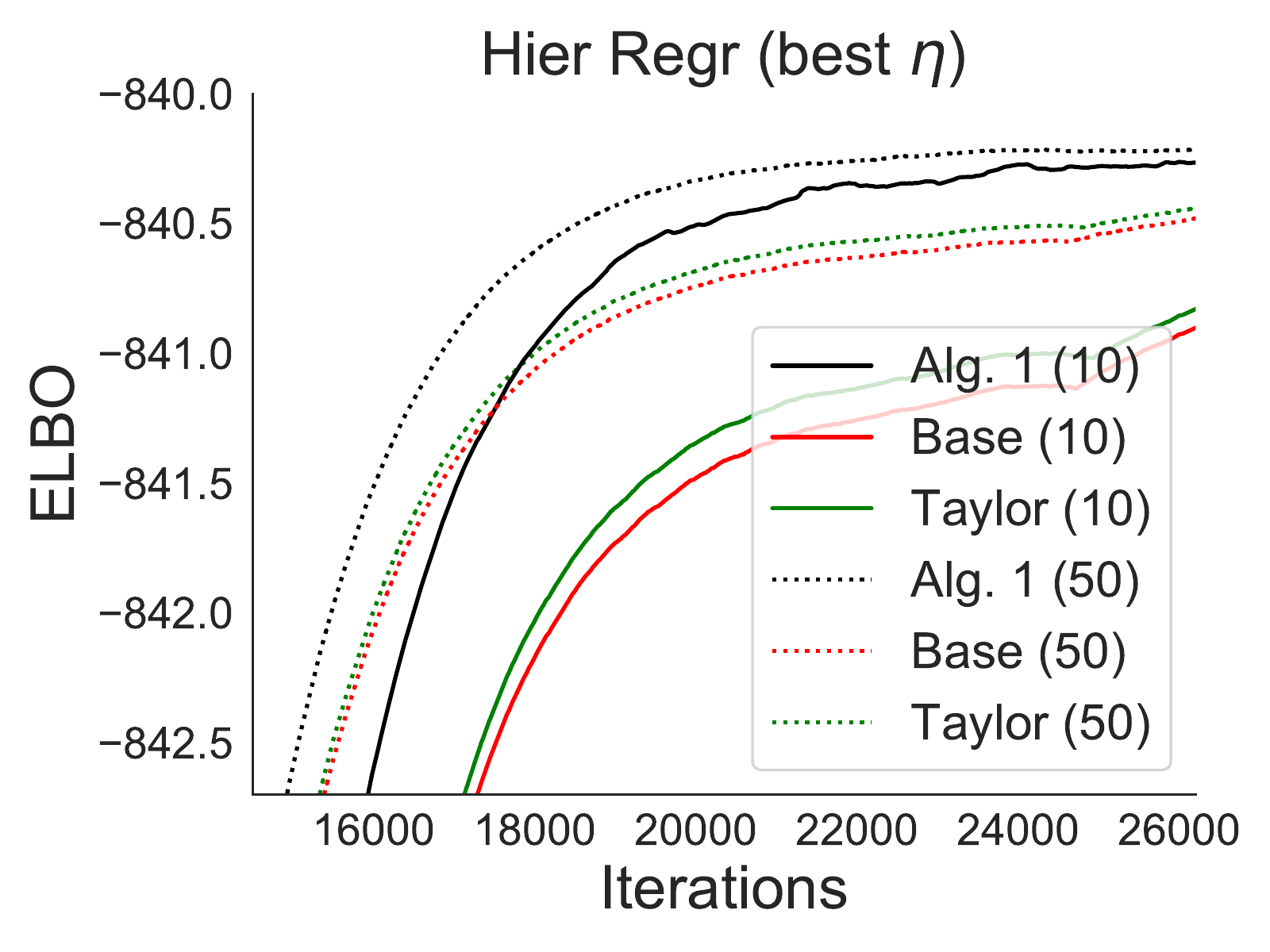}

    \includegraphics[scale=0.28,trim={0 0 0 0},clip]{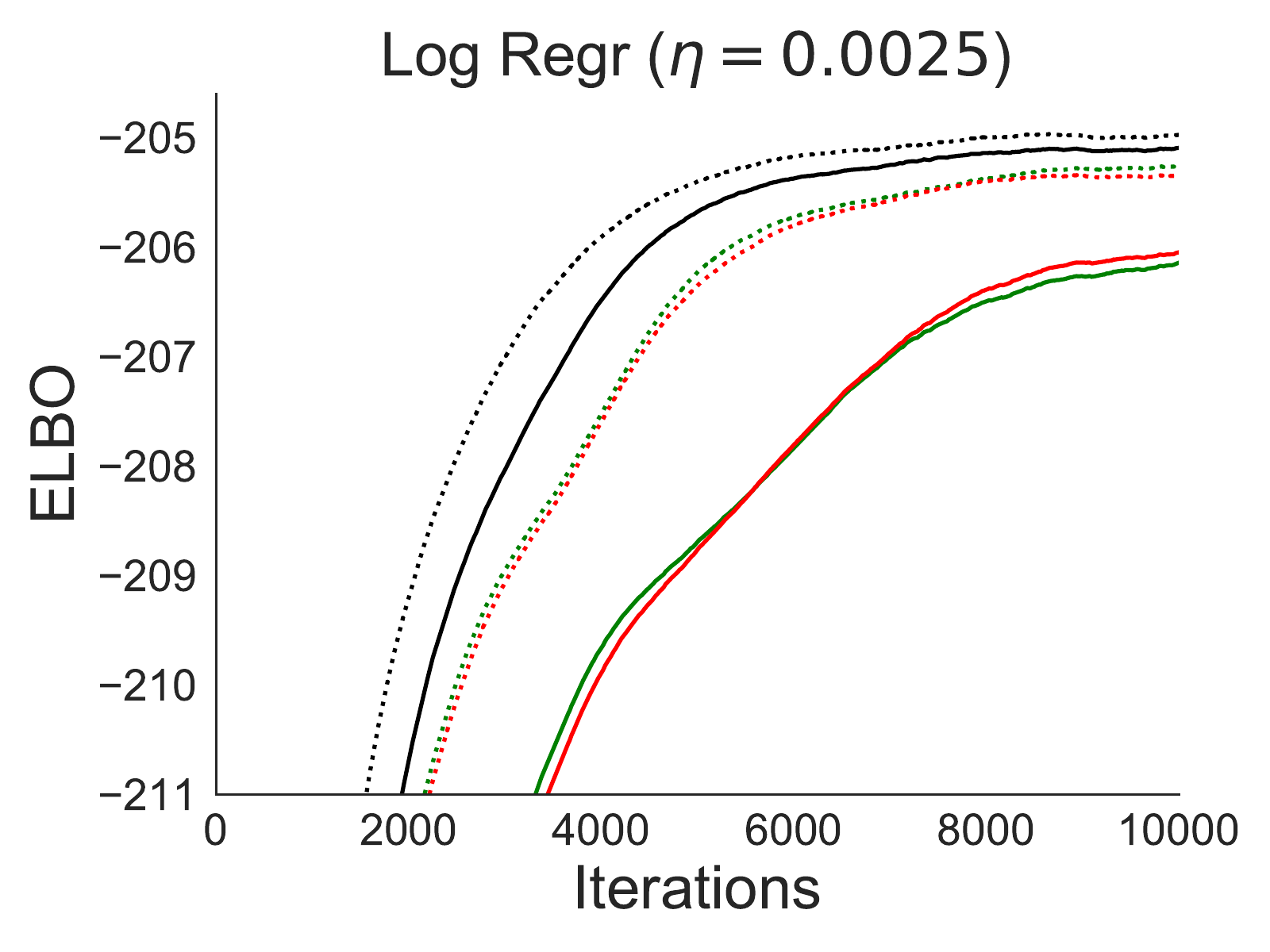}
    \includegraphics[scale=0.28,trim={2.5cm 0 0 0},clip]{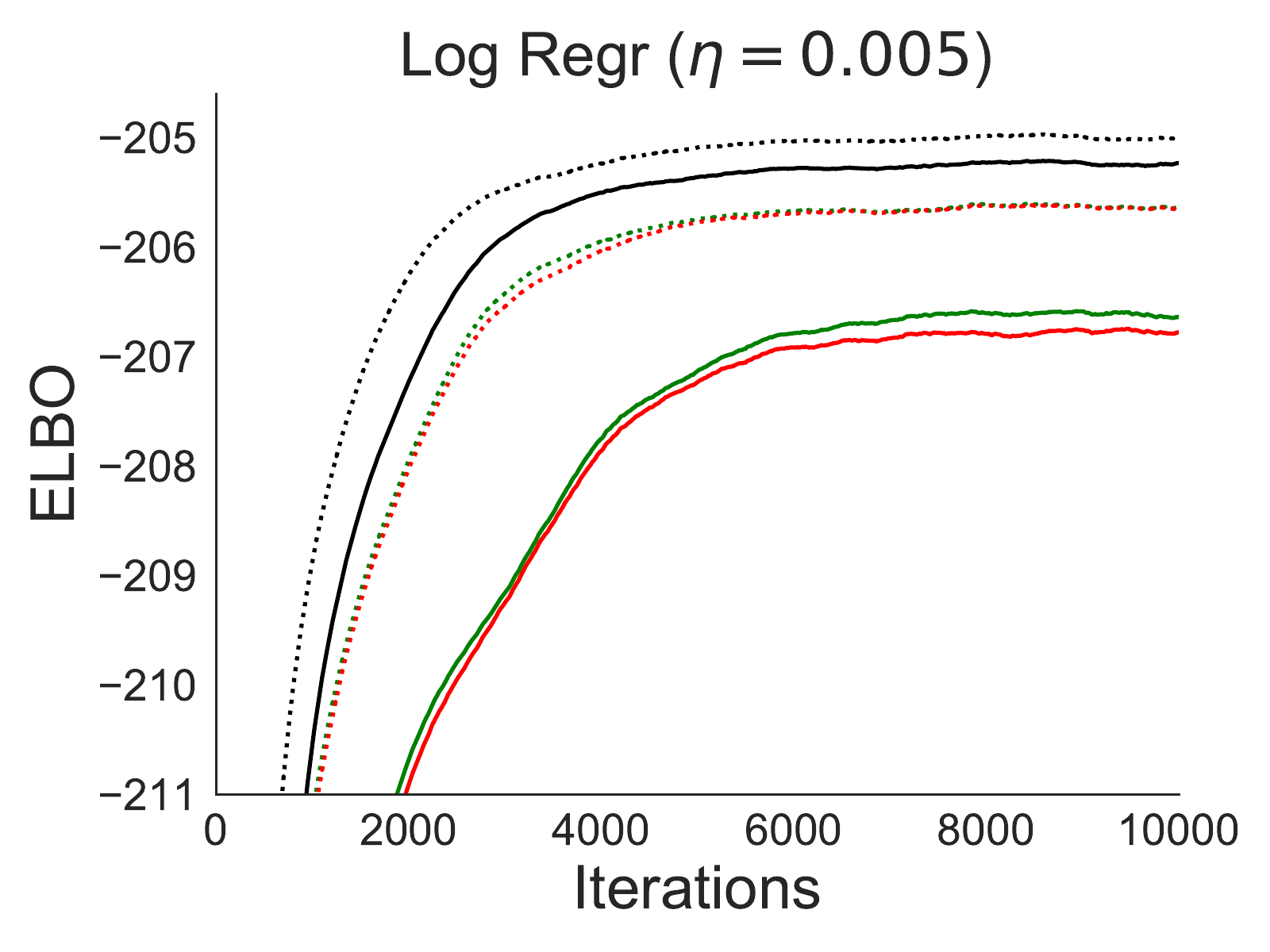}
    \includegraphics[scale=0.28,trim={2.5cm 0 0 0},clip]{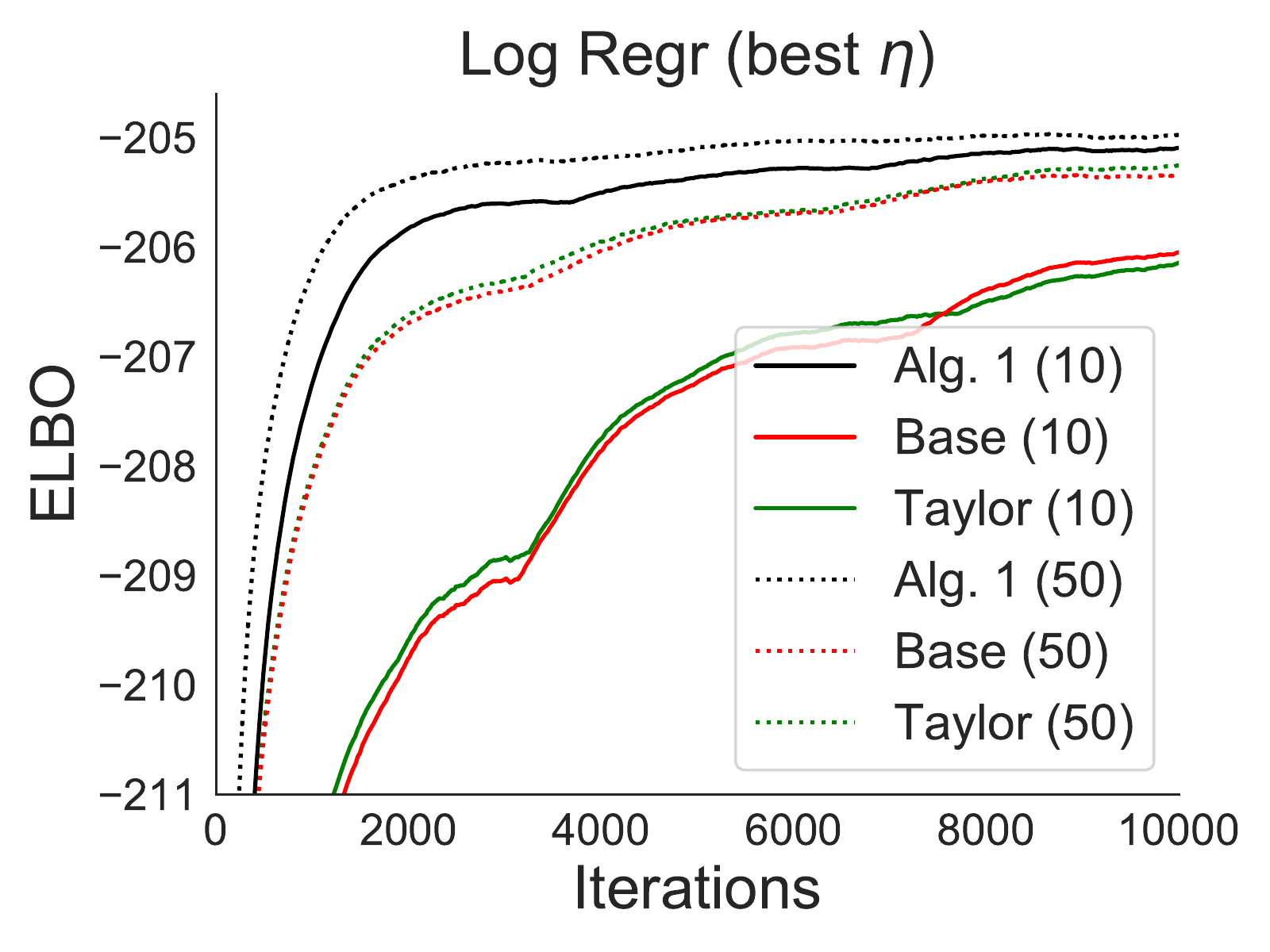}

    \includegraphics[scale=0.28,trim={0 0 0 0},clip]{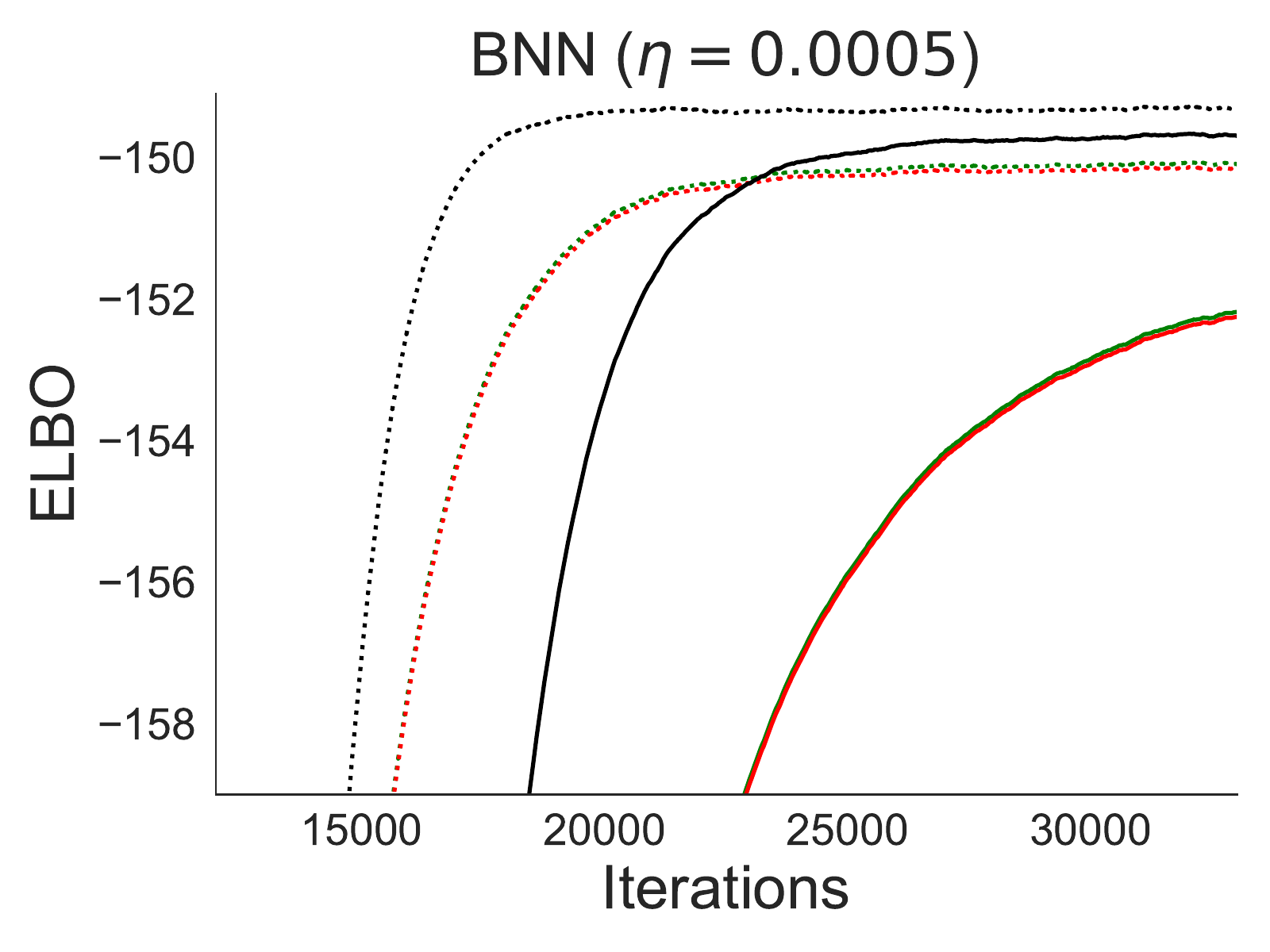}
    \includegraphics[scale=0.28,trim={2.5cm 0 0 0},clip]{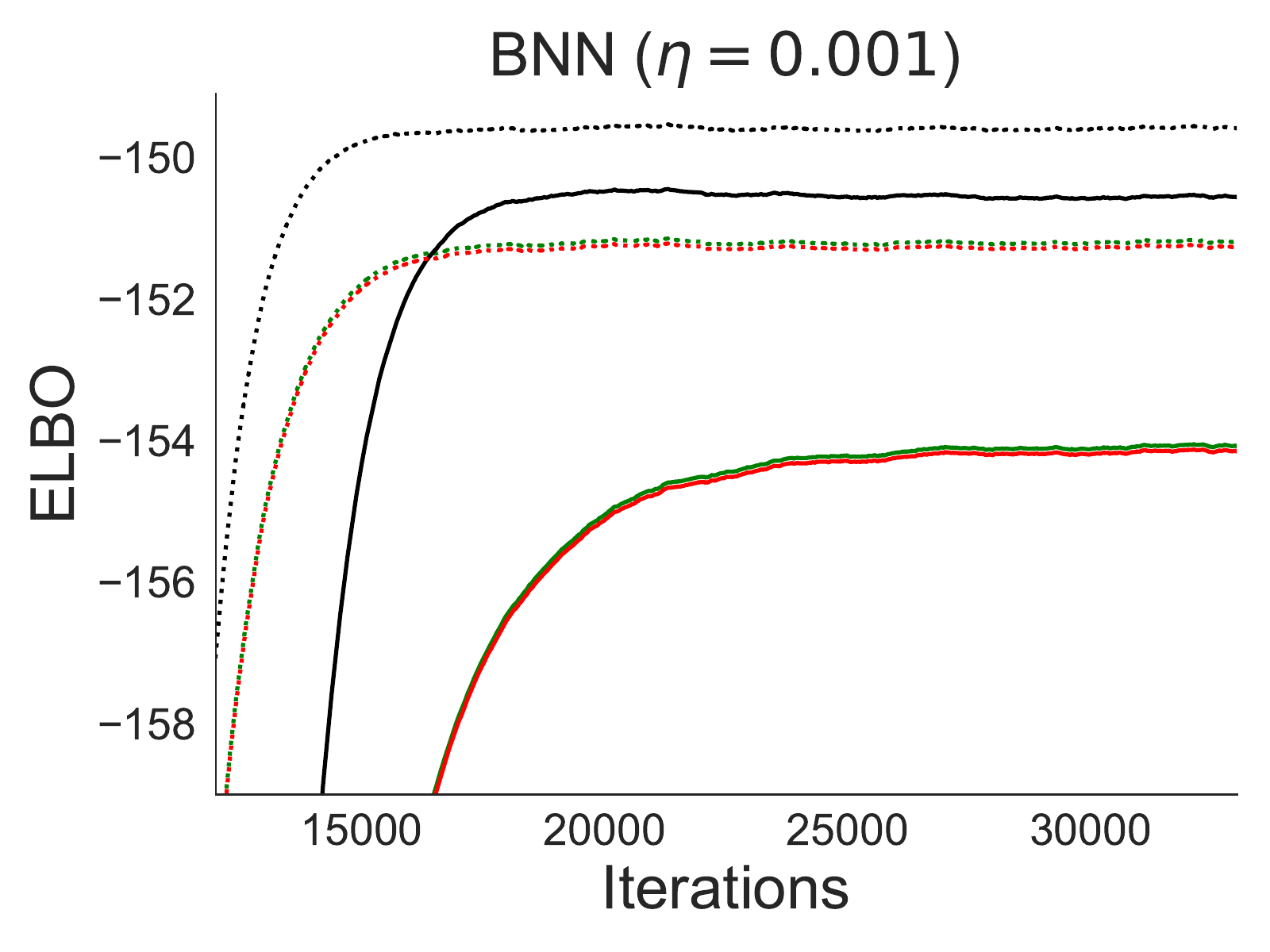}
    \includegraphics[scale=0.28,trim={2.5cm 0 0 0},clip]{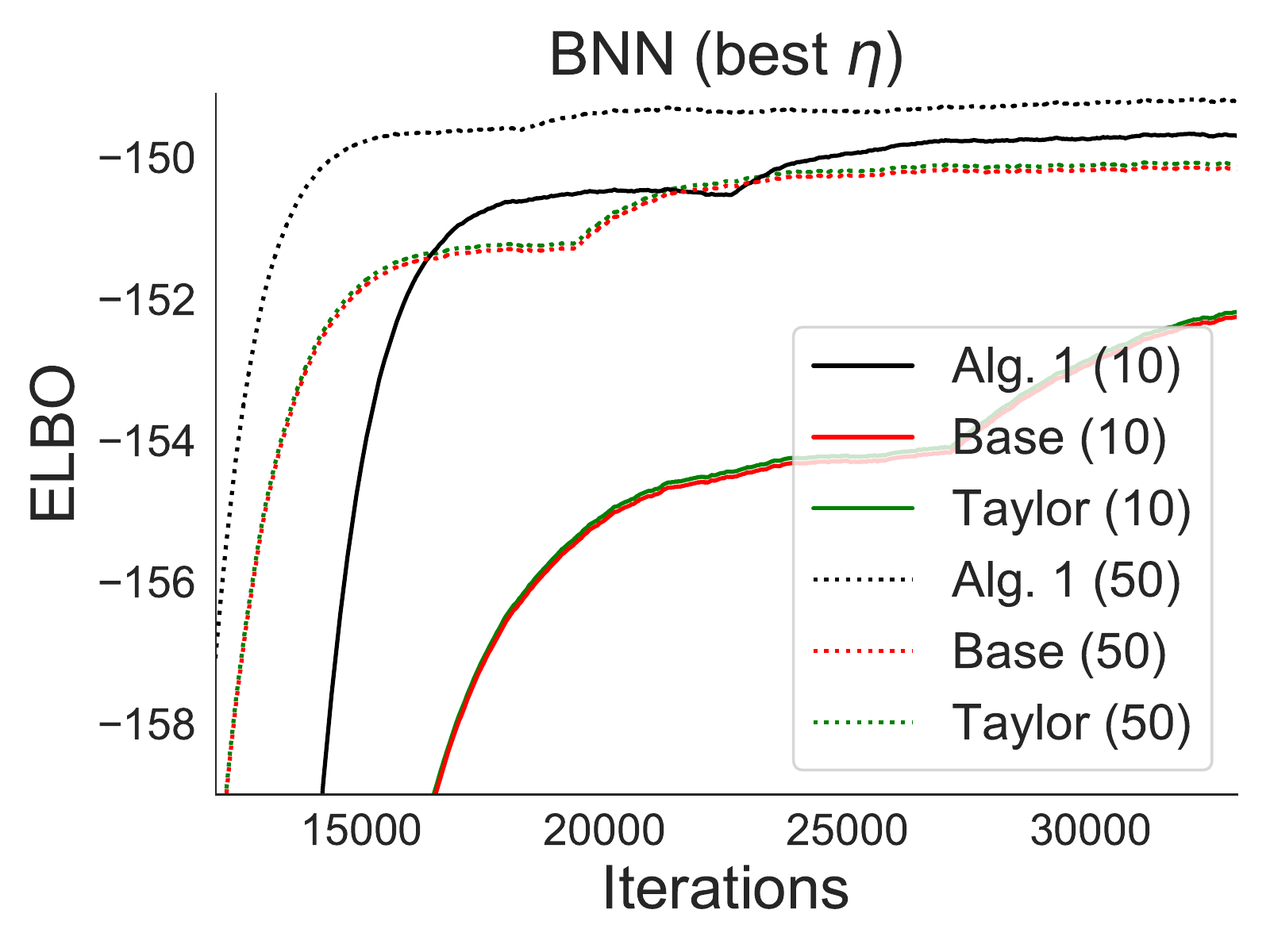}
    \caption{\textbf{The use of our control variate yields improved optimization convergence.} VI using a diagonal plus low rank Gaussian variational distribution. The first two columns show results for two different step-sizes, and the third one using the best step-size chosen retrospectively. "Base (M)" stands for the base reparameterization gradient estimated using $M$ samples, and "Taylor (M)" for using a Taylor-expansion based control variate for variance reduction.}
    \label{fig:opt1}
    \end{center}
\end{figure}

To test the robustness of optimization, in Fig.~\ref{fig:per1} we show the final training ELBO after 80000 steps as a function of the step size used. Our method is less sensitive to the choice of step size. In particular, our method gives reasonable results with larger learning rates, which translates to better results with a smaller number of iterations.

\begin{figure}[ht!]
    \begin{center}
    \includegraphics[scale=0.28,trim={0 0 0 0},clip]{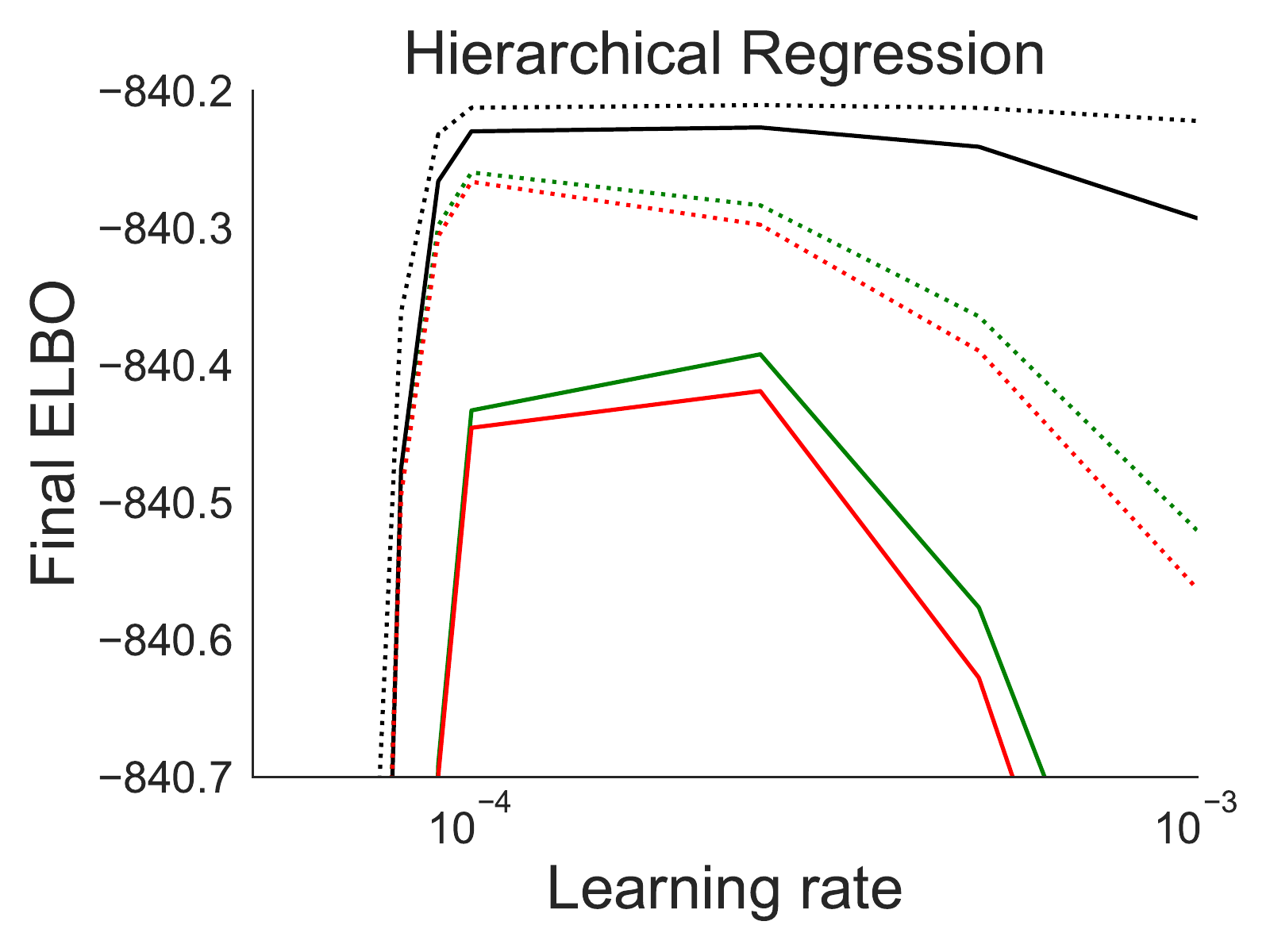}
    \includegraphics[scale=0.28,trim={1.3cm 0 0 0},clip]{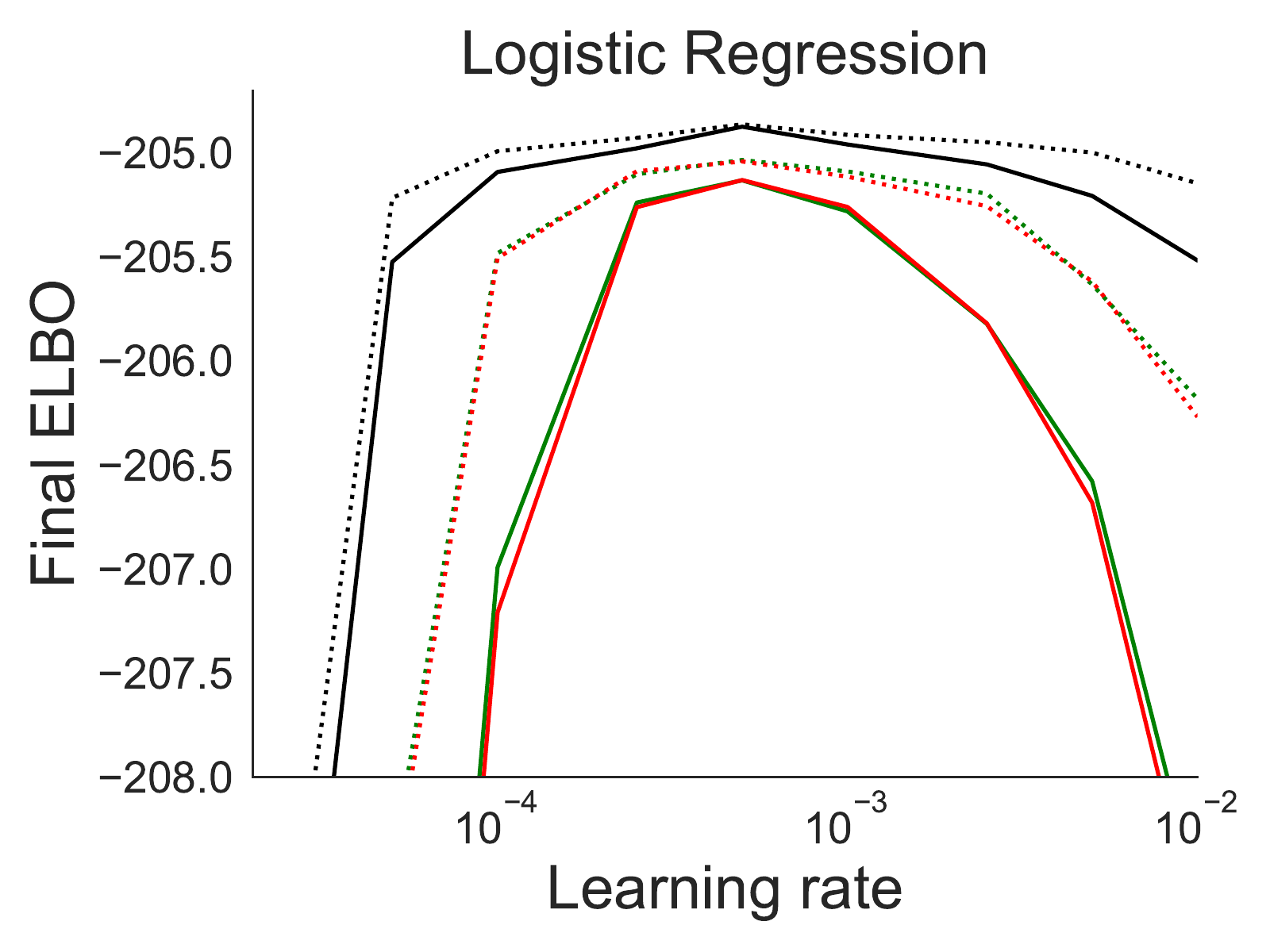}
    \includegraphics[scale=0.28,trim={1.3cm 0 0 0},clip]{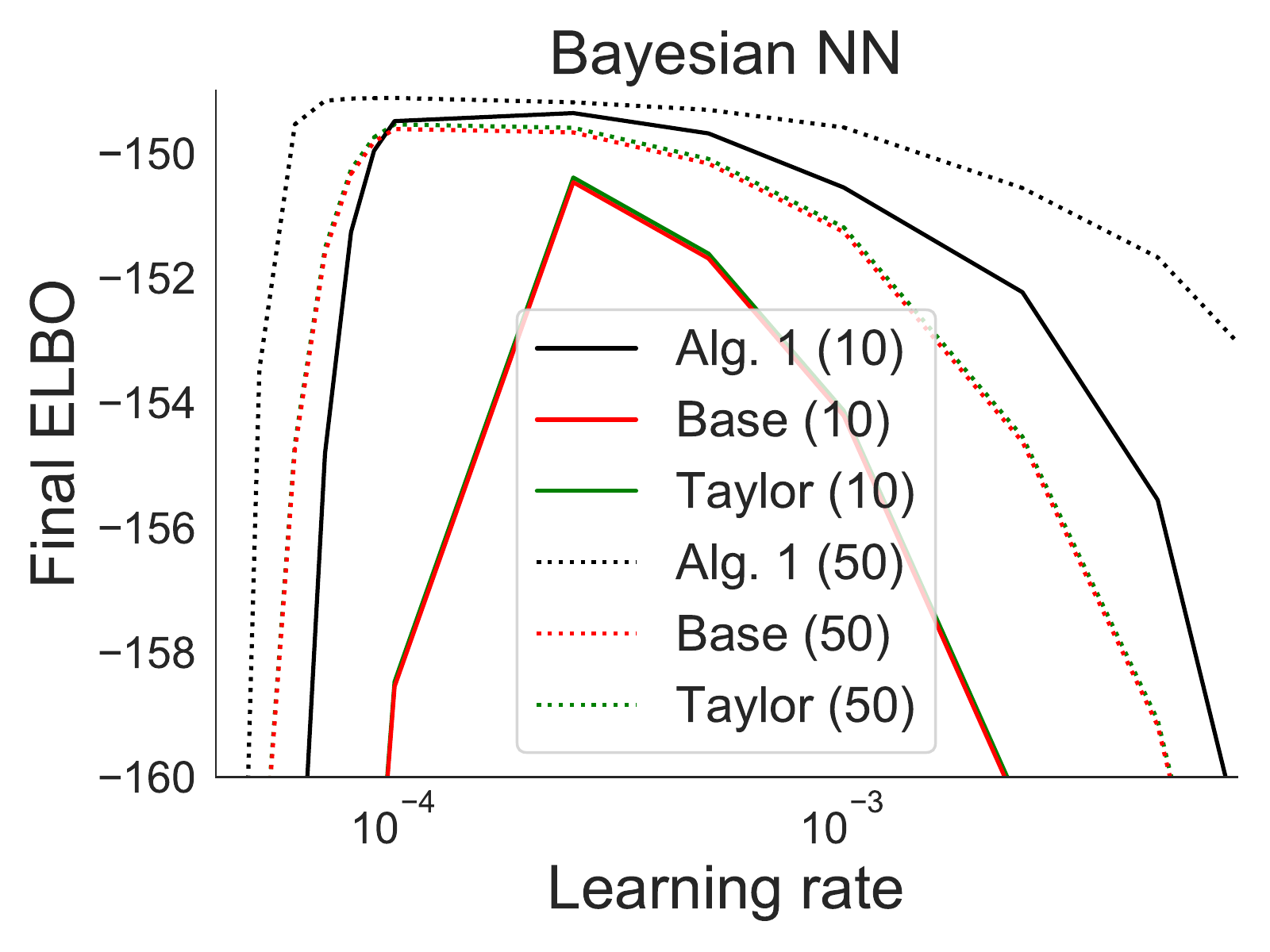}
    \caption{\textbf{The use of our control variate yields good results for a wider range of step sizes.} VI using a diagonal plus low rank covariance Gaussian variational distribution. The plots show the final ELBO achieved after training for 80000 steps vs. step size used. (Higher ELBO is better.)}
    \label{fig:per1}
    \end{center}
\end{figure}

For space reasons, results for diagonal Gaussians and Gaussians with arbitrary full-rank covariances as variational distributions are shown in Appendix \ref{app:resultsdiag}. Results for full-rank Gaussians are similar to the ones shown in Figs.~\ref{fig:opt1} and \ref{fig:per1}. Our method performs considerably better than competing approaches (our method with $M = 10$ outperforms competing approaches with $M = 50$), and Taylor-based control variates yield no improvement over the no control variate baseline.

On the other hand, with diagonal Gaussians, our approach and Taylor-based control variates perform similarly -- both are significantly better than the no control variate baseline. We attribute the success of Taylor-based approaches in this case to two related factors. First, diagonal approximations tend to under-estimate the true variance, so a local Taylor approximation may be more effective. Second, for diagonal Gaussians most of the gradient variance comes from mean parameters, where a second-order Taylor approach is tractable.

\subsubsection{Estimator's Variance as Optimization Proceeds}

Fig.~\ref{fig:results_underover_s4} showed a comparison of the variance reduction achieved by our control variate in the ideal setting for which the control variate's parameters $v$ were fully optimized at every step. While insightful, the analysis did not reflect how the control variate is used in practice, where the parameters $v$ "track" the optimal parameters as $w$ changes via a double-descent scheme. We now show results for this practical setting. We set the variational distribution to be a diagonal plus low rank Gaussian and perform optimization with each of the gradient estimators. We estimate the variance of each estimator as optimization proceeds. Results are shown in Fig.~\ref{fig:variance}. It can be observed that our control variate yields variance reductions of several orders of magnitude, while Taylor-based control variates lead to barely any variance reduction at all. 
This is aligned with our previous analysis and results.

\begin{figure}[ht!]
    \begin{center}
    \includegraphics[scale=0.28,trim={0 0 0 0},clip]{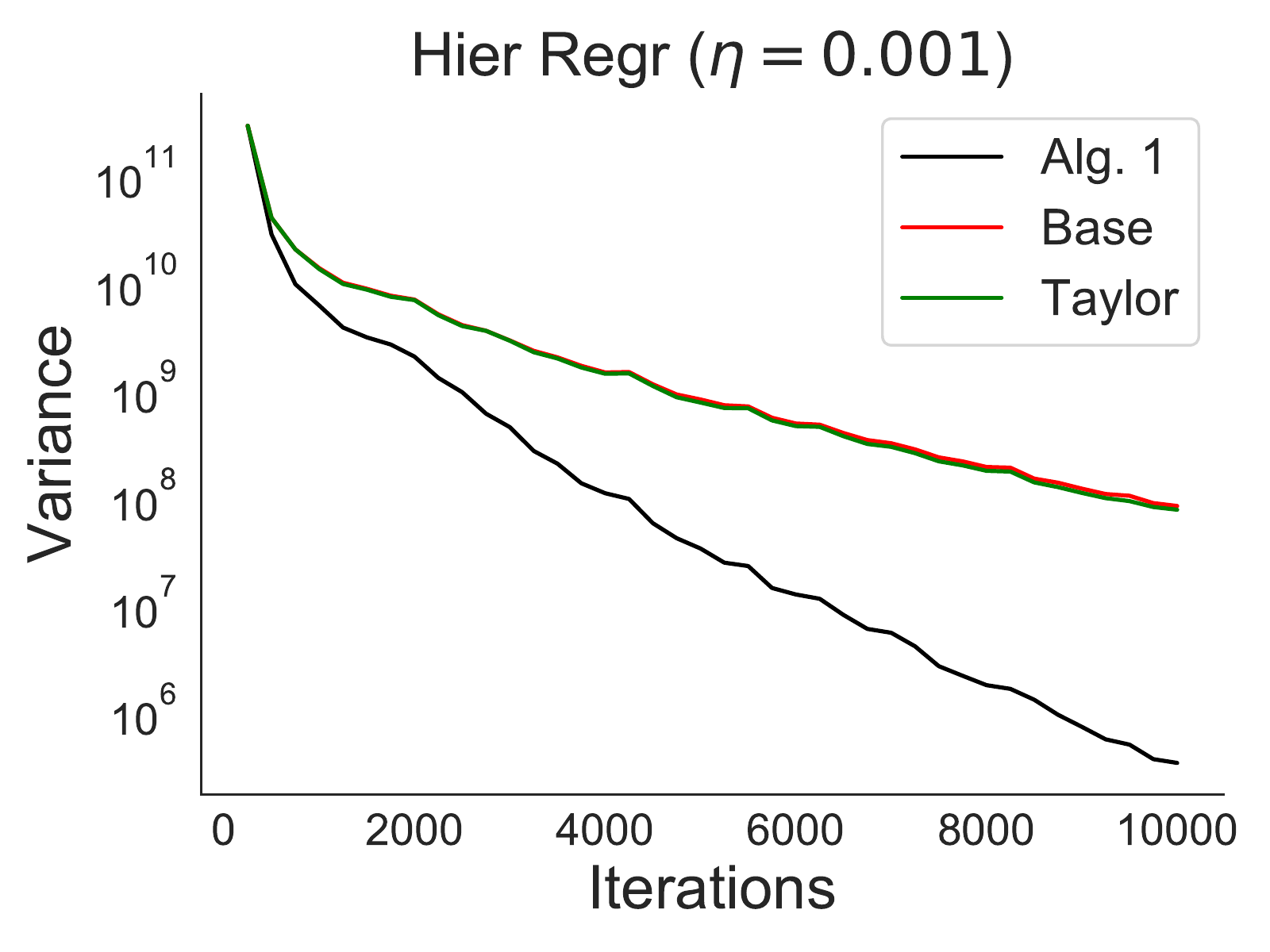}\hfill
    \includegraphics[scale=0.28,trim={1.3cm 0 0 0},clip]{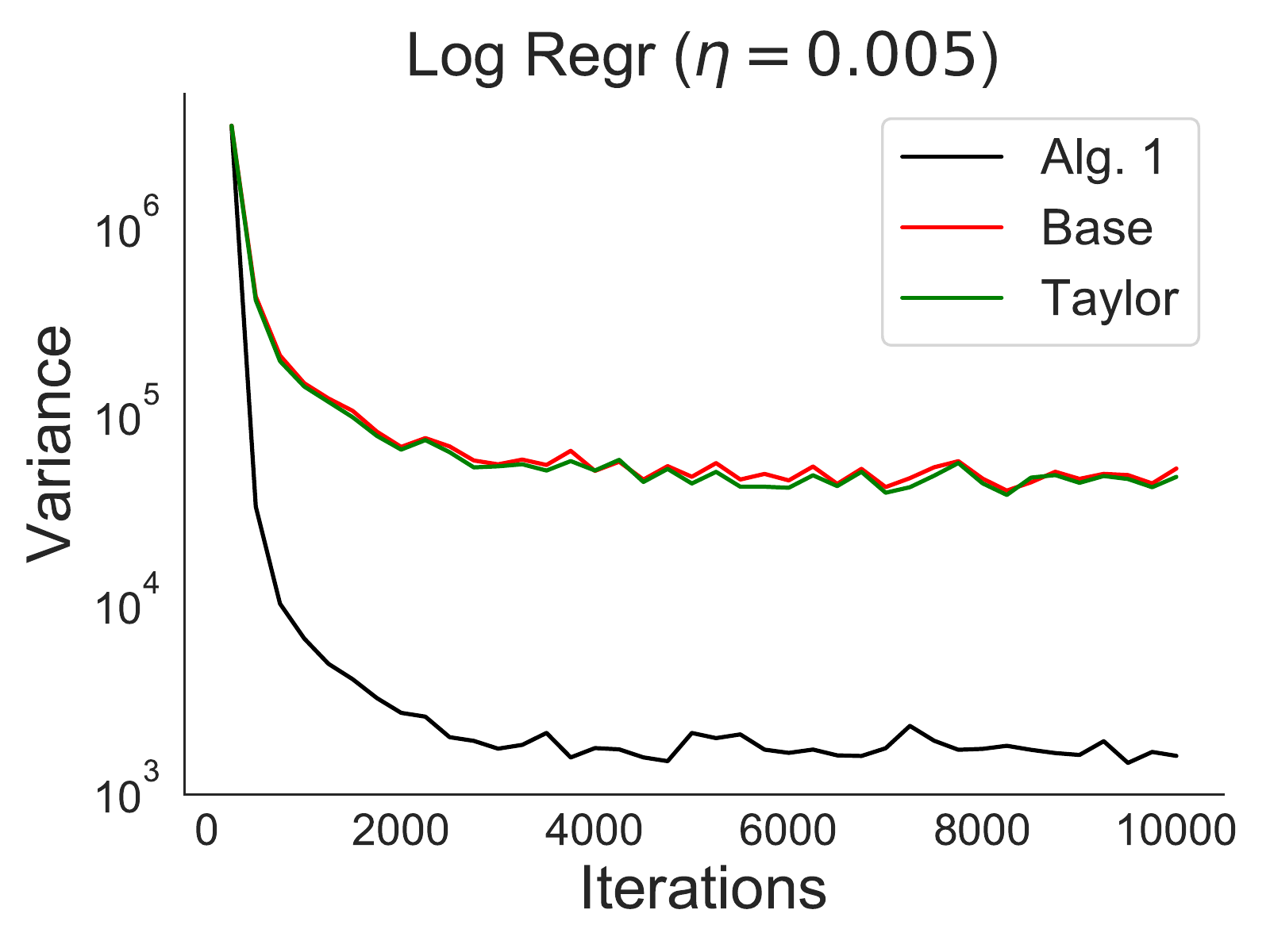}\hfill
    \includegraphics[scale=0.28,trim={1.3cm 0 0 0},clip]{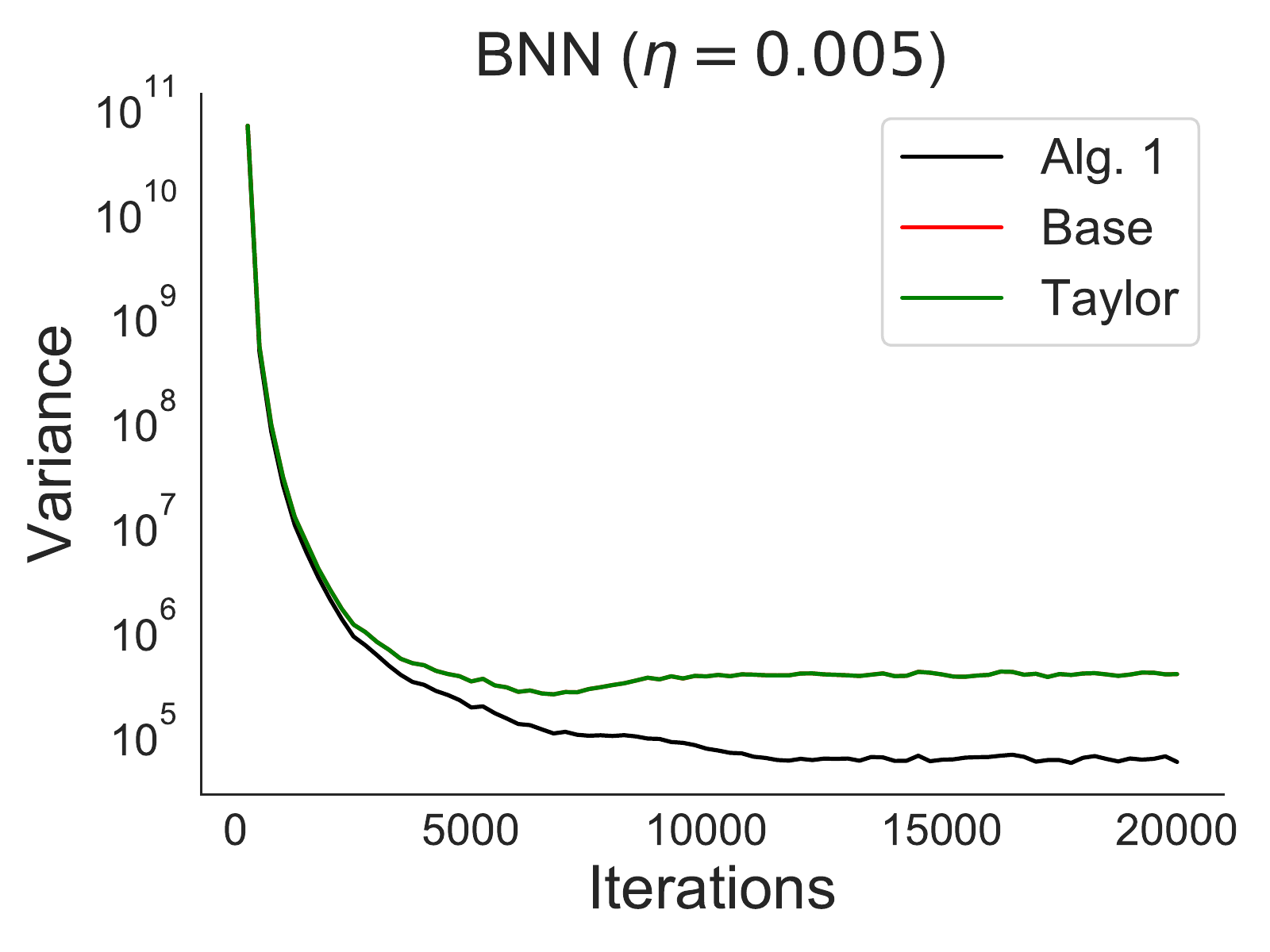}
    \caption{\textbf{The use of our control variate yields large reductions in variance.} Variance of different gradient estimators as optimization proceeds for the three models considered. "Base" stands for the base reparameterization gradient, and "Taylor" for using a Taylor-expansion based control variate for variance reduction. For the BNN model the lines for the base estimator and the Taylor control variate are almost completely overlapped, and thus indistinguishable in the plot. (All methods have the same variance at initialization because control variate weights are initialized to $0$.)}
    \label{fig:variance}
    \end{center}
\end{figure}

\subsubsection{Wall-clock Time Results}

Fig.~\ref{fig:opt1time} shows results in terms of wall-clock time instead of iterations.
These results' main purpose is visualization, they are the same as the ones in Fig.~\ref{fig:opt1} (right column) with the x-axis scaled for each estimator with the values from Table~\ref{tab:costs}.

\begin{figure}[ht!]
    \begin{center}
    \includegraphics[scale=0.28,trim={0 0 0 0},clip]{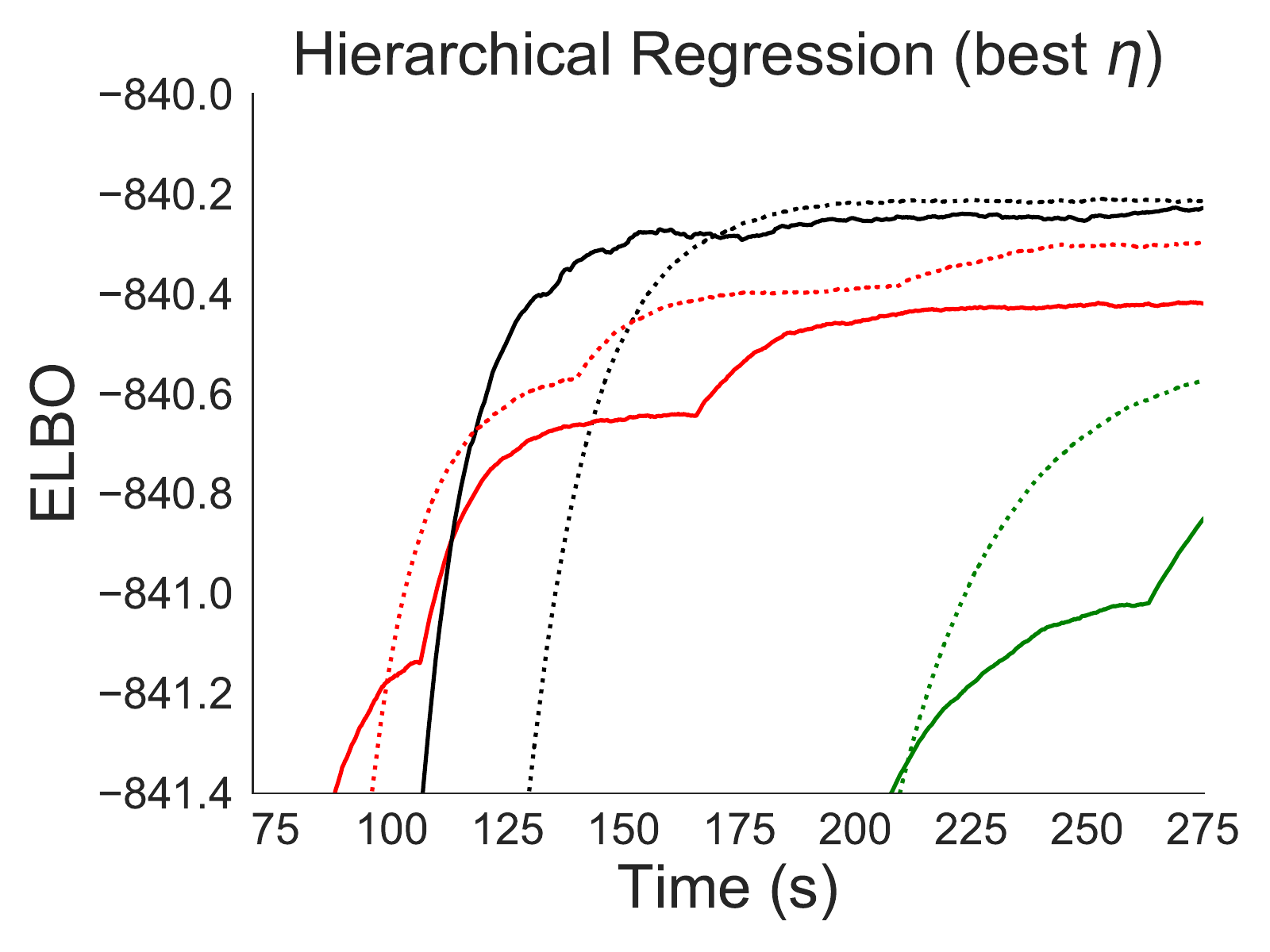}\hfill
    \includegraphics[scale=0.28,trim={1.3cm 0 0 0},clip]{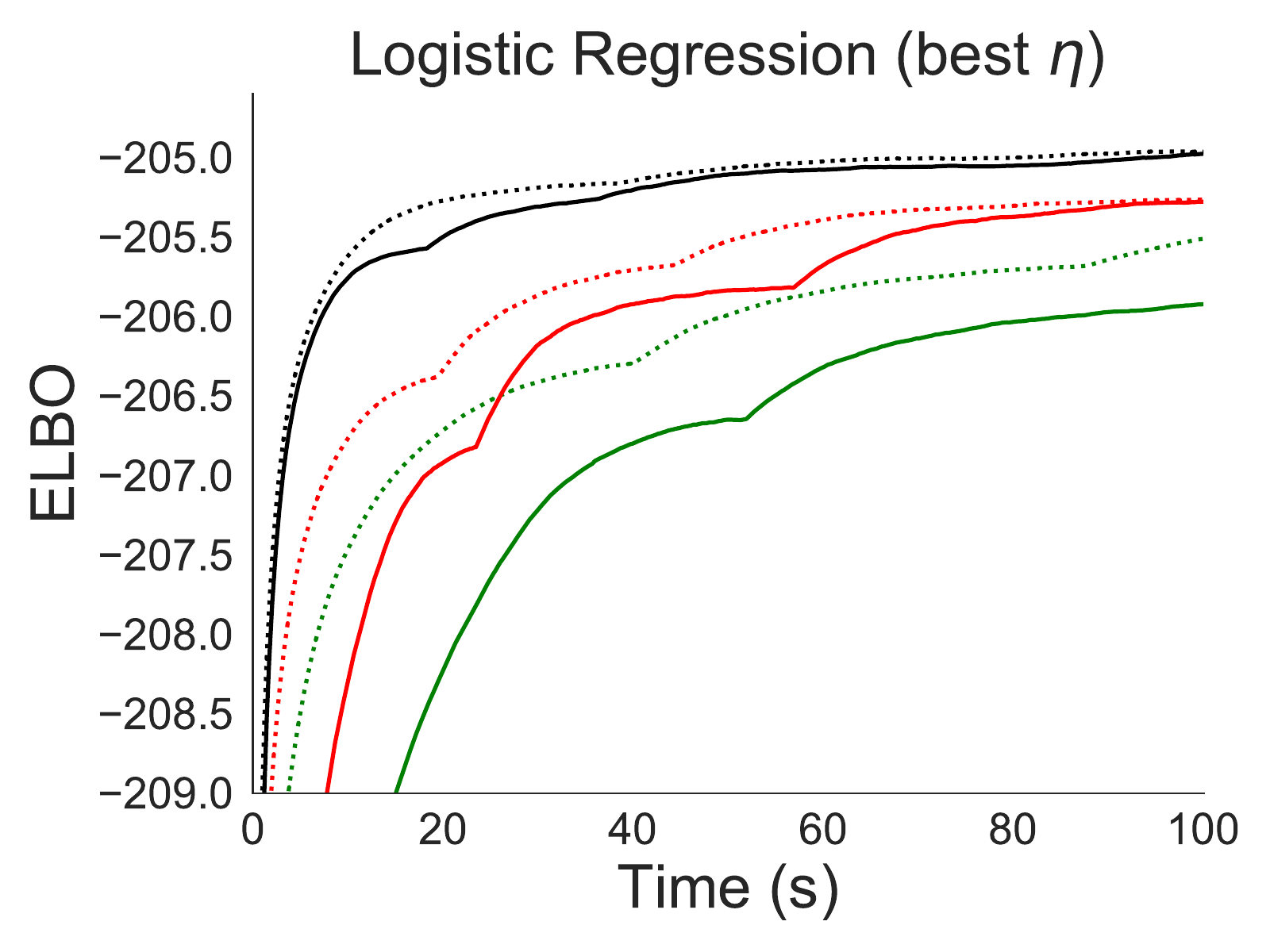}\hfill
    \includegraphics[scale=0.28,trim={1.3cm 0 0 0},clip]{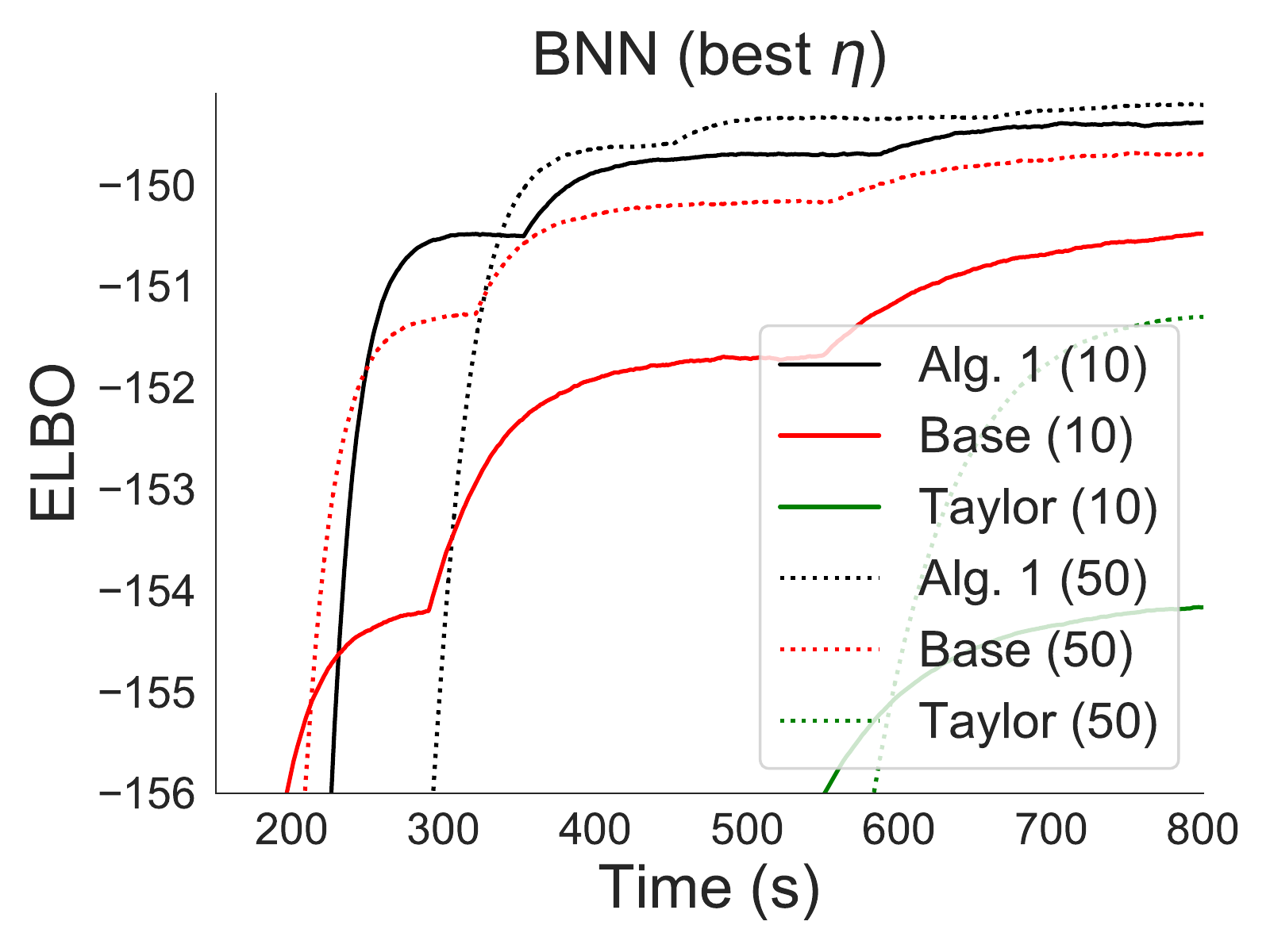}
    \caption{\textbf{The use of our control variate yields improved optimization convergence.} VI using a diagonal plus low rank Gaussian variational distribution, with the best step-size chosen retrospectively for each time horizon. "Base (M)" stands for the base reparameterization gradient estimated using $M$ samples, and "Taylor (M)" for using a Taylor-expansion based control variate for variance reduction.}
    \label{fig:opt1time}
    \end{center}
\end{figure}

Finally, it is worth mentioning that our control variate may be used jointly with other variance reduction methods. For instance, the sticking-the-landing (STL) estimator \cite{stickingthelanding} can be used with our control variate in two ways: (i) setting the base gradient estimator to be the STL estimator; and (ii) creating the ``STL control variate'' and using in concert with our control variate \cite{cvs}. In addition, while we focus on reparameterization, our control variate could be used with other estimators as well, such as the score function or generalized reparameterization \cite{generalreparam_blei}, as long as the covariance of the variational distribution is known. This is done by obtaining the second term from eq.~\ref{eq:cv} using the corresponding estimator (instead of reparameterization).

%% file: sections/impact.tex

\section*{Broader Impact}


In this work we present a new algorithm that yields improved performance for VI with non factorized distributions. We believe this algorithm could be included in VI-based automatic inference tools to improve their performance. This could have an impact in several areas since these tools, such as ADVI \cite{advi} (in Stan \cite{stan}), are used by researchers and practitioners in many different fields.

%% file: sections/appendix.tex

\appendix

\section{Results with Other Variational Distributions} \label{app:resultsdiag}

\subsection{Gaussian with Arbitrary Full-rank Covariance Variational Distribution}

\begin{figure}[ht!]
    \begin{center}
    \includegraphics[scale=0.29,trim={0 0 0 0},clip]{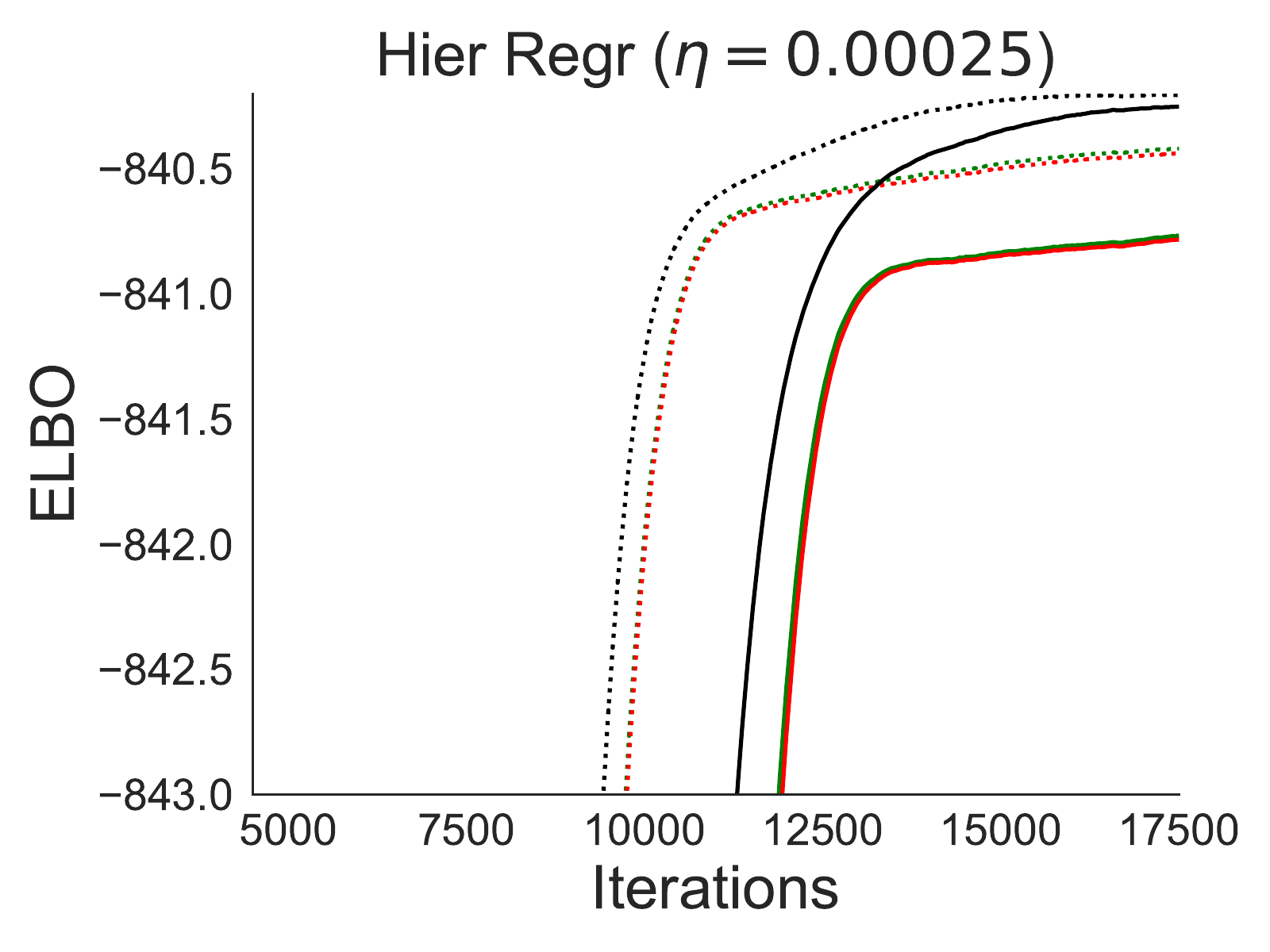}
    \includegraphics[scale=0.29,trim={3cm 0 0 0},clip]{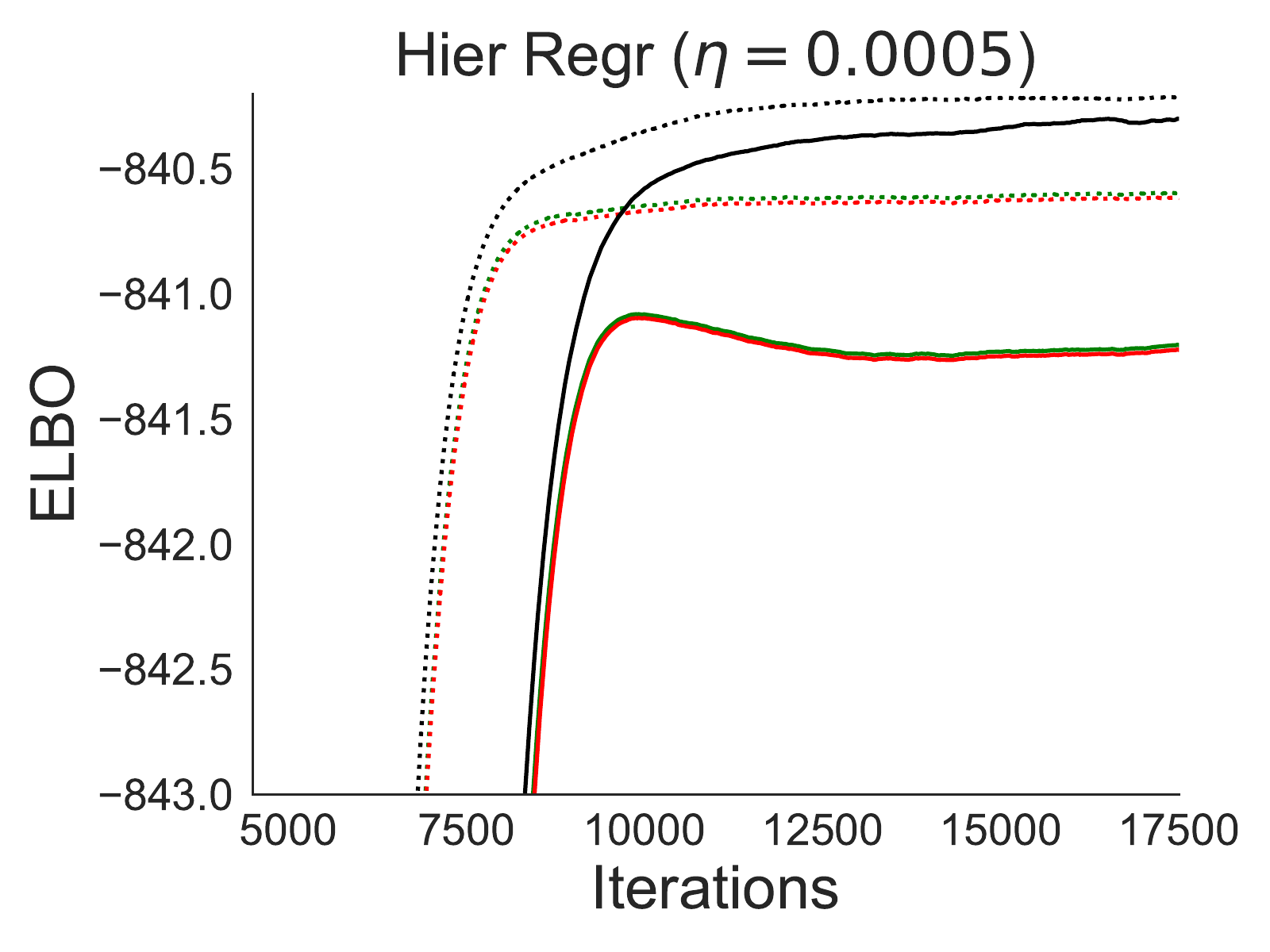}
    \includegraphics[scale=0.29,trim={3cm 0 0 0},clip]{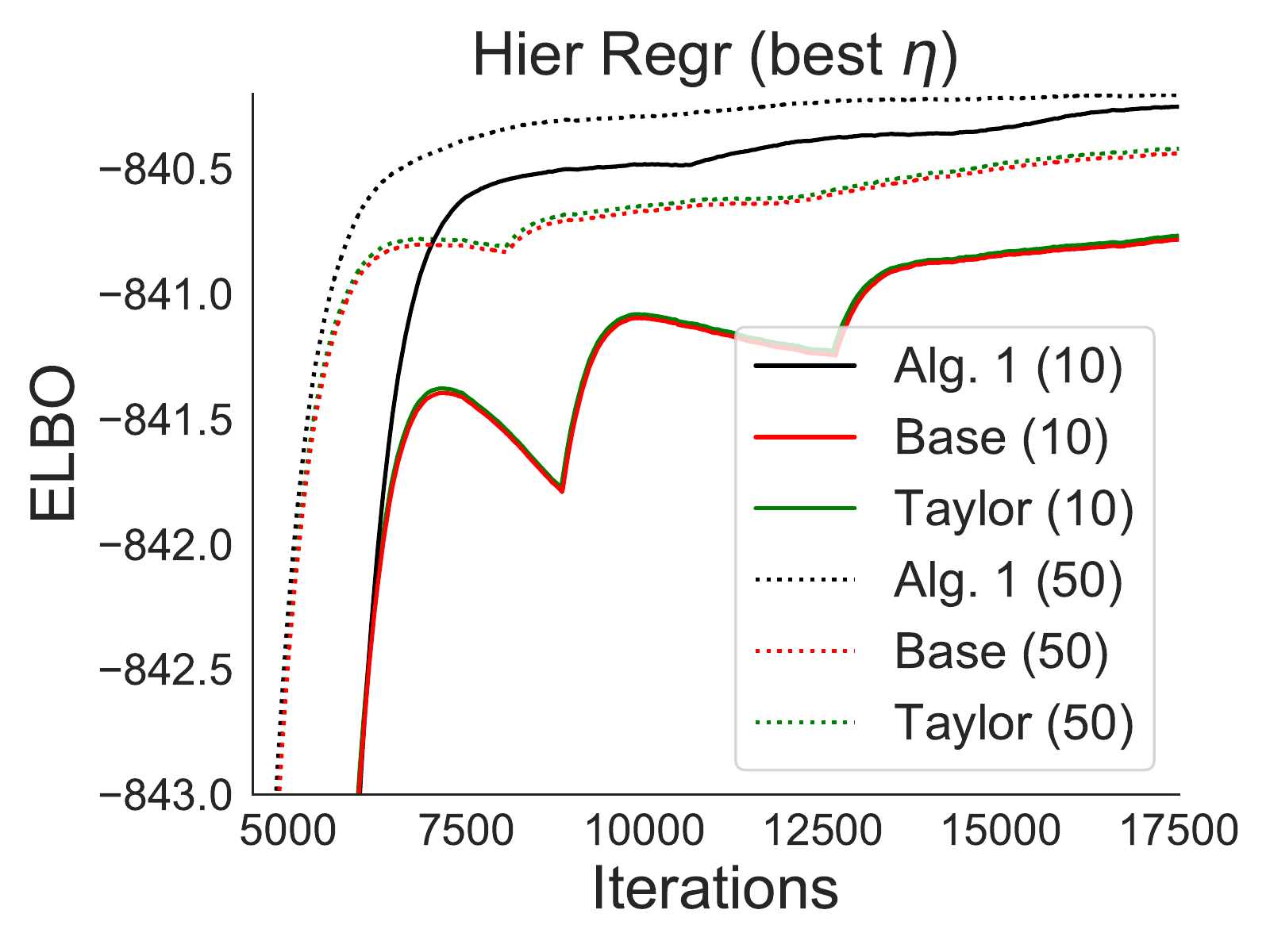}

    \includegraphics[scale=0.28,trim={0 0 0 0},clip]{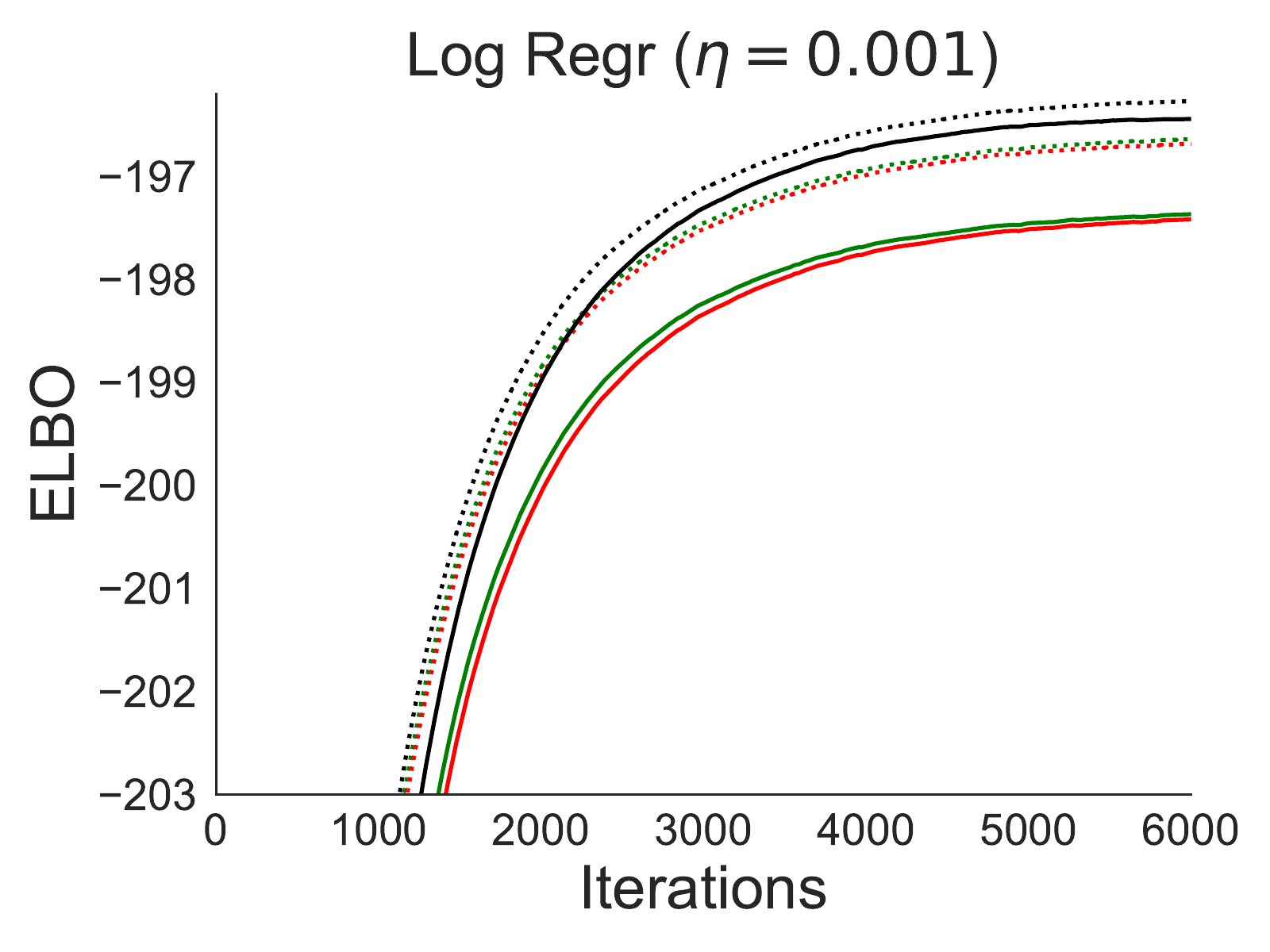}
    \includegraphics[scale=0.28,trim={2.5cm 0 0 0},clip]{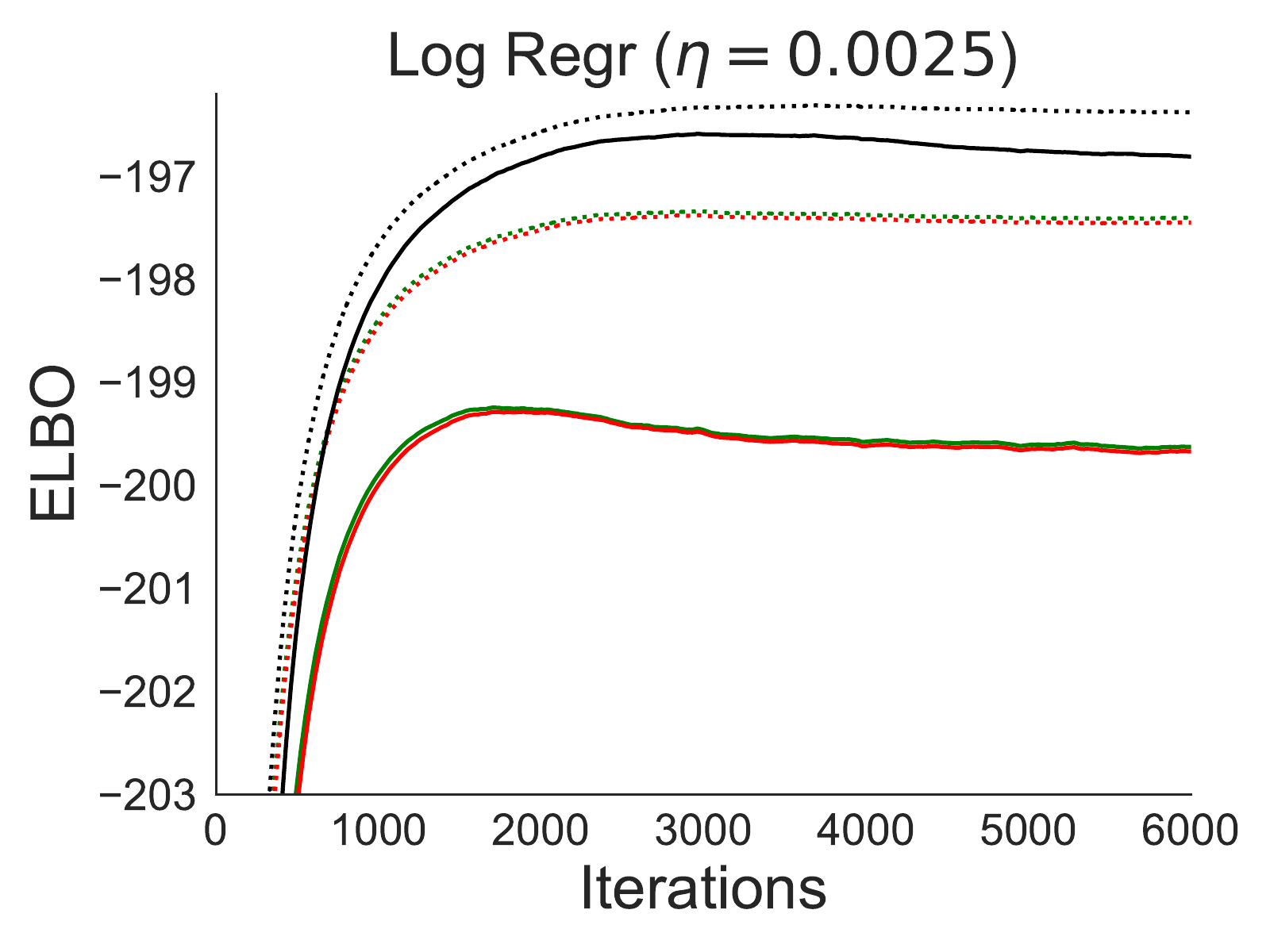}
    \includegraphics[scale=0.28,trim={2.5cm 0 0 0},clip]{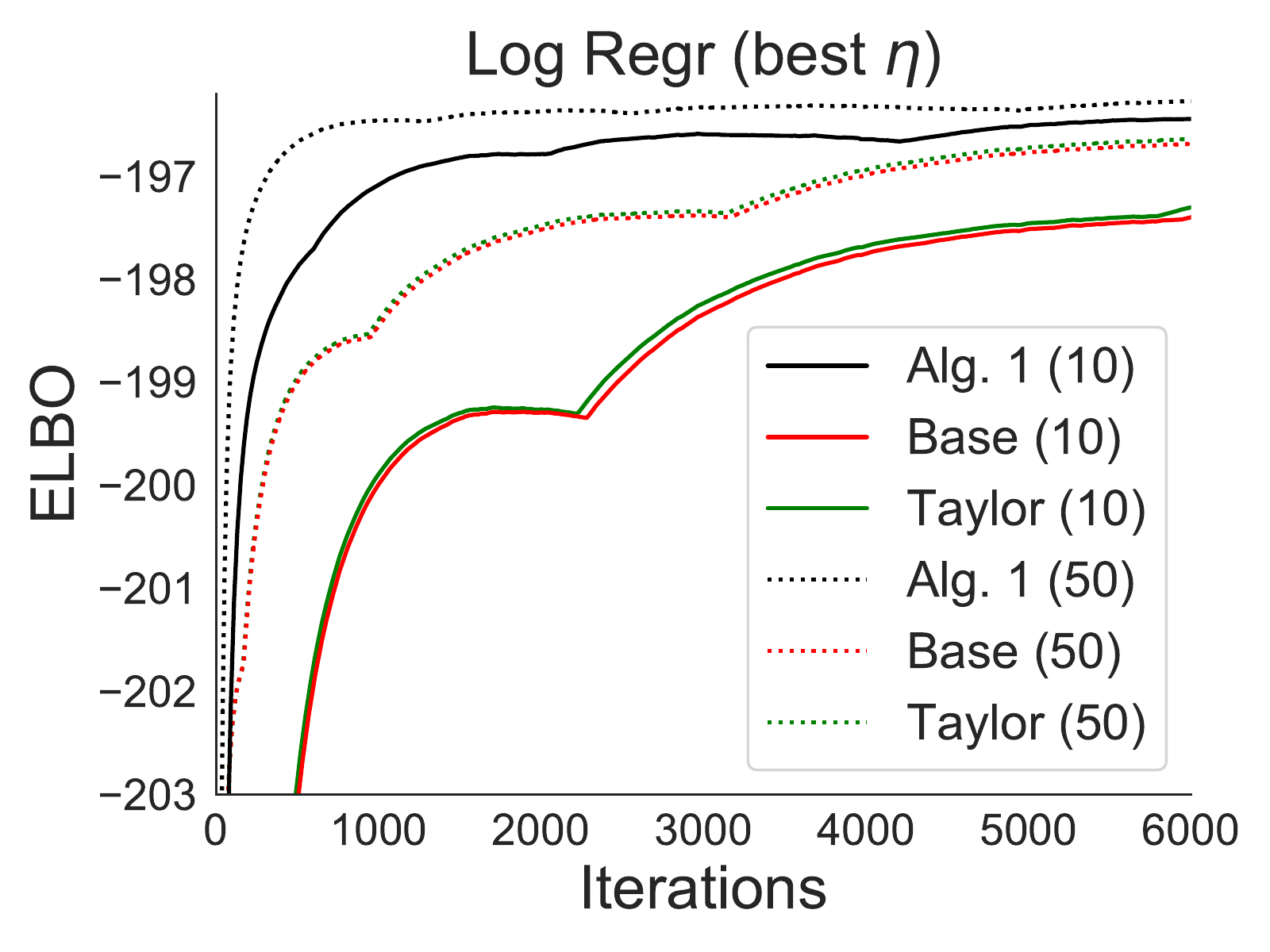}
    \caption{VI using a Gaussian with a full-rank covariance. The first two columns show results for two different step-sizes, and the third one using the best step-size chosen retrospectively. (Higher ELBO is better.)}
    \label{fig:opt22}
    \end{center}
\end{figure}

\begin{figure}[ht!]
    \begin{center}
    \includegraphics[scale=0.28,trim={0 0 0 0},clip]{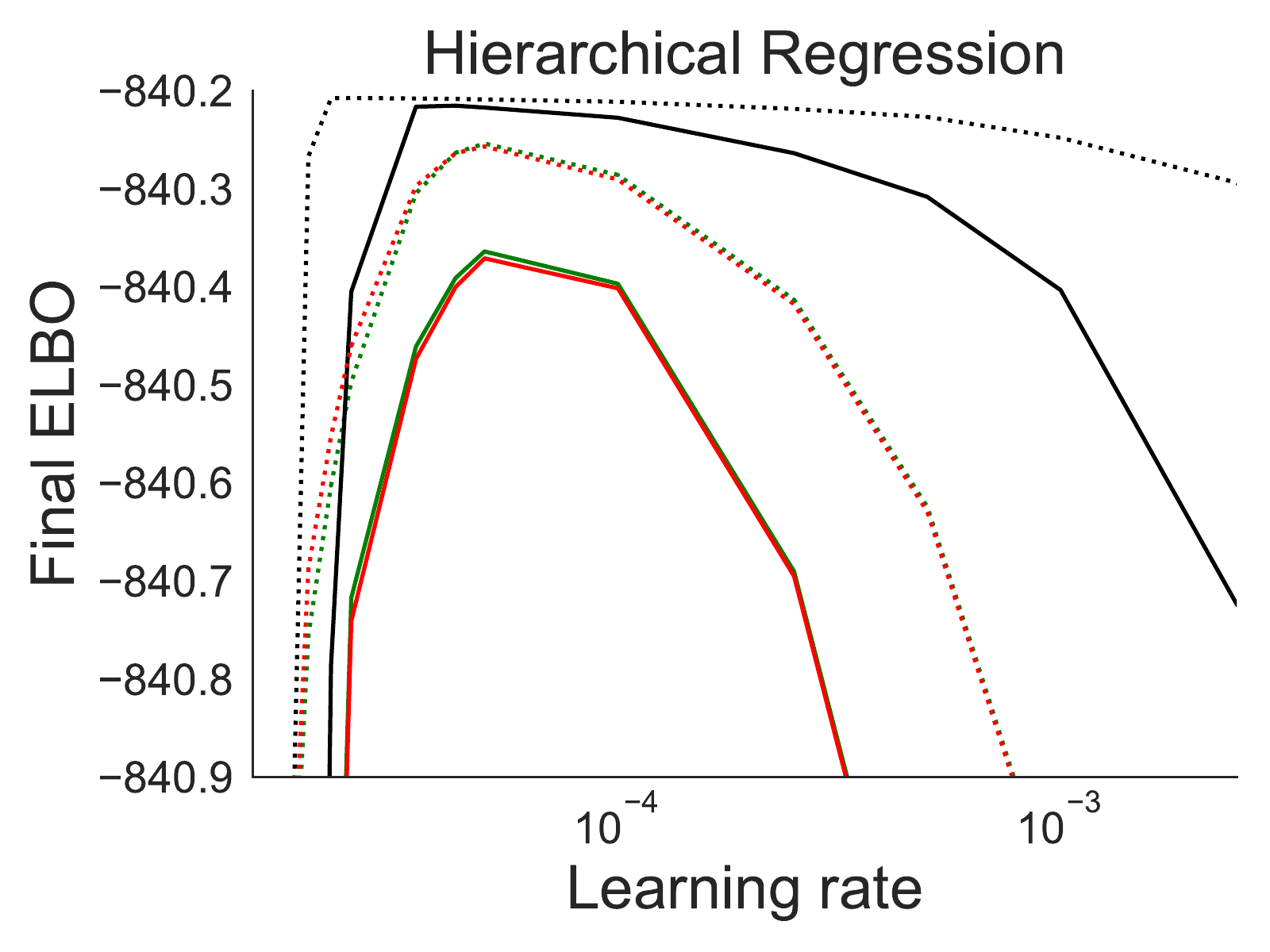}\hspace{0.7cm}
    \includegraphics[scale=0.28,trim={1.3cm 0 0 0},clip]{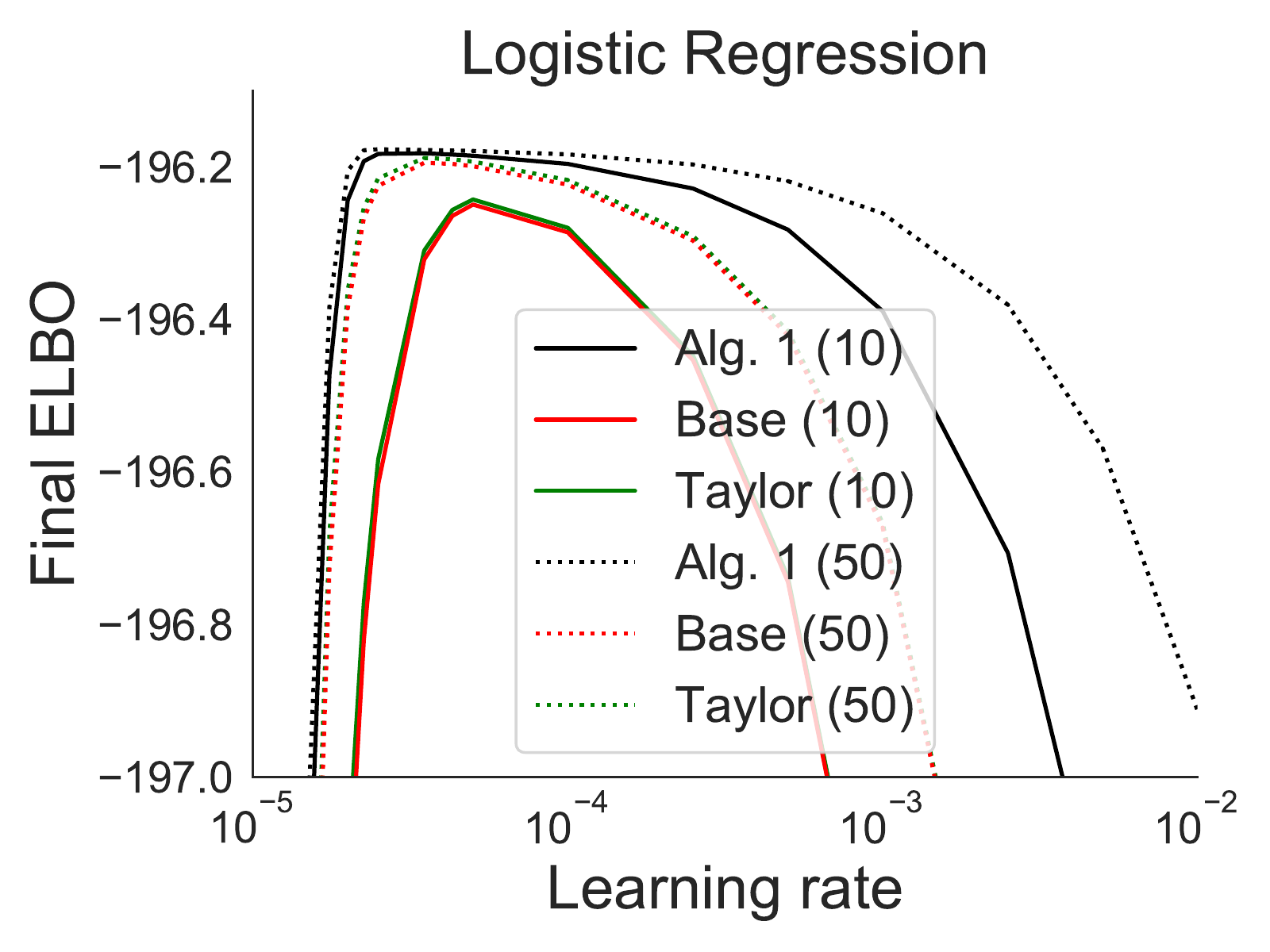}
    \caption{VI using a Gaussian with a full-rank covariance. The plots show the final ELBO achieved after training for 80000 steps vs. step size used. (Higher ELBO is better.)}
    \label{fig:per22}
    \end{center}
\end{figure}

\begin{figure}[ht!]
    \begin{center}
    \includegraphics[scale=0.28,trim={0 0 0 0},clip]{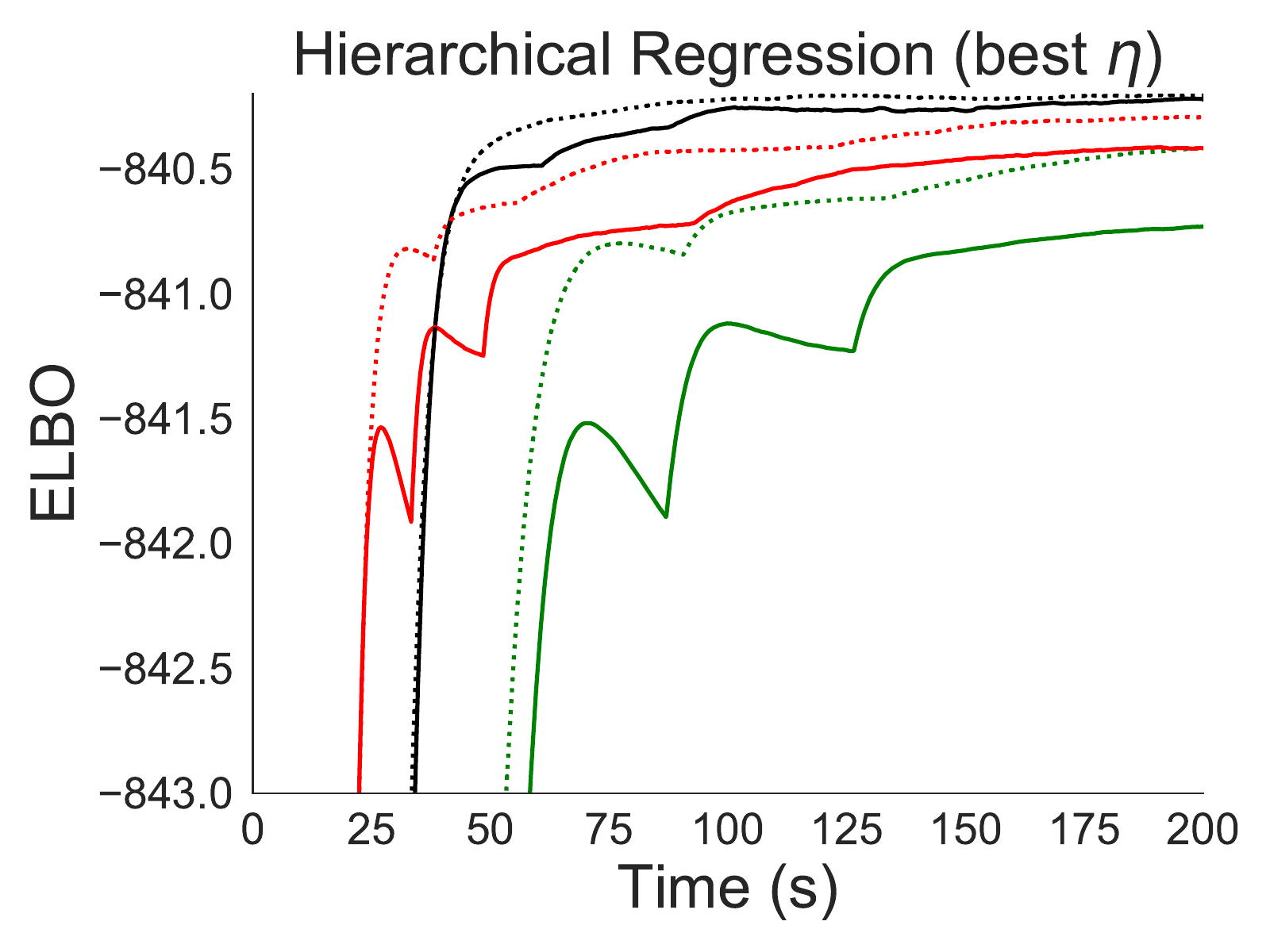}\hspace{0.7cm}
    \includegraphics[scale=0.28,trim={1.3cm 0 0 0},clip]{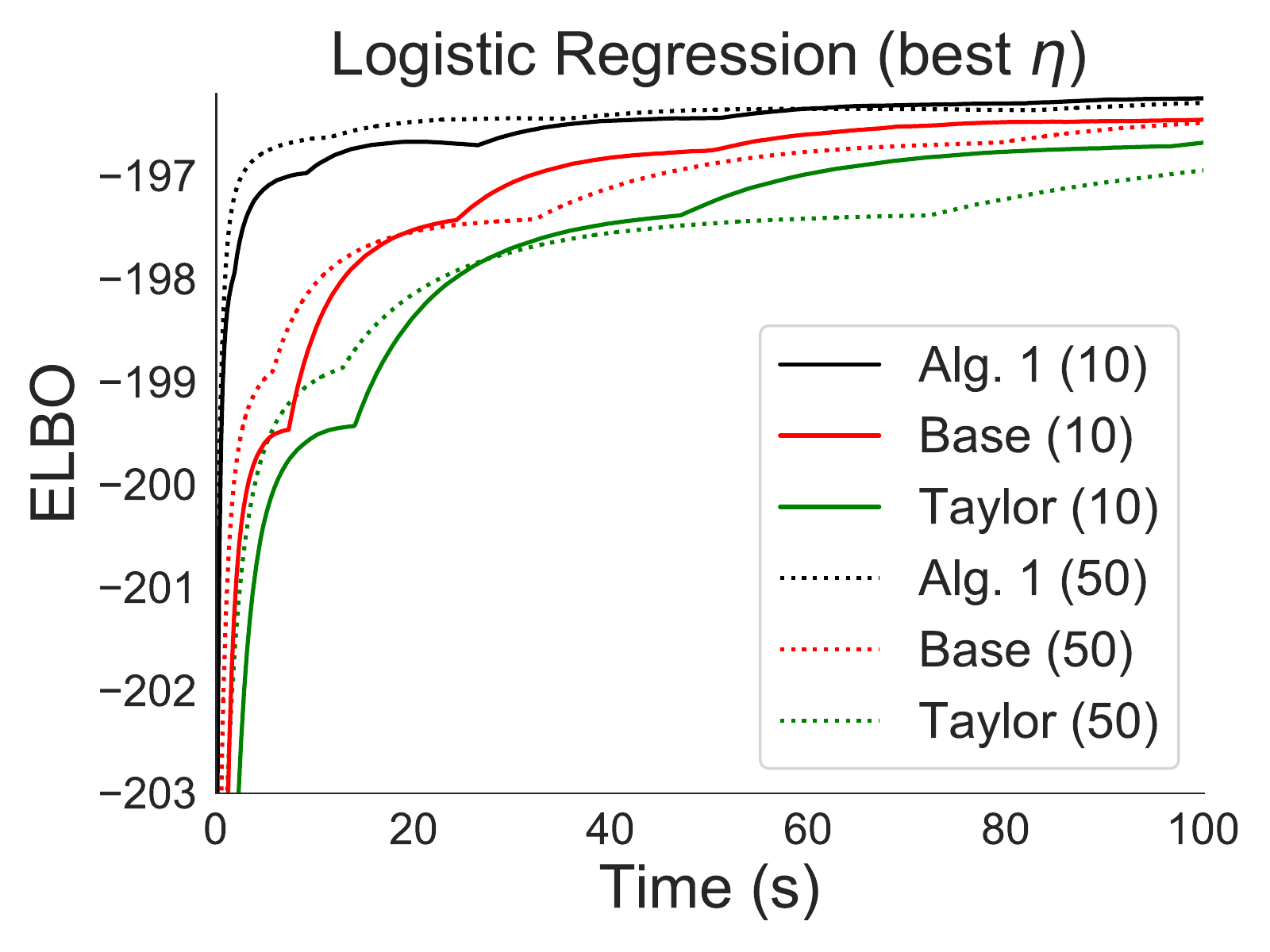}
    \caption{VI using a Gaussian with a full-rank covariance, with the best step-size chosen retrospectively. (Higher ELBO is better.)}
    \label{fig:opt2time}
    \end{center}
\end{figure}

\clearpage
\newpage

\subsection{Fully-factorized Gaussian Variational Distribution}

\begin{figure}[ht!]
    \begin{center}
    \includegraphics[scale=0.29,trim={0 0 0 0},clip]{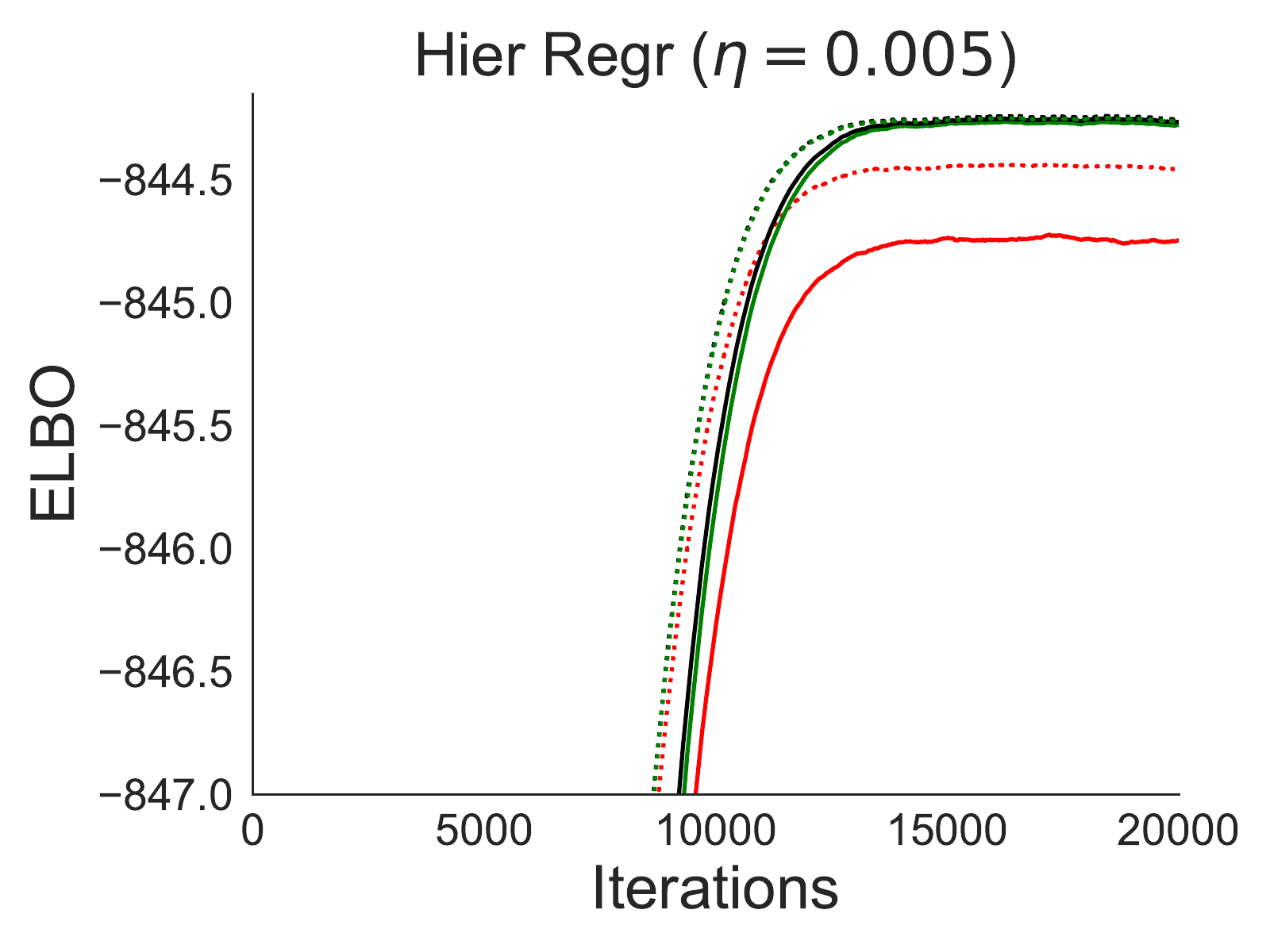}
    \includegraphics[scale=0.29,trim={3cm 0 0 0},clip]{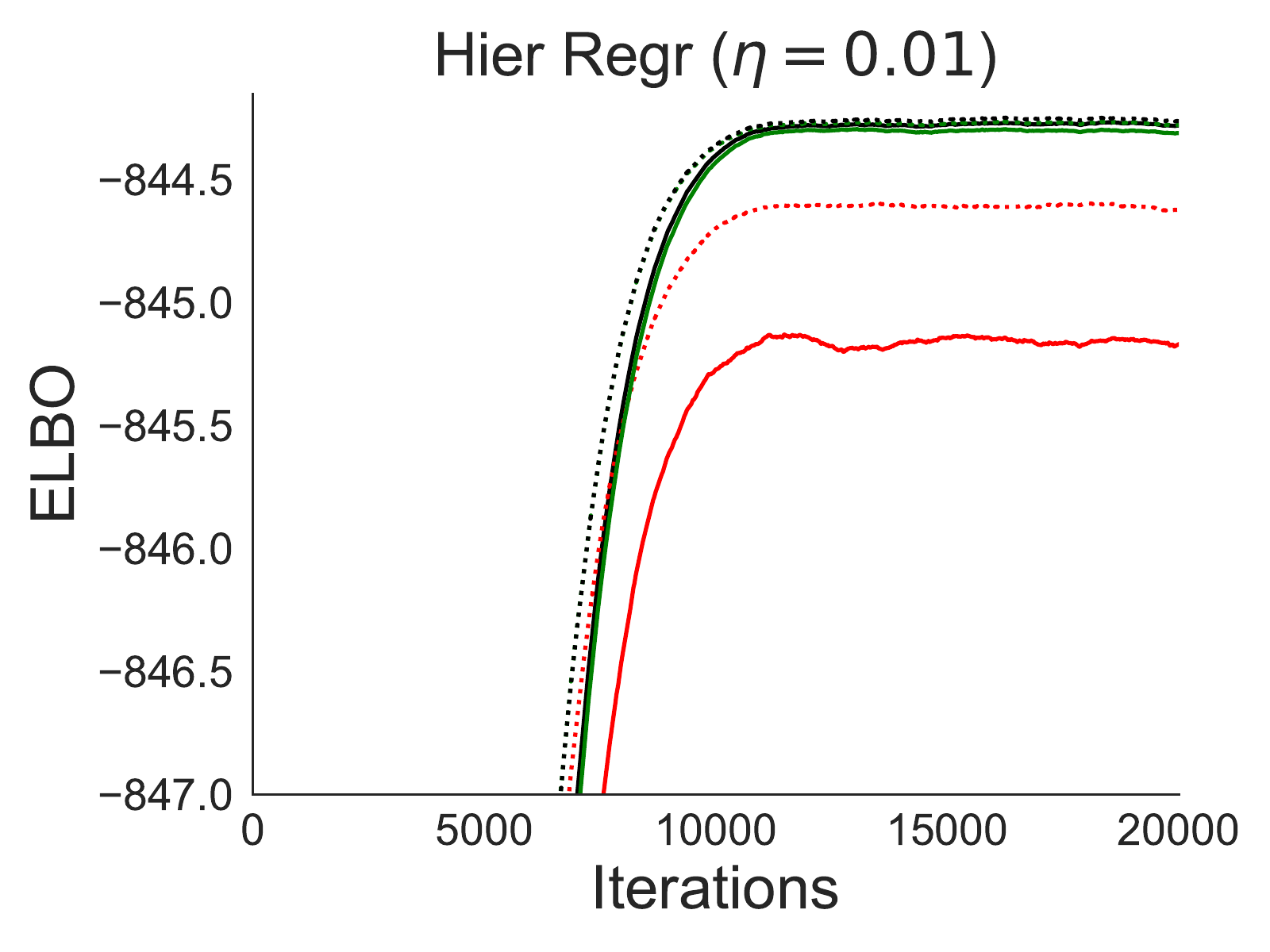}
    \includegraphics[scale=0.29,trim={3cm 0 0 0},clip]{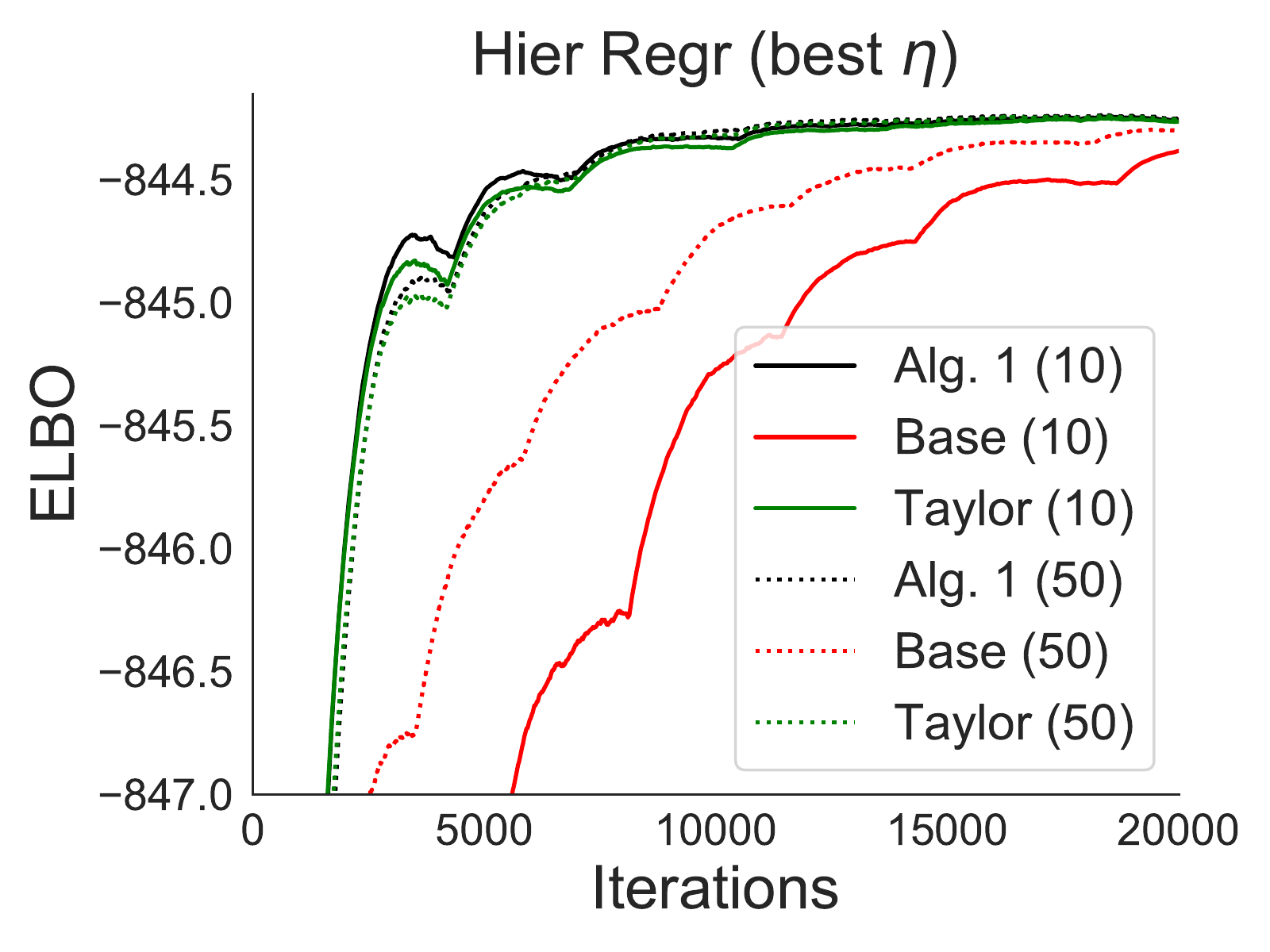}

    \includegraphics[scale=0.28,trim={0 0 0 0},clip]{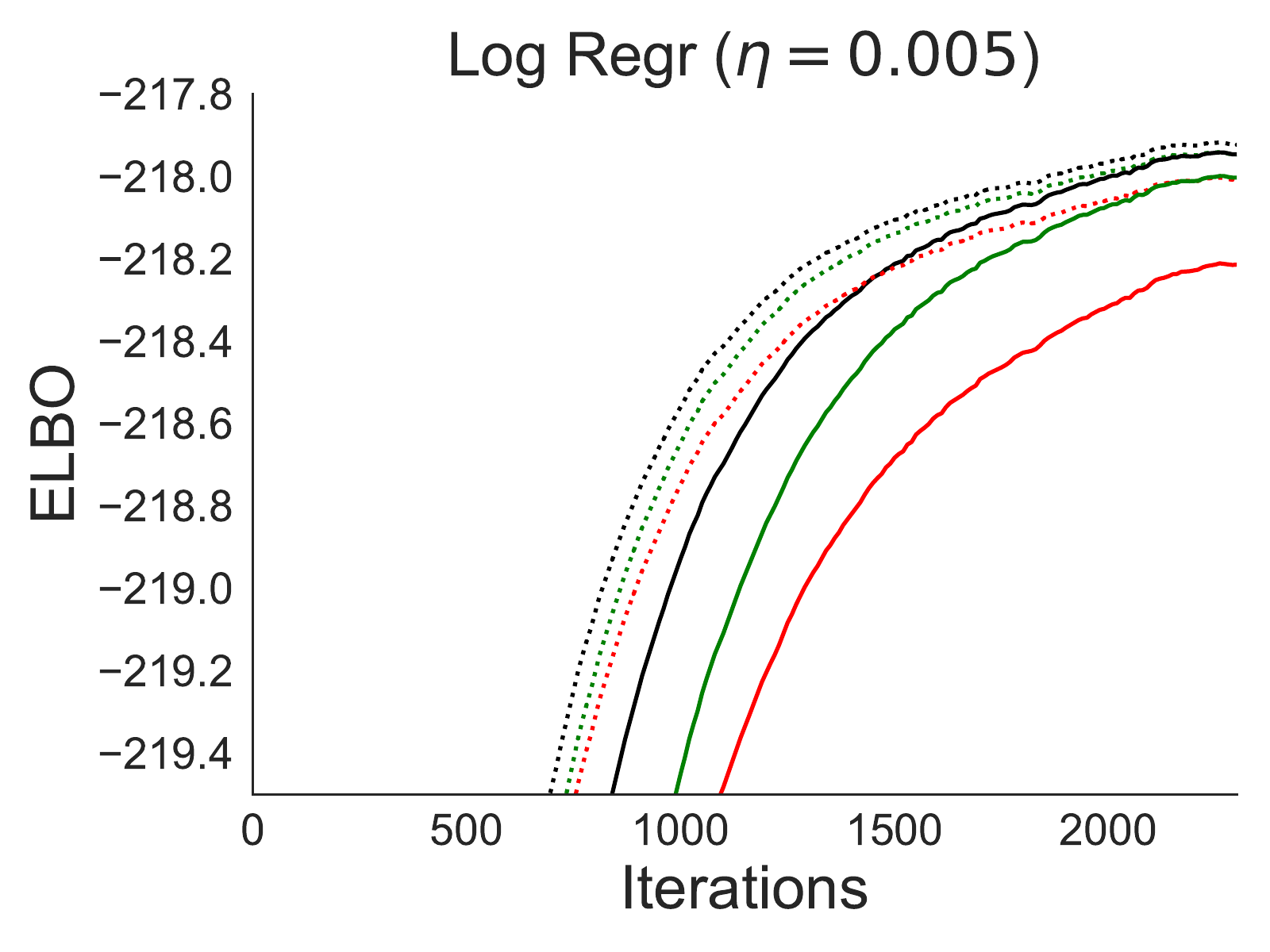}
    \includegraphics[scale=0.28,trim={3cm 0 0 0},clip]{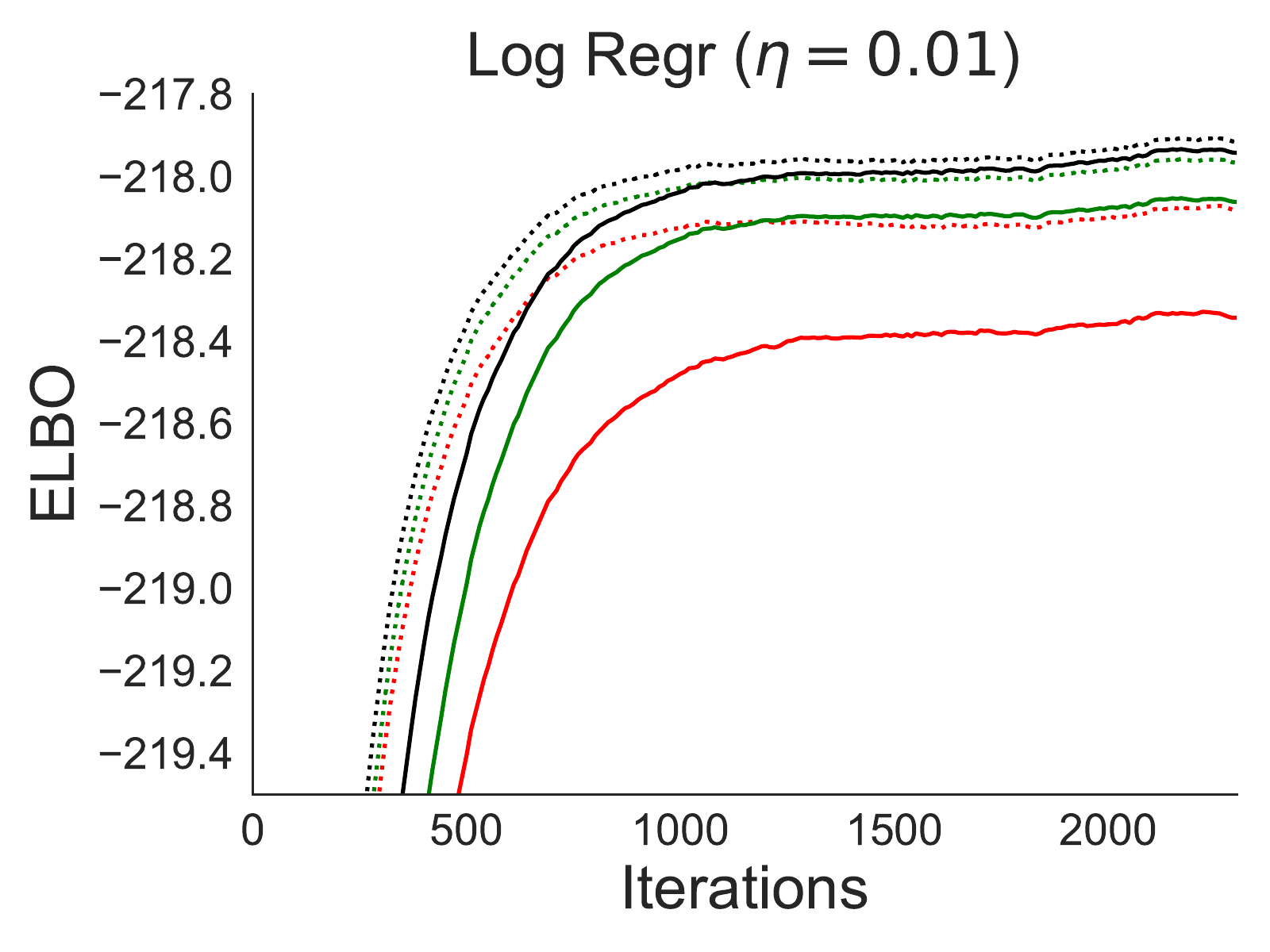}
    \includegraphics[scale=0.28,trim={3cm 0 0 0},clip]{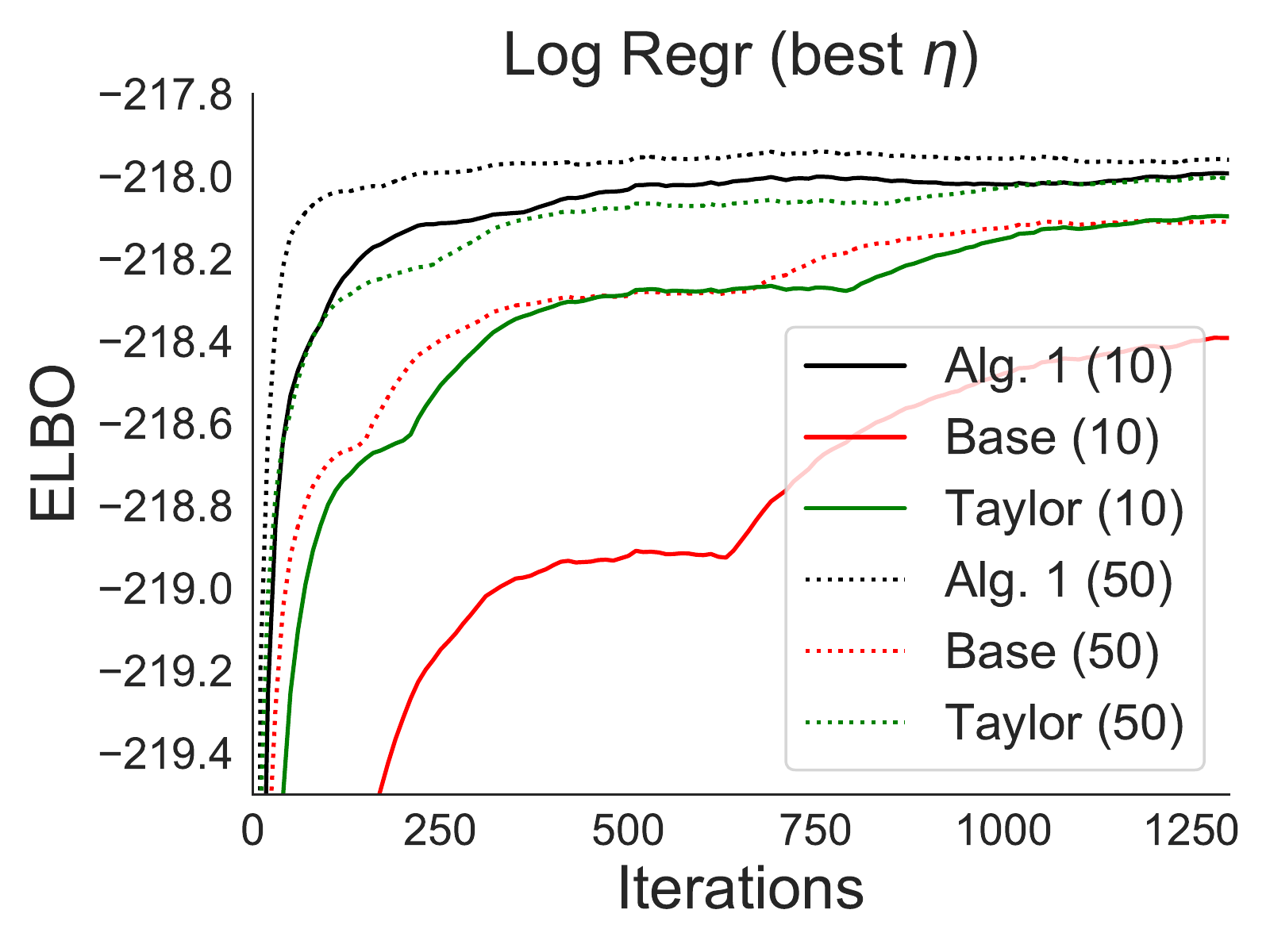}


    \includegraphics[scale=0.28,trim={0 0 0 0},clip]{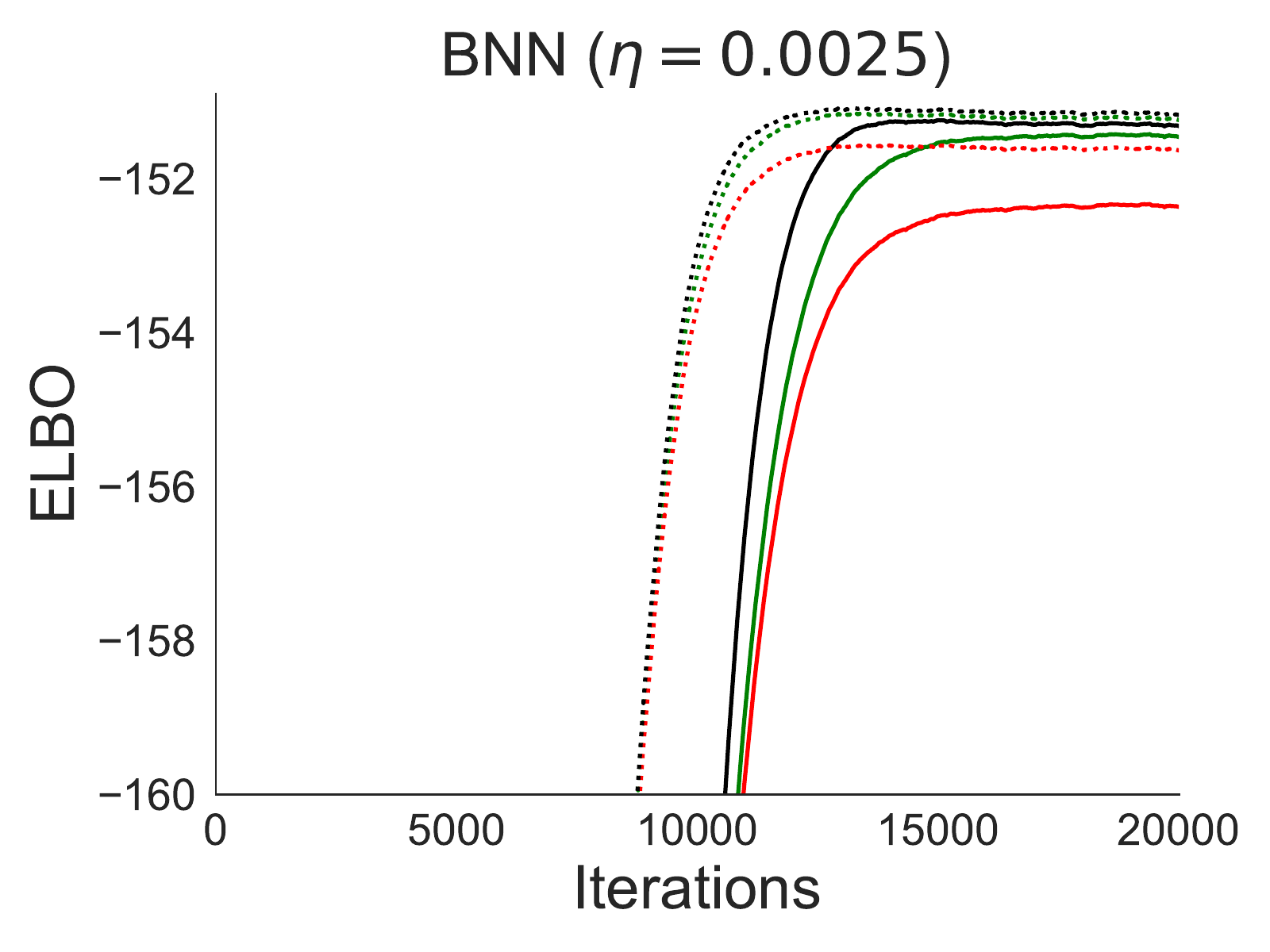}
    \includegraphics[scale=0.28,trim={2.5cm 0 0 0},clip]{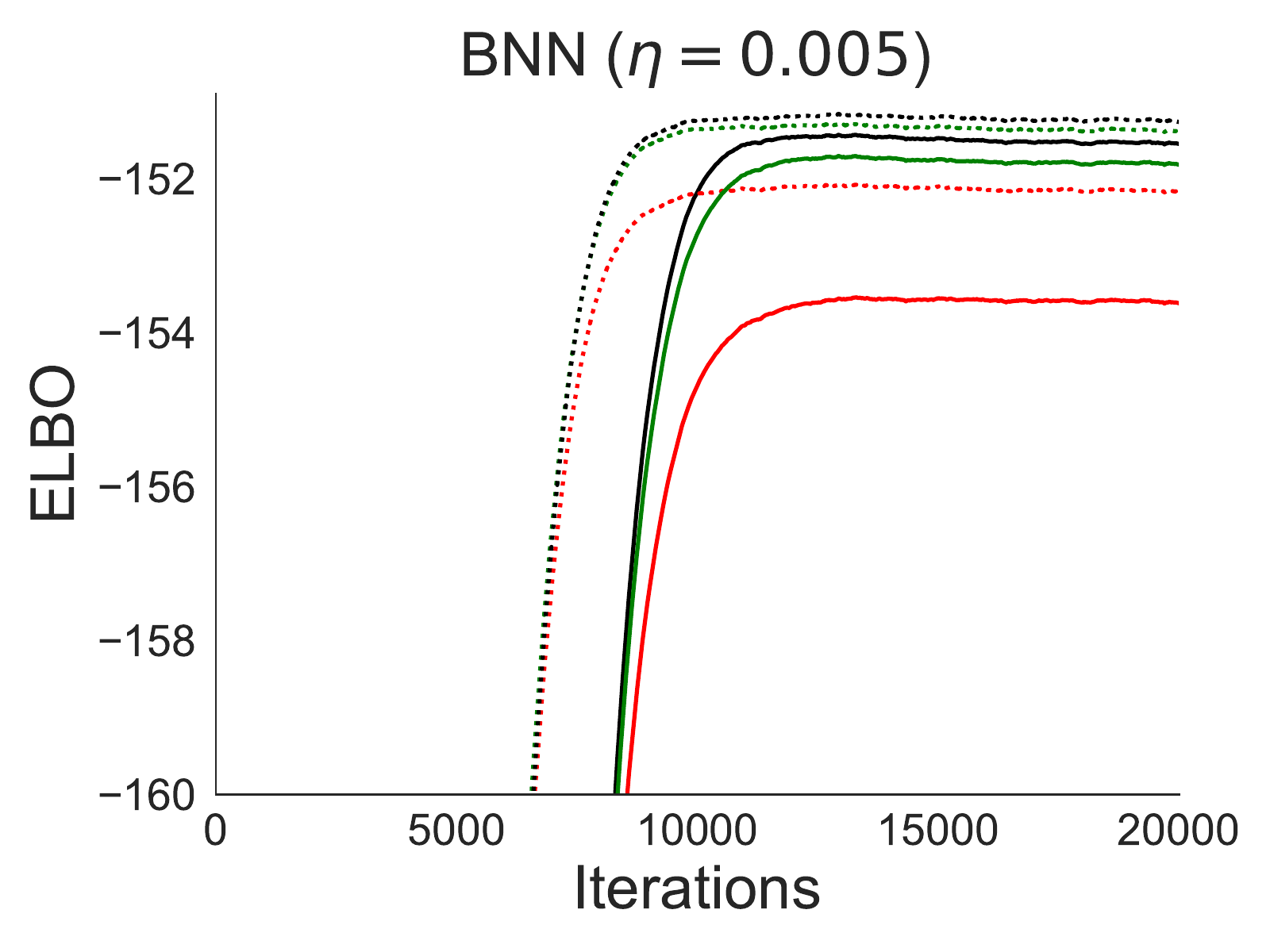}
    \includegraphics[scale=0.28,trim={2.5cm 0 0 0},clip]{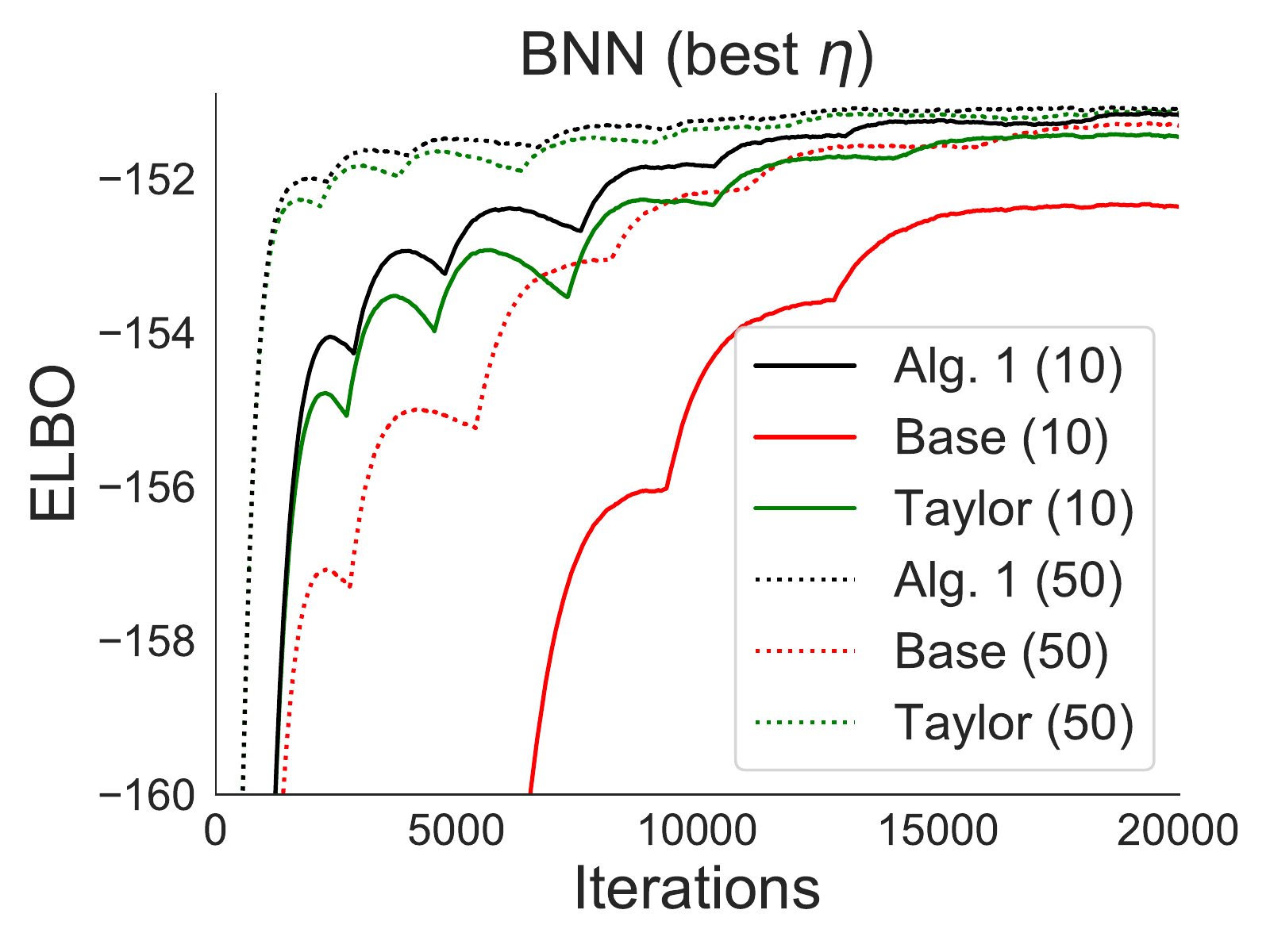}
    \caption{VI using a fully-factorized Gaussian. The first two columns show results for two different step-sizes, and the third one using the best step-size chosen retrospectively. (Higher ELBO is better.)}
    \label{fig:opt33}
    \end{center}
\end{figure}

\begin{figure}[ht!]
    \begin{center}
    \includegraphics[scale=0.26,trim={0 0 0 0},clip]{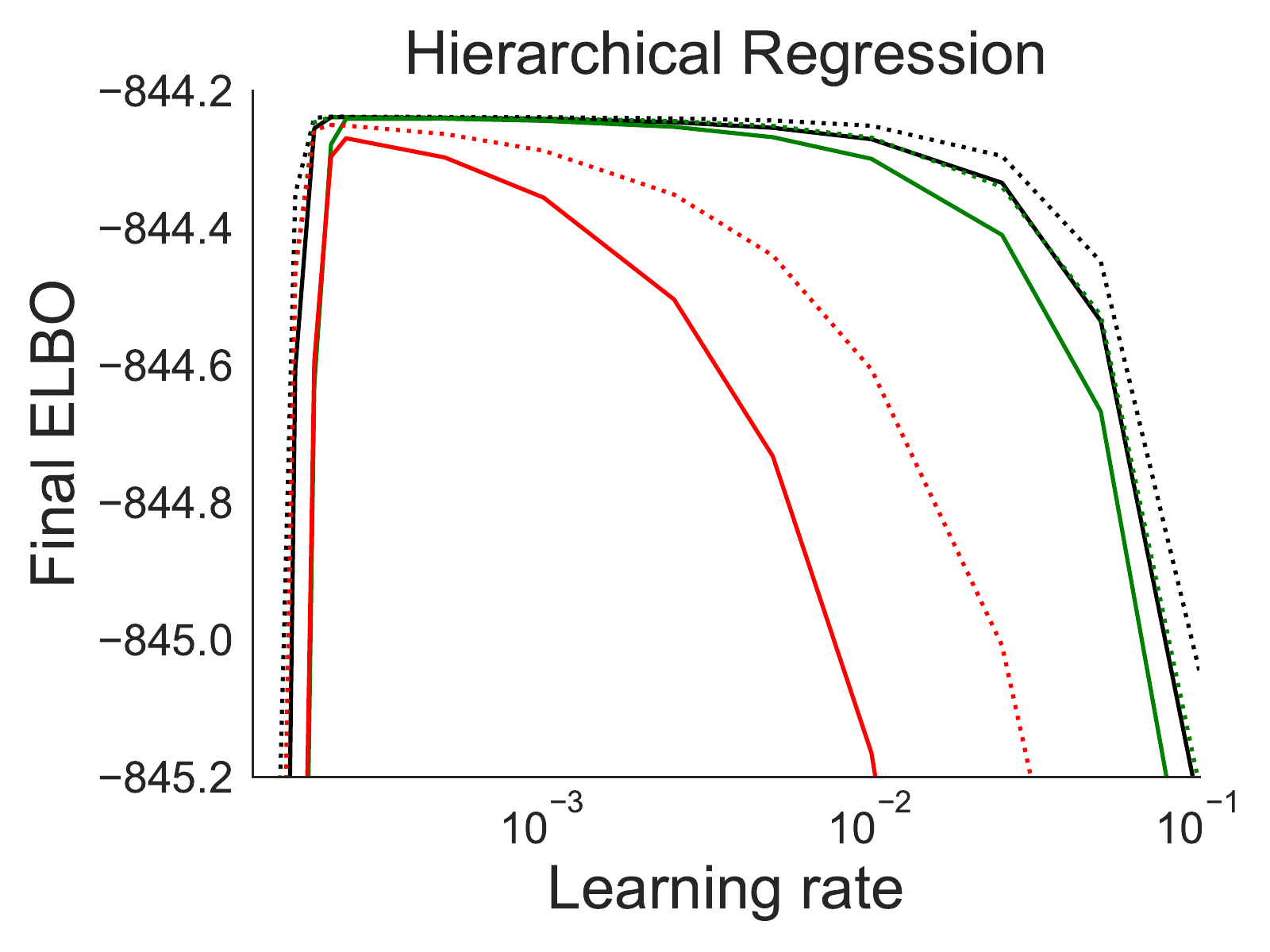}
    \includegraphics[scale=0.26,trim={1.3cm 0 0 0},clip]{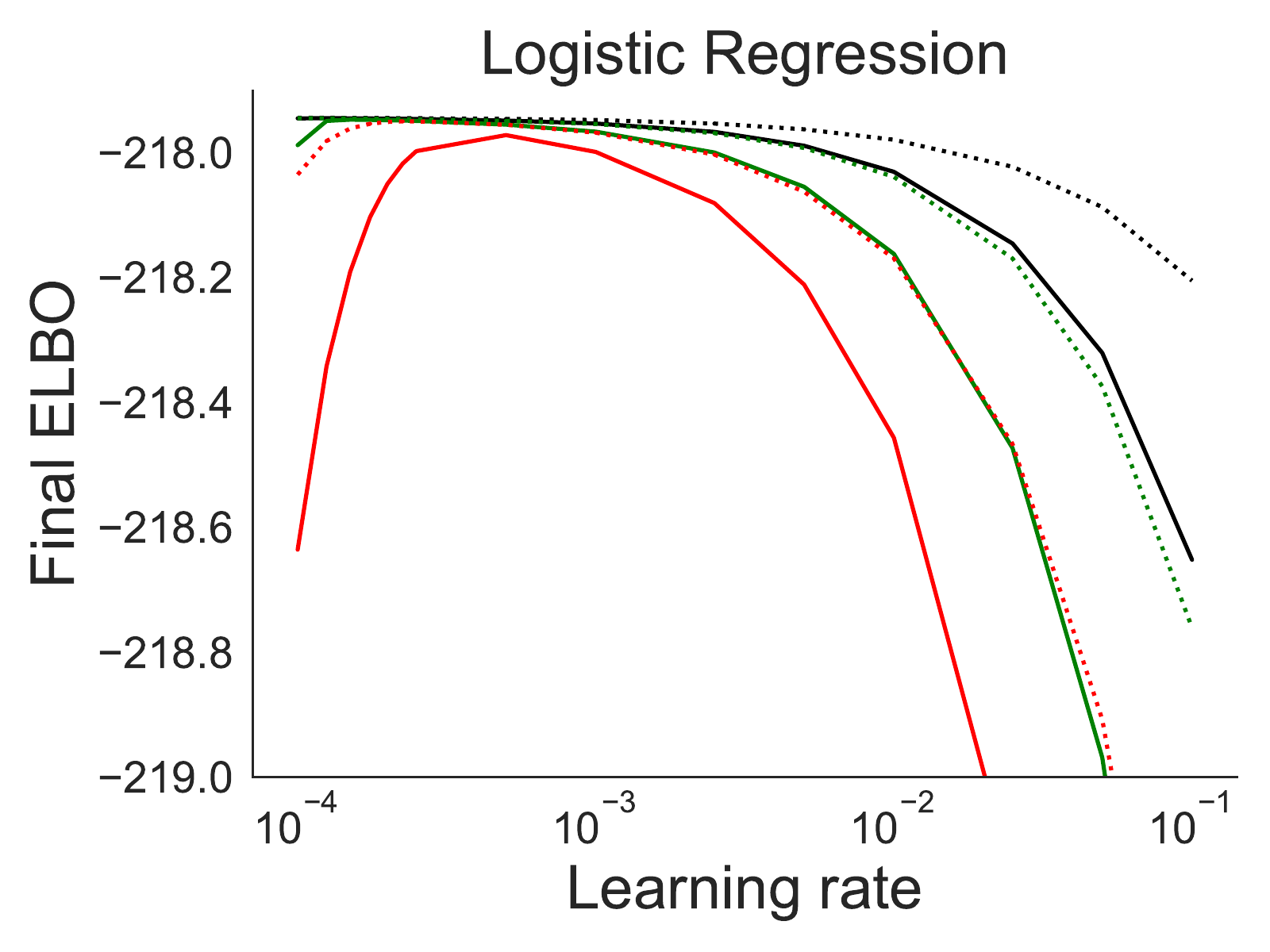}
    \includegraphics[scale=0.26,trim={1.3cm 0 0 0},clip]{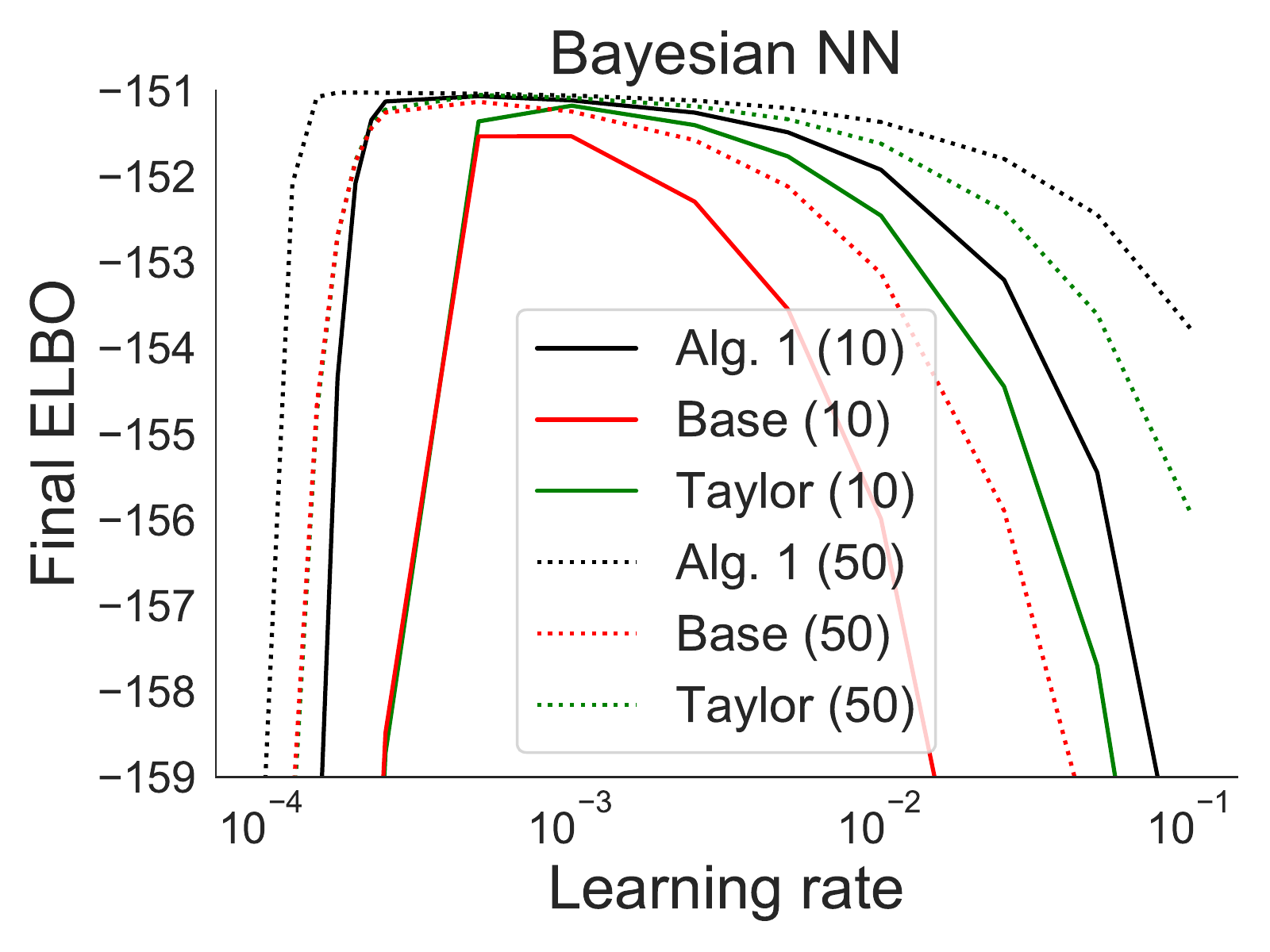}
    \caption{VI using a fully-factorized Gaussian. The plots show the final ELBO achieved after training for 40000 steps vs. step size used. (Higher ELBO is better.)}
    \label{fig:per33}
    \end{center}
\end{figure}

\begin{figure}[ht!]
    \begin{center}
    \includegraphics[scale=0.26,trim={0 0 0 0},clip]{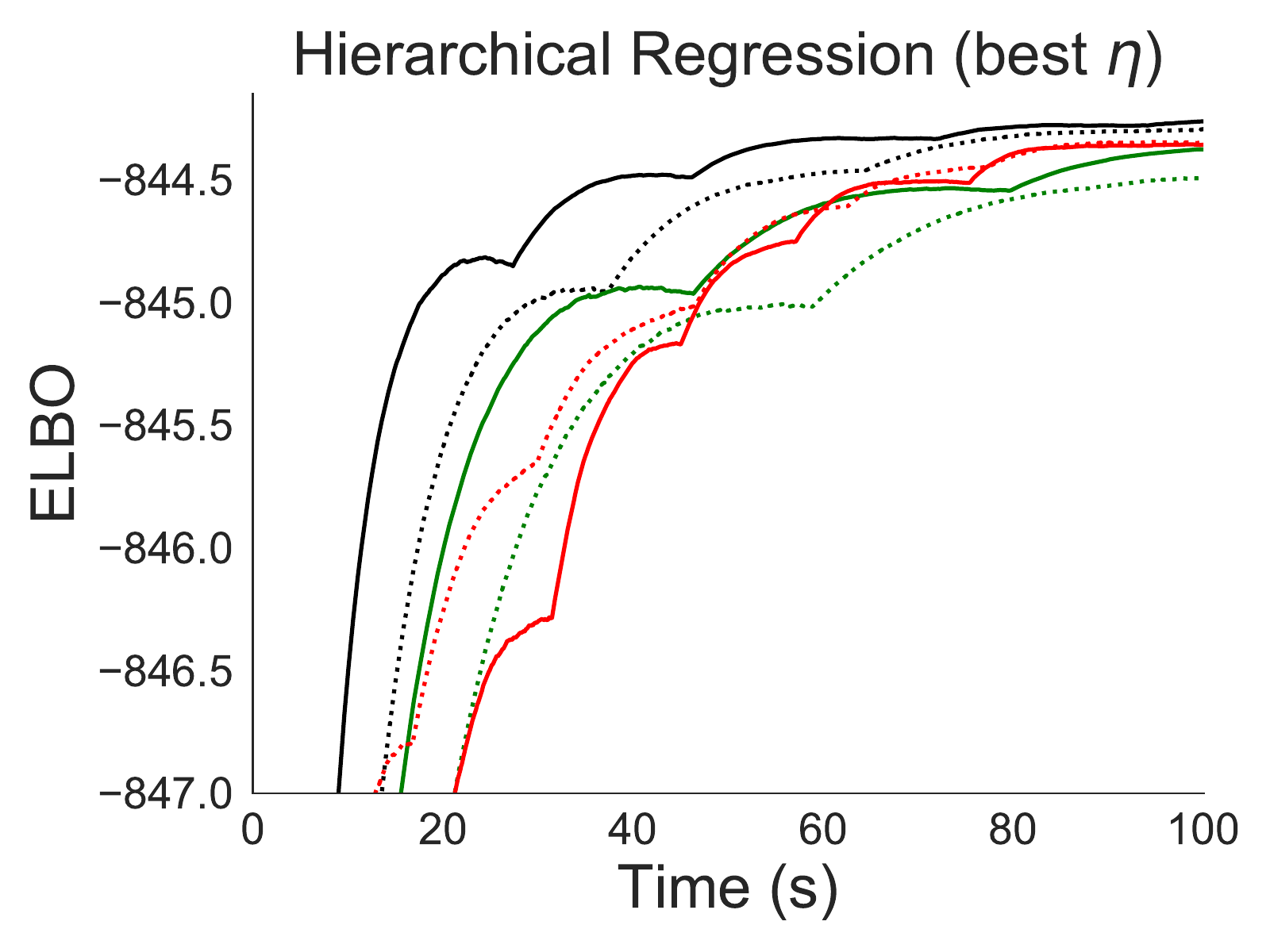}\hfill
    \includegraphics[scale=0.26,trim={1.3cm 0 0 0},clip]{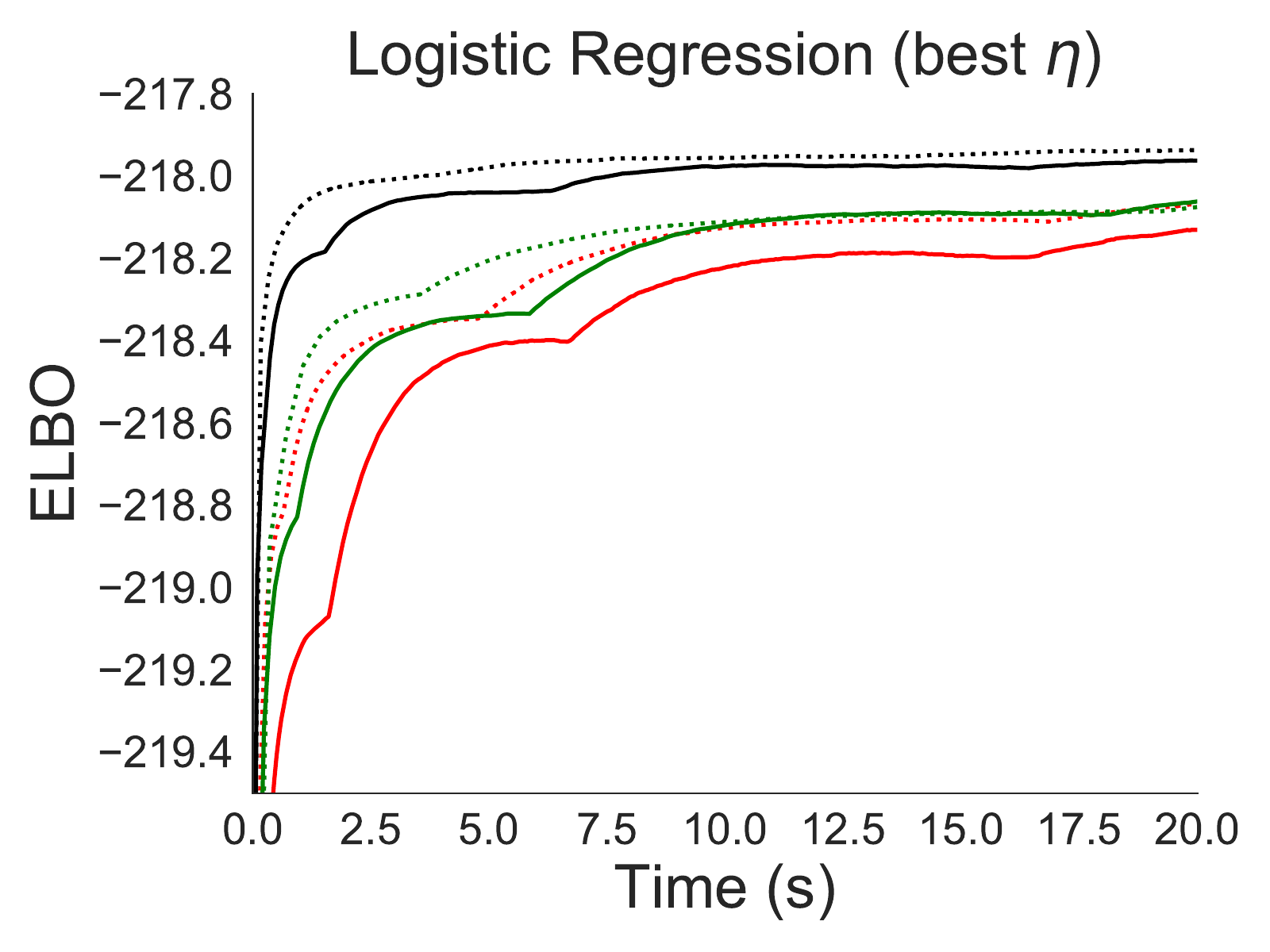}\hfill
    \includegraphics[scale=0.26,trim={1.3cm 0 0 0},clip]{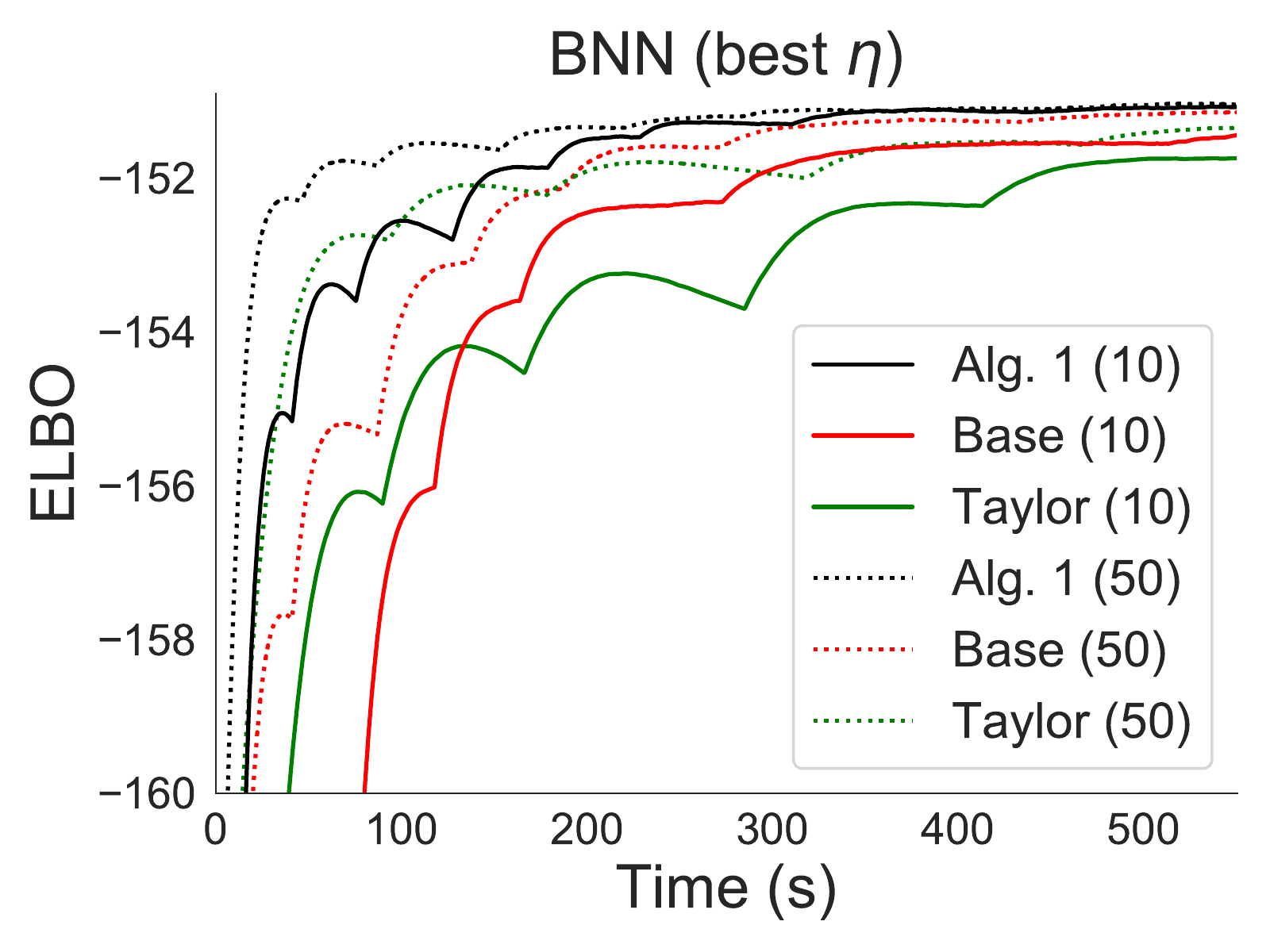}
    \caption{VI using a diagonal Gaussian, with the best step-size chosen retrospectively. (Higher ELBO is better.)}
    \label{fig:opt3time}
    \end{center}
\end{figure}

\clearpage
\newpage

\section{Results for Other Ranks} \label{sec:otherranks}

Fig.~\ref{fig:rankss} shows results obtained using different values for the control variate's rank $r_v$. For clarity, in all cases we use $M=10$ and we do not include results obtained using the Taylor expansion based control variate. It can be observed that the control variate leads to improved performance for a wide range of ranks. However, using a rank that is too low may hinder its benefits considerably (this can be clearly seen for the logistic regression model).

\begin{figure}[ht!]
    \begin{center}
    \includegraphics[scale=0.28,trim={0 0 0 0},clip]{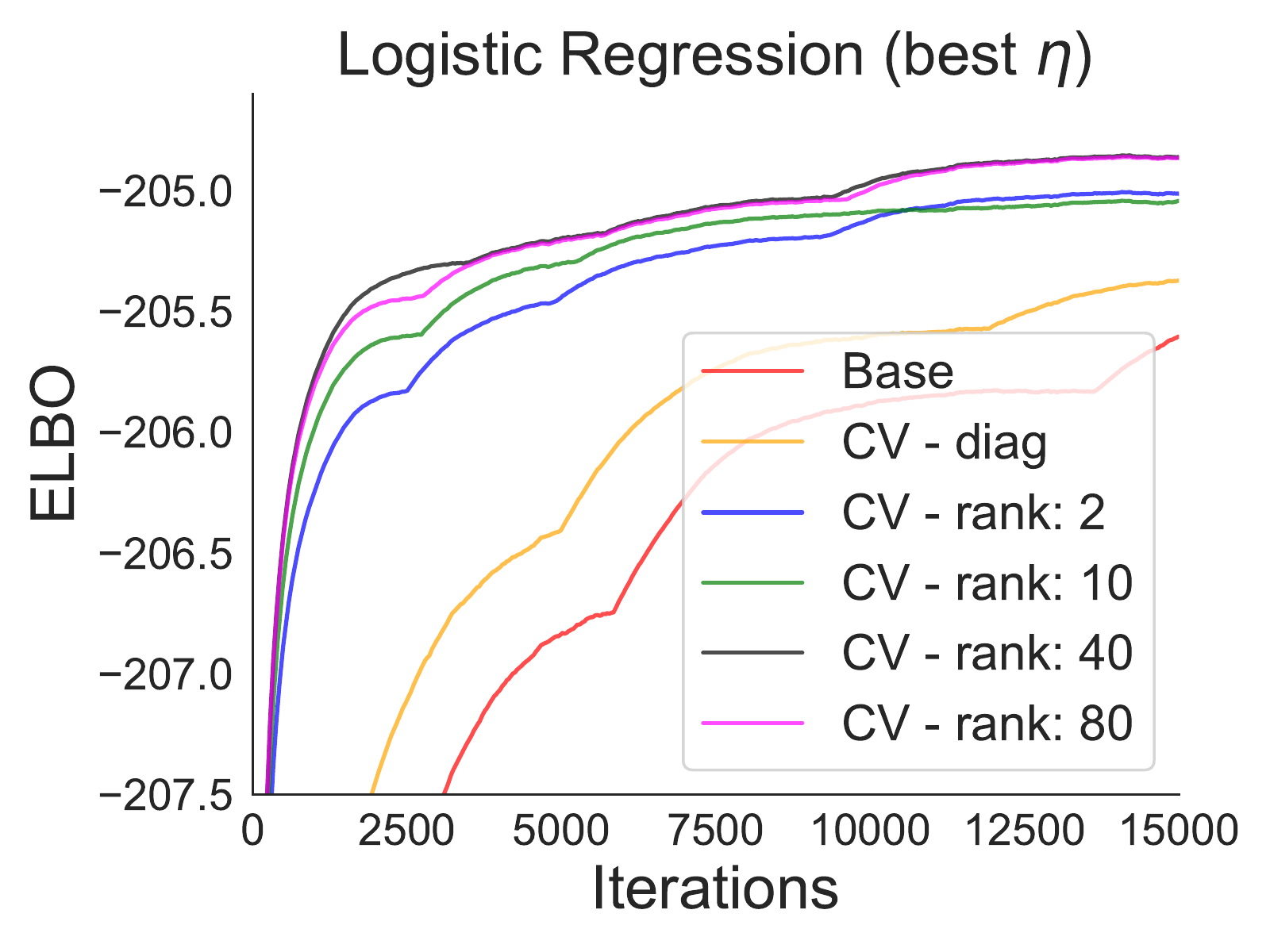}\hspace{0.3cm}
    \includegraphics[scale=0.28,trim={3cm 0 0 0},clip]{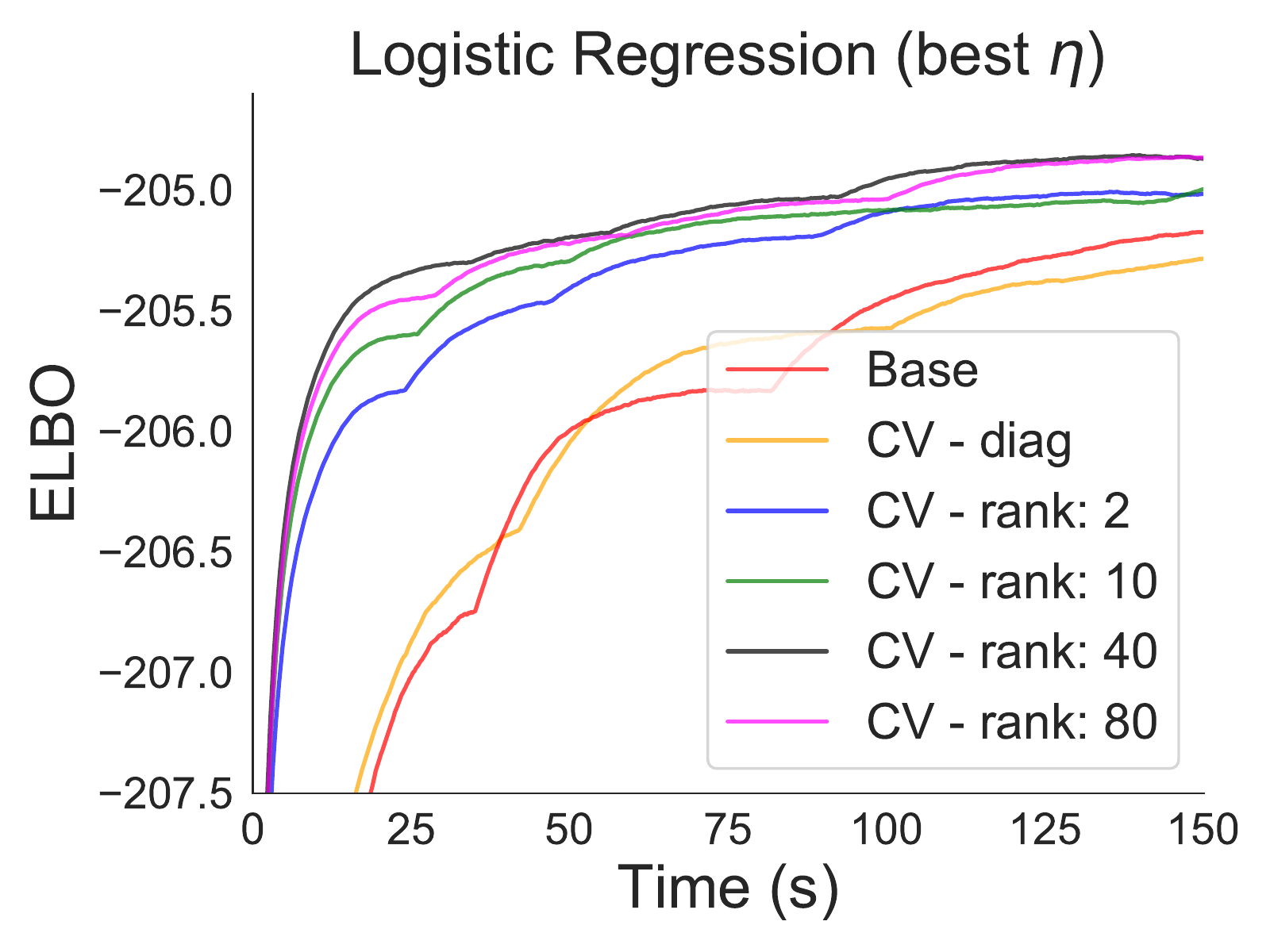}
    
    \vspace{0.5cm}
    \includegraphics[scale=0.28,trim={0 0 0 0},clip]{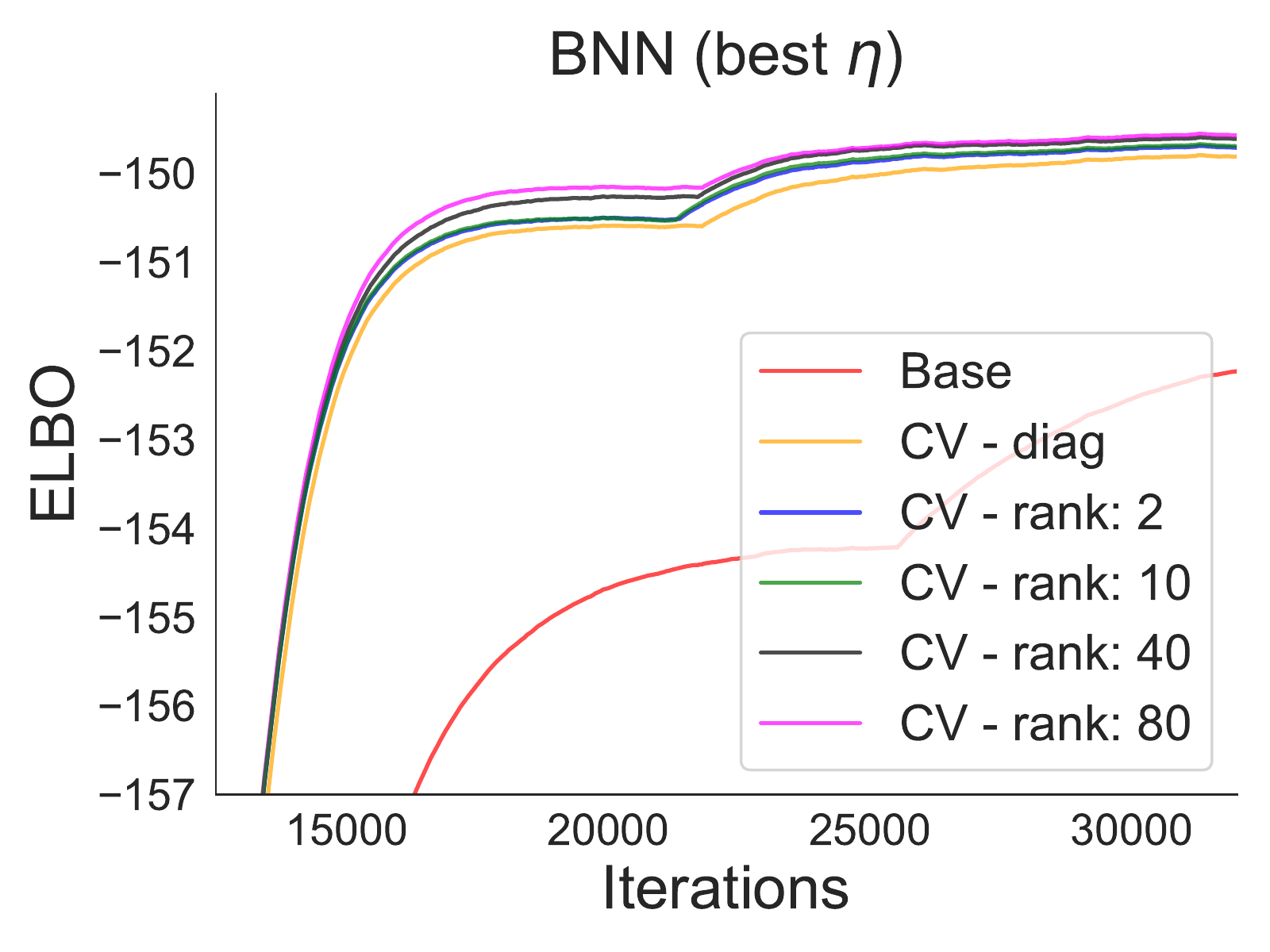}\hspace{0.3cm}
    \includegraphics[scale=0.28,trim={2.5cm 0 0 0},clip]{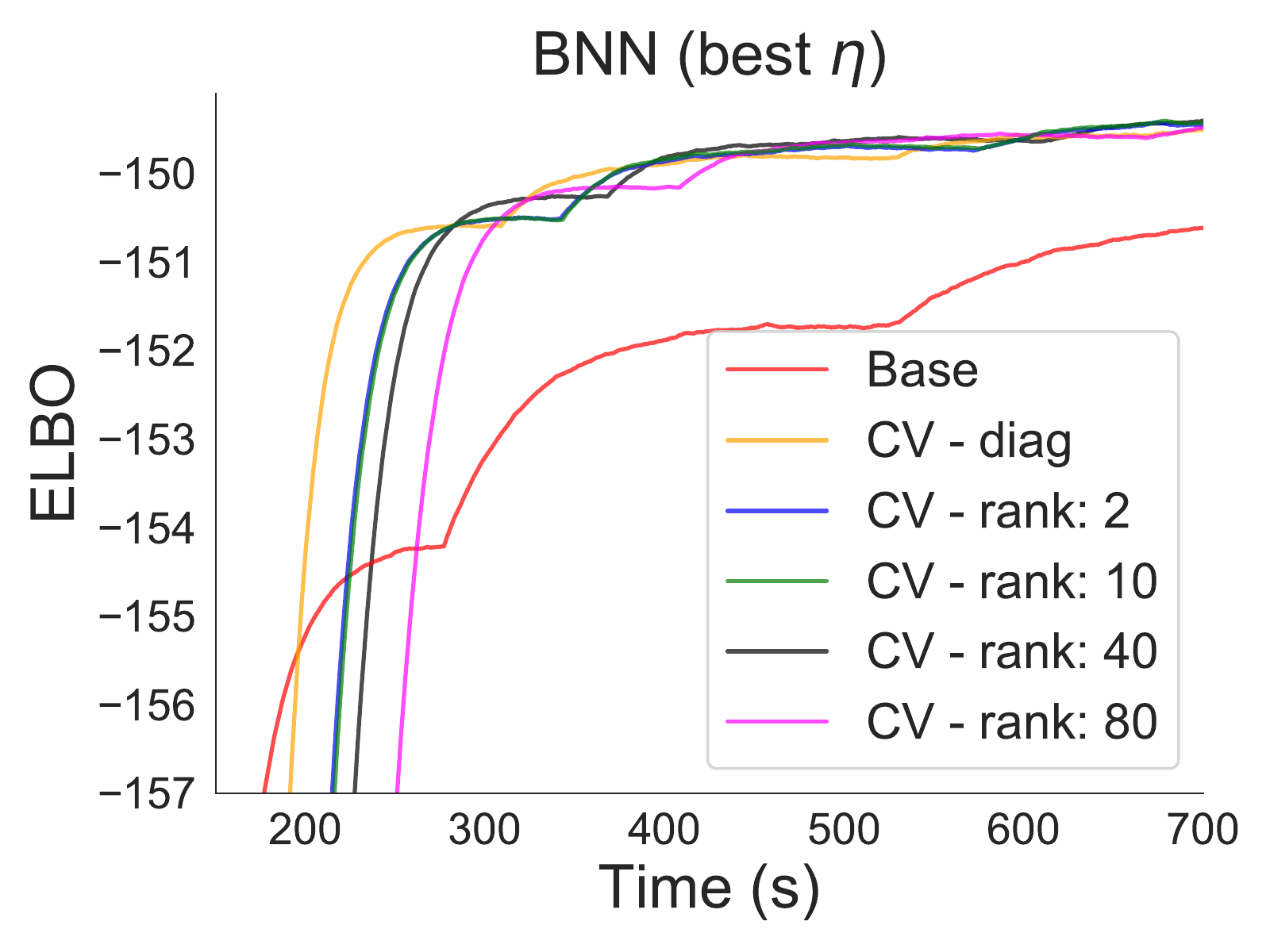}
    \caption{VI using a diagonal plus low rank Gaussian, using different ranks for our control variate.}
    \label{fig:rankss}
    \end{center}
\end{figure}

\section{Models Used} \label{app:models}

\textbf{Bayesian logistic regression:} We use a subset of $700$ rows of the \textit{a1a} dataset. In this case the posterior $p(z|x)$ has dimensionality $d = 120$. Let $\{x_i, y_i\}$, where $y_i$ is binary, represent the \textit{i}-th sample in the dataset. The model is given by
\begin{align*}
w_i & \sim \mathcal{N}(0, 1),\\
p_i & = \left(1 + \exp (w_0 + w \cdot x_i)\right)^{-1},\\
y_i & \sim \mathrm{Bernoulli}(p_i).
\end{align*}


\textbf{Hierarchical Poisson model:} By Gelman et al. \cite{frisk}. The model measures the relative stop-and-frisk events in different precincts in New York city, for different ethnicities. In this case the posterior $p(z|x)$ has dimensionality $d = 37$. The model is given by
\begin{align*}
\mu & \sim \mathcal{N}(0, 10^2)\\
\log \sigma_\alpha & \sim \mathcal{N}(0, 10^2),\\
\log \sigma_\beta & \sim \mathcal{N}(0, 10^2),\\
\alpha_e & \sim \mathcal{N}(0, \sigma_\alpha^2),\\
\beta_p & \sim \mathcal{N}(0, \sigma_\beta^2),\\
\lambda_{ep} & = \exp(\mu + \alpha_e + \beta_p + \log N_{ep}),\\
Y_{ep} & \sim \mathrm{Poisson}(\lambda_{ep}).
\end{align*}

Here, $e$ stands for ethnicity, $p$ for precinct, $Y_{ep}$ for the number of stops in precinct $p$ within ethnicity group $e$ (observed), and $N_{ep}$ for the total number of arrests in precinct $p$ within ethnicity group $e$ (which is observed).

\textbf{Bayesian neural network:} As done by Miller et al. \cite{traylorreducevariance_adam} we use a subset of 100 rows from the ``Red-wine'' dataset. We implement a neural network with one hidden layer with 50 units and Relu activations. In this case the posterior $p(z|x)$ has dimensionality $d = 653$. Let $\{x_i, y_i\}$, where $y_i$ is an integer between one and ten, represent the \textit{i}-th sample in the dataset. The model is given by
\begin{align*}
\log \alpha & \sim \mathrm{Gamma}(1, 0.1),\\
\log \tau & \sim \mathrm{Gamma}(1, 0.1),\\
w_i & \sim \mathcal{N}(0, 1/\alpha), & \mbox{(weights and biases)}\\
\hat{y}_i & = \mathrm{FeedForward}(x_i, W),\\
y_i & \sim \mathcal{N}(\hat{y}_i, 1/\tau).
\end{align*}

\section{Proof of Lemma} \label{app:proof}

\begin{lemma}
Let $\hat f(z)$ be defined as in Eq. \ref{eq:fhat}. If $q_w(z)$ is a distribution with mean $\mu_w$ and covariance matrix $\Sigma_w$, then
\begin{equation}
\E_{q_w(\rz)} \hat f_v(\rz) = b_v^\top(\mu_w - z_0) + \frac{1}{2} \mathrm{tr}(B_v \Sigma_w)
+ \frac{1}{2}\big(\mu_w^\top B_v \mu_w - z_0^\top B_v \mu_w - \mu_w^\top B_v z_0 + z_0^\top B_v z_0 \big)
\end{equation}
\end{lemma}

\begin{proof}
We have 

\[\hat f(z) = b^\top (z-z_0) + \frac{1}{2} (z - z_0)^\top B (z - z_0).\]

Taking the expectation with respect to $q_w(z)$ gives
\begin{equation}
\E_{q_w(z)} \hat f(z) = b^\top(\mu_w - z_0)
+ \frac{1}{2} \underbrace{\E_{q_w(z)}[(z - z_0)^\top B (z - z_0)}_{t(w)}] \label{eq:proof1}
\end{equation}

We now deal with the term in the second line of Eq. \ref{eq:proof1}, $t(w)$.
\begin{align*}
t(w) & = \E_{q_w(z)}[(z - z_0)^\top B (z - z_0)]\\
& = \E_{q_w(z)}[\mathrm{tr}\left((z - z_0)^\top B (z - z_0)\right)]\\
& = \E_{q_w(z)}[\mathrm{tr}\left(B (z - z_0) (z - z_0)^\top\right)]\\
& = \mathrm{tr} \left( B \E_{q_w(z)}[(z - z_0) (z - z_0)^\top]\right)\\
& = \mathrm{tr} \left( B \E_{q_w(z)}[z z^\top - z z_0^\top - z_0 z^\top + z_0 z_0^\top]\right)\\
& = \mathrm{tr} \left( B \E_{q_w(z)}[z z^\top - z z_0^\top - z_0 z^\top + z_0 z_0^\top]\right)\\
& = \mathrm{tr} \left( B \E[(z-\mu_w + \mu_w) (z-\mu_w+\mu_w)^\top - z z_0^\top - z_0 z^\top + z_0 z_0^\top]\right)\\
& = \mathrm{tr} \left( B \left(\E[(z-\mu_w) (z-\mu_w)^\top] + \mu_w \mu_w^\top - \mu_w z_0^\top - z_0 \mu_w ^\top + z_0 z_0^\top \right) \right)\\
& = \mathrm{tr} \left( B \left(\Sigma_w + \mu_w \mu_w^\top - \mu_w z_0^\top - z_0 \mu_w ^\top + z_0 z_0^\top \right) \right)\\
& = \mathrm{tr} \left( B \Sigma_w\right) + \mu_w^\top B \mu_w^\top - z_0^\top B \mu_w - \mu_w^\top B z_0 + z_0^\top B z_0.
\end{align*}

Combining Eq. \ref{eq:proof1} with the expression for $t(w)$ completes the proof.
\end{proof}

\newpage

\section{Details on Taylor-based Control Variates} \label{app:miller}

There is closely related work exploring Taylor-expansion based control variates for reparameterization gradients by Miller et al. \cite{traylorreducevariance_adam}. They develop a control variate for the case where $q_w$ is a fully-factorized Gaussian.

\textbf{Note:} In their paper, Miller et al. derived a control variate for the case where $q_w$ is a fully-factorized Gaussian parameterized by its mean $\mu = [\mu_1, \hdots, \mu_d]$ and standard deviation $\sigma = [\sigma_1, \hdots, \sigma_d]$ ($w = \{\mu, \sigma\}$). That is, $q_w(z) = \mathcal{N}(z|\mu, \mathrm{diag}(\sigma^2))$. However, in their code (publicly available) they use a different parameterization. Instead of using $\sigma$, they use a different set of parameters, $\psi$, to represent the log of the standard deviation of $q_w$. That is, $q_w(z) = \mathcal{N}(z|\mu, \mathrm{diag}(e^{2\psi}))$. In order to explain, replicate and compare against the method they use, we derive the details of their approach for the latter case. (This derivation is not present in their paper, but follows all the steps closely.)

Miller et al. introduced a control variate to reduce the variance of the estimator of the gradient with respect to the mean parameters $\mu$ and a control variate to reduce the variance of the estimator of the gradient with respect to the log-scale parameters $\psi$. We will denote these control variates $c_\mu(w, \epsilon)$ and $c_\psi(w, \epsilon)$, respectively. Their main idea is to use curvature information about the model (via its Hessian) to construct both control variates. The control variate they propose for the mean parameters $c_\mu(w, \epsilon)$ can be computed efficiently via Hessian-vector products. On the other hand, the original proposal for $c_\psi(w, \epsilon)$ requires computing the (often) intractable Hessian $\hess$. To avoid this the authors propose an alternative control variate $\tilde c_\psi(w, \epsilon)$ based on some tractable approximations.

The authors noted that the use of these approximations lead to a significant deterioration of the control variate's variance reduction capability. However, no formal analysis that explained this was presented. We study these approximations in detail and explain exactly why this quality reduction is observed. Simply put, we observe that these approximations lead to a control variate that does not use curvature information about the model at all.

The rest of this section is organized as follows. In \ref{app:finalmiller}, we present the resulting control variates obtained after applying the required approximations to deal with the intractable Hessian: $c_\mu(w, \epsilon)$ and $\tilde c_\psi(w, \epsilon)$. In \ref{app:originalmiller}, we present Miller et al. original (intractable) control variate, $c_\psi(w, \epsilon)$, explain the source of intractability, and explain how the approximation used leads to the "weaker" control variate $\tilde c_\mu(w, \epsilon)$ presented in \ref{app:finalmiller}. Finally, in \ref{app:millerdrawbacks} we describe the drawbacks of the approach, and extend the approach to the case where $q_w$ is a Gaussian with a full-rank or diagonal plus low rank covariance matrix.

\subsection{Final control variate after approximations} \label{app:finalmiller}

Let $q_w(z)$ be the variational distribution. The gradient that must be estimated is given by

\begin{align}
\nabla_w \E_{q_w(\rz)} f(\rz) & = \nabla_w \E_{q_0(\repsilon)} f(\Tr)\\
& = \E_{q_0(\repsilon)} \nabla_w f(\Tr)\\
& = \E_{q_0(\repsilon)} \left(\frac{d\, \Tr}{d\,w}\right)^\top \nabla f(\Tr),
\end{align}

where $\nabla f(\T)$ is $\nabla f(z)$ evaluated at $z = \T$. The gradient estimator obtained with a sample $\epsilon \sim q_0$ is given by

\begin{equation}
g(\epsilon) = \left(\frac{d\, \T}{d\,w}\right)^\top \nabla f(\T). \label{eq:mgm}
\end{equation}

Miller et al. \cite{traylorreducevariance_adam} propose to build a control variate using an approximation $\nabla \hat f(z)$ of $\nabla f(z)$. The control variate is given by the difference between the gradient estimator using this approximation and its expectation,

\begin{align}
c(w, \epsilon) & = \left(\frac{d\, \T}{d\,w}\right)^\top \nabla \hat f(\T) - \E_{q_0(\repsilon)} \left(\frac{d\, \T}{d\,w}\right)^\top \nabla \hat f(\T). \label{eq:0m0}
\end{align}

The quality of the control variate directly depends on the quality of the approximation $\nabla \hat f$. If $\nabla \hat f$ is very close to $\nabla f$, the control variate is able to approximate and cancel the estimator's noise. On the other hand, bad approximations lead to a small (or none) reduction in variance.

This idea is applied to fully-factorized Gaussian with parameters $\psi$ representing the log-scale. The reparameterization transformation is given by

\begin{align}
\mathcal{T}_w(\epsilon) = \mu + e^\psi \odot \epsilon,
\end{align}

where $\odot$ is the element-wise product between vectors. The parameters are $w=(\mu,\psi)$. The control variate is derived differently for $\mu$ and $\psi$. We discuss the two cases separately.
 
\textbf{Control variate for $\mu$}. For $\mu$, the authors set $\nabla \hat f(z)$ to be a first order Taylor expansion of the true gradient around $\mu$. That is, $\nabla \hat f(z) = \nabla f(\mu) + \nabla^2 f(\mu) (z - \mu)$, where $\hess$ is the Hessian of $f$ evaluated at $z = \mu$. Then, it is not hard to show that the control variate becomes\footnote{
 To see this, observe that
 
 \begin{align}
\nabla \hat f(\T)
&= \nabla f(\mu) + \nabla^2 f(\mu) (\T - \mu) = \nabla f(\mu) + \nabla^2 f(\mu)( e^\psi \odot \epsilon).
 \end{align}
  
The Jacobian of $\mathcal{T}$ with respect to $\mu$ is $ \frac{d\,\mathcal{T}_w(\epsilon)}{d \, \mu} = I$. Then, we can calculate that
 \begin{align}
 c_\mu(w, \epsilon)
& =  \left(\frac{d\, \T}{d\,\mu}\right)^\top \nabla \hat f(\T) - \E_{q_0(\repsilon)} \left(\frac{d\, \T}{d\,\mu}\right)^\top \nabla \hat f(\T) \\
& = \nabla f(\mu) + \nabla^2 f(\mu)(e^\psi \odot \epsilon) - \E \left[ \nabla f(\mu) + \nabla^2 f(\mu)(e^\psi \odot \epsilon) \right]\\
& = \nabla^2 f(\mu)(e^\psi \odot \epsilon).
 \end{align}
 
 }
 
 \begin{align}
c_\mu(w, \epsilon) &  = \nabla^2 f(\mu)(e^\psi \odot \epsilon).
\end{align}

This control variate can be computed efficiently using Hessian-vector products, and will be effective when the approximation $\nabla \hat f(z)$ is close to $\nabla f(z)$ for $z \sim q_w(z)$.

The following derivation for $\psi$ is different from that given by Miller et al. We show that it is equivalent in Sec. \ref{app:originalmiller}.

\textbf{Control variate for $\psi$}. For $\psi$, it is necessary -- in order to obtain a closed-form expectation -- to use a {\em constant} approximation of the form $\nabla \hat{f}(z) = \nabla f(\mu)$ (using the first order Taylor expansion as for $c_\mu(w, \epsilon)$ leads to intractable terms, see Section \ref{app:originalmiller}). Then, it turns out that the expectation part of the control variate is zero, and so the control variate becomes\footnote{In this case the Jacobian of $\mathcal{T}$ with respect to $\psi$ is
 $ \frac{d\,\mathcal{T}_w(\epsilon)}{d \, \psi} = \mathrm{diag}(e^\psi \odot \epsilon).$ It follows that

 \begin{align}
 \tilde c_\psi(w, \epsilon)
& =  \left(\frac{d\, \T}{d\,\psi}\right)^\top \nabla \hat f(\T) - \E_{q_0(\repsilon)} \left(\frac{d\, \T}{d\,\psi}\right)^\top \nabla \hat f(\T) \\
& =  \mathrm{diag}(e^\psi \odot \epsilon) \nabla f(\mu) - \E_{q_0(\repsilon)}\mathrm{diag}(e^\psi \odot \epsilon) \nabla f(\mu) \\
& =  e^\psi \odot \epsilon \odot  \nabla f(\mu) - \E_{q_0(\repsilon)}e^\psi \odot \epsilon \odot \nabla f(\mu) \\
& =  e^\psi \odot \epsilon \odot  \nabla f(\mu) 
 \end{align}}

\begin{align}
 \tilde c_\psi(w, \epsilon) = e^\psi \odot \epsilon \odot  \nabla f(\mu) \label{eq:millerbad}
\end{align}
 
 It can be observed that $\tilde c_\psi(w, \epsilon)$ does not use curvature information about the model. This control variate will be effective only in cases where $\nabla f(\mu)$ is close to $\nabla f(z)$ for $z \sim q_w(z)$.

\subsection{Original Derivation} \label{app:originalmiller}

Miller et al. \cite{traylorreducevariance_adam} gave a more elaborate derivation of the above control variate for $\psi$. They start with the same first-order Taylor expansion $\nabla \hat f(z) = \nabla f(\mu) + \nabla^2 f(\mu) (z - \mu)$ as used for $\mu$. Applied directly, this suggests the control variate\footnote{Again, $\frac{d\,\mathcal{T}_w(\epsilon)}{d \, \psi} = \mathrm{diag}(e^\psi \odot \epsilon)$ and $\nabla \hat f(\T) = \nabla f(\mu) + \nabla^2 f(\mu)( e^\psi \odot \epsilon)$. We thus have that

\begin{align}
 c_\psi(w, \epsilon)
& = \left(\frac{d\, \T}{d\,\psi}\right)^\top \nabla \hat f(\T) - \E_{q_0(\repsilon)} \left(\frac{d\, \T}{d\,\psi}\right)^\top \nabla \hat f(\T) \\
& = \mathrm{diag}(e^\psi \odot \epsilon) \left( \nabla f(\mu) + \nabla^2 f(\mu)( e^\psi \odot \epsilon) \right)- \E_{q_0(\repsilon)} \mathrm{diag}(e^\psi \odot \epsilon) \left( \nabla f(\mu) + \nabla^2 f(\mu)( e^\psi \odot \epsilon) \right) \\
& = \left( \nabla f(\mu) + \nabla^2 f(\mu)( e^\psi \odot \epsilon) \right) \odot e^\psi \odot \epsilon - \E_{q_0(\repsilon)} \left( \nabla^2 f(\mu)( e^\psi \odot \epsilon) \right) \odot (e^\psi \odot \epsilon) \\
\end{align}
Finally, we can observe that $\E \left[ \left(\nabla f(\mu) + \nabla^2 f(\mu)(e^\psi \odot \epsilon)\right) \odot \epsilon \odot e^s \right] = \mathrm{diag}(\nabla^2 f(\mu)) \odot e^{2\psi}$.
}

\small
\begin{align}
c_\psi(w, \epsilon)
& = \left(\nabla f(\mu) + \nabla^2 f(\mu)(e^\psi \odot \epsilon)\right) \odot \epsilon \odot e^\psi - \underbrace{\E_{q_0(\epsilon)} \left(\nabla f(\mu) + \nabla^2 f(\mu)(e^\psi \odot \epsilon)\right) \odot \epsilon \odot e^\psi}_{\mathrm{diag}(\nabla^2 f(\mu)) \odot e^{2\psi}}. \label{eq:origcpsi}
\end{align}\normalsize

The first term from Eq. \ref{eq:origcpsi} can be computed efficiently using Hessian-vector products. The second term, however, is often intractable, since it requires the diagonal of the Hessian. In such cases, the authors propose to apply a further estimation process to estimate it using a baseline \cite{bekas2007estimator, VIforMCobjectives_mnih}. The idea is that often gradients are estimated in a minibatch, based on a set of samples $\epsilon_1, \hdots, \epsilon_N$. Then, the expectation can be estimated without bias using the other samples in the minibatch. This results in the control variate for sample $i$ of
\small
\begin{align}
c_\psi(w, \epsilon_i) & = \left(\nabla f(\mu) + \nabla^2 f(\mu)(e^\psi \odot \epsilon_i)\right) \odot \epsilon \odot e^\psi 
 - \underbrace{\frac{1}{N - 1} \sum_{\substack{j = 1 \\ j \neq i}}^N \left(\hess (e^\psi \odot \epsilon_j)\right) \odot \epsilon_j \odot e^\psi}_{\mathrm{baseline}}.
\end{align}\normalsize
At a first glance it may appear that this control variate uses curvature information from the model via the Hessian $\hess$. However, a careful inspection shows that all these terms cancel out. The control variate for the full minibatch is simply

\begin{align}
c_\psi(w, \epsilon_1, \cdots, \epsilon_N) &= \sum_{i=1}^N c_\psi(w, \epsilon_i) = \sum_{i=1}^N \nabla f(\mu) \odot \epsilon_i \odot e^\psi.
\end{align}

This, of course, is exactly the same as taking a minibatch of the control variate derived in Eq. \ref{eq:millerbad}. Thus, the ideas of minibatch and baseline may somewhat obscure what is happening. It is not necessary to invoke the machinery of a baseline, nor to draw samples in a minibatch. A zero-th order Taylor expansion is equivalent, and has the practical advantage of remaining valid with a single sample. While some details of the baseline procedure were not available in the published paper, we confirmed this is equivalent to the control variate used in the publicly available code.

\subsection{Limitations of the approach and extensions} \label{app:millerdrawbacks}

One limitation of the above approach is that the control variate for $\psi$ is not very effective. Unless the diagonal of the Hessian is tractable, it uses a very crude approximation for $\nabla f(z)$. Thus, one would naturally expect this control variate to perform worse when the diagonal of the Hessian is not tractable. Indeed, this can be observed in the results obtained by Miller et al. \cite{traylorreducevariance_adam}. Table 1 in their paper shows that the tractable control variate (Eq. \ref{eq:millerbad}, tractable), leads to a variance reduction several orders of magnitude worse than the one obtained using the control variate based on the true Hessian (Eq. \ref{eq:origcpsi}, often intractable to compute).

In their simulations, this relatively poor performance for $\psi$ does not represent a big inconvenience. That is because of the following empirical observation: when using a fully-factorized Gaussian as variational distribution most of the gradient variance comes from mean parameters $\mu$, where a much better approximation of $\nabla f$ can be used. However, our results in this paper show that with non fully-factorized distributions most of the variance is often contributed by the scale parameters (see Fig. \ref{fig:results_underover_s4}).

A second limitation is that their approach requires manual distribution-specific derivations. More specifically, in order to use the control variate with another distribution the expectation

\[ \E \jact^\top \nabla \hat f(\Tr)\]

must be computed. In order to do so, a closed form expression for the Jacobian of $\T$ is required. (One cannot use automatic differentiation for this since a mathematical expression for the Jacobian is needed in order to derive the expectation). Thus, extending the approach to other variational distributions is not trivial, and the difficulty depends on the variational distribution chosen. We now present three cases, two for which the extension can be done without much work (full-rank and diagonal plus low rank Gaussians), and other for which the extension requires extensive calculations (Householder flows \cite{householderflows}).

\textbf{Full-rank Gaussian:} In this case we have $q_w(z) = \mathcal{N}(z|\mu, \Sigma)$. The parameters are $w = (\mu, S)$, where S parameterizes the covariance matrix as $S S^\top = \Sigma$, and reparameterization is given by $z = \mu + S \epsilon$. If we let $\mathrm{vec}(S)$ be a vector that contains all rows of $S$ in order, we get that the required Jacobians are given by

\begin{equation}
\frac{d\,\mathcal{T}_w(\epsilon)}{d\,\mu} = I \,\,\,\, \mbox{ and } \,\,\,\, \frac{d\,\mathcal{T}_w(\epsilon)}{d\,\mathrm{vec}(S)} = 
\left[ \begin{array}{cccc}
\epsilon^\top & 0_d^\top & \hdots & 0_d^\top\\
0_d^\top & \epsilon^\top & \hdots & 0_d^\top\\
& \hdots &  & \\
0_d^\top & 0_d^\top & \hdots & \epsilon^\top\\
\end{array} \right],
\end{equation}

where $\epsilon^\top$ is a row vector of dimension $d$ and $0_d^\top$ is the zero row vector of dimension $d$. The Jacobian $\frac{d\,\mathcal{T}_w(\epsilon)}{d\,\mathrm{vec}(S)}$ has dimension $d \times d^2$. Following section \ref{app:finalmiller} and using the above expressions for the Jacobians we get 

\begin{equation}
c_\mu(w, \epsilon) = \hess S \epsilon \,\,\,\,\,\, \mbox{ and } \,\,\,\,\,\, \tilde c_S(w, \epsilon) = \nabla f(\mu) \epsilon^\top.
\end{equation}

Both $c_\mu(w, \epsilon)$ and $\tilde  c_S(w, \epsilon)$ can be computed efficiently.

\textbf{Diagonal plus low rank Gaussian: } In this case we have $q_w(z) = \mathcal{N}(z|\mu, \Sigma)$. The parameters are $w = (\mu, \psi, U)$, where $\mu$ and $\psi$ are vectors of dimension $d$, and $U$ is a matrix of size $d \times r$. The covariance is parameterized as $\Sigma = \mathrm{diag}(e^{2\psi}) +  U U^\top$. Reparameterization is given by $z = \mu + e^\psi \odot \epsilon_d + U \epsilon_r$, where $\epsilon_d$ and $\epsilon_r$ are independent samples of standard Normal distributions of dimension $d$ and $r$, respectively. In this case the required Jacobians are given by

\small
\begin{equation}
\frac{d\,\mathcal{T}_w(\epsilon_d, \epsilon_r)}{d\,\mu} = I
\,\, , \,\,
\frac{d\,\mathcal{T}_w(\epsilon_d, \epsilon_r)}{d \, \psi} = \mathrm{diag}(e^\psi \odot \epsilon_d)
\,\,\,\, \mbox{ and } \,\,\,\, \frac{d\,\mathcal{T}_w(\epsilon_d, \epsilon_r)}{d\,\mathrm{vec}(U)} = 
\left[ \begin{array}{cccc}
\epsilon_r^\top & 0_r^\top & \hdots & 0_r^\top\\
0_r^\top & \epsilon_r^\top & \hdots & 0_r^\top\\
& \hdots &  & \\
0_r^\top & 0_r^\top & \hdots & \epsilon_r^\top\\
\end{array} \right].
\end{equation}\normalsize

Following section \ref{app:finalmiller} and using the above expressions for the Jacobians we get 

\begin{align}
c_\mu(w, \epsilon_d, \epsilon_r) & = \hess (e^\psi \odot \epsilon_d + U \epsilon_r)\\
\tilde c_\psi(w, \epsilon_d, \epsilon_r) & = \nabla f(\mu) \odot e^\psi \odot \epsilon_d\\
\tilde c_U(w, \epsilon_d, \epsilon_r) & = \nabla f(\mu) \epsilon_r^\top.
\end{align}

\textbf{Householder flows:} In this case we have a Gaussian distribution with reparameterization given by $z = \mu + \prod_{i=1}^M H(v_i) D \epsilon$, where $M$ is the number of flow steps used, $D = \mathrm{diag}(\sigma)$ is a diagonal matrix, and $H_i$ is a matrix parameterized by vector $v_i$ as $H_i(v_i) = \left(I - 2\frac{v_i v_i^\top}{\Vert v_i \Vert^2} \right)$. The parameter set is given by $w = \{\mu, \sigma, v_1, \hdots, v_M\}$. In this case, computing the Jacobians required to apply Miller et al. approach is quite complex, because of the complex dependency of $\mathcal{T}_w$ on the parameters $v_i$.

%% file: ms.bbl
\begin{thebibliography}{10}

\bibitem{bekas2007estimator}
Costas Bekas, Effrosyni Kokiopoulou, and Yousef Saad.
\newblock An estimator for the diagonal of a matrix.
\newblock {\em Applied numerical mathematics}, 57(11-12):1214--1229, 2007.

\bibitem{blei2017variational}
David~M Blei, Alp Kucukelbir, and Jon~D McAuliffe.
\newblock Variational inference: A review for statisticians.
\newblock {\em Journal of the American Statistical Association},
  112(518):859--877, 2017.

\bibitem{boustati2020amortized}
Ayman Boustati, Sattar Vakili, James Hensman, and ST~John.
\newblock Amortized variance reduction for doubly stochastic objectives.
\newblock {\em arXiv preprint arXiv:2003.04125}, 2020.

\bibitem{stan}
Bob Carpenter, Andrew Gelman, Matthew~D Hoffman, Daniel Lee, Ben Goodrich,
  Michael Betancourt, Marcus Brubaker, Jiqiang Guo, Peter Li, and Allen
  Riddell.
\newblock Stan: A probabilistic programming language.
\newblock {\em Journal of statistical software}, 76(1), 2017.

\bibitem{chaloner1995bayesian}
Kathryn Chaloner and Isabella Verdinelli.
\newblock Bayesian experimental design: A review.
\newblock {\em Statistical Science}, pages 273--304, 1995.

\bibitem{cvs}
Tomas Geffner and Justin Domke.
\newblock Using large ensembles of control variates for variational inference.
\newblock In {\em Advances in Neural Information Processing Systems}, pages
  9960--9970, 2018.

\bibitem{frisk}
Andrew Gelman, Jeffrey Fagan, and Alex Kiss.
\newblock An analysis of the new york city police department's
  ``stop-and-frisk'' policy in the context of claims of racial bias.
\newblock {\em Journal of the American Statistical Association},
  102(479):813--823, 2007.

\bibitem{glasserman2013monte}
Paul Glasserman.
\newblock {\em Monte Carlo methods in financial engineering}, volume~53.
\newblock Springer Science \& Business Media, 2013.

\bibitem{backpropvoid}
Will Grathwohl, Dami Choi, Yuhuai Wu, Geoff Roeder, and David Duvenaud.
\newblock Backpropagation through the void: Optimizing control variates for
  black-box gradient estimation.
\newblock In {\em Proceedings of the International Conference on Learning
  Representations}, 2018.

\bibitem{jaakkola2000bayesian}
Tommi~S Jaakkola and Michael~I Jordan.
\newblock Bayesian parameter estimation via variational methods.
\newblock {\em Statistics and Computing}, 10(1):25--37, 2000.

\bibitem{gumbel1}
Eric Jang, Shixiang Gu, and Ben Poole.
\newblock Categorical reparameterization with gumbel-softmax.
\newblock {\em arXiv preprint arXiv:1611.01144}, 2016.

\bibitem{jordan1999introduction}
Michael~I Jordan, Zoubin Ghahramani, Tommi~S Jaakkola, and Lawrence~K Saul.
\newblock An introduction to variational methods for graphical models.
\newblock {\em Machine learning}, 37(2):183--233, 1999.

\bibitem{adam}
Diederik~P Kingma and Jimmy Ba.
\newblock Adam: A method for stochastic optimization.
\newblock {\em arXiv preprint arXiv:1412.6980}, 2014.

\bibitem{vaes_welling}
Diederik~P Kingma and Max Welling.
\newblock Auto-encoding variational bayes.
\newblock In {\em Proceedings of the International Conference on Learning
  Representations}, 2013.

\bibitem{advi}
Alp Kucukelbir, Dustin Tran, Rajesh Ranganath, Andrew Gelman, and David~M Blei.
\newblock Automatic differentiation variational inference.
\newblock {\em The Journal of Machine Learning Research}, 18(1):430--474, 2017.

\bibitem{gumbel2}
Chris~J Maddison, Andriy Mnih, and Yee~Whye Teh.
\newblock The concrete distribution: A continuous relaxation of discrete random
  variables.
\newblock {\em arXiv preprint arXiv:1611.00712}, 2016.

\bibitem{traylorreducevariance_adam}
Andrew Miller, Nick Foti, Alexander D'Amour, and Ryan~P Adams.
\newblock Reducing reparameterization gradient variance.
\newblock In {\em Advances in Neural Information Processing Systems}, pages
  3708--3718, 2017.

\bibitem{neuralVI_minh}
Andriy Mnih and Karol Gregor.
\newblock Neural variational inference and learning in belief networks.
\newblock In {\em International Conference on Machine Learning}, 2014.

\bibitem{VIforMCobjectives_mnih}
Andriy Mnih and Danilo Rezende.
\newblock Variational inference for monte carlo objectives.
\newblock In {\em International Conference on Machine Learning}, pages
  2188--2196, 2016.

\bibitem{mohamed2019monte}
Shakir Mohamed, Mihaela Rosca, Michael Figurnov, and Andriy Mnih.
\newblock Monte carlo gradient estimation in machine learning.
\newblock {\em arXiv preprint arXiv:1906.10652}, 2019.

\bibitem{diagLR}
Victor M-H Ong, David~J Nott, and Michael~S Smith.
\newblock Gaussian variational approximation with a factor covariance
  structure.
\newblock {\em Journal of Computational and Graphical Statistics},
  27(3):465--478, 2018.

\bibitem{mcbook}
Art~B. Owen.
\newblock {\em Monte Carlo theory, methods and examples}.
\newblock 2013.

\bibitem{viasstochastic_jordan}
John Paisley, David Blei, and Michael Jordan.
\newblock Variational bayesian inference with stochastic search.
\newblock In {\em Proceedings of the 29th International Conference on Machine
  Learning (ICML-12)}, pages 1363--1370, 2012.

\bibitem{pflug2012optimization}
Georg~Ch Pflug.
\newblock {\em Optimization of stochastic models: the interface between
  simulation and optimization}, volume 373.
\newblock Springer Science \& Business Media, 2012.

\bibitem{blackbox_blei}
Rajesh Ranganath, Sean Gerrish, and David Blei.
\newblock Black box variational inference.
\newblock In {\em Artificial Intelligence and Statistics}, pages 814--822,
  2014.

\bibitem{rezende2014stochastic}
Danilo~Jimenez Rezende, Shakir Mohamed, and Daan Wierstra.
\newblock Stochastic backpropagation and approximate inference in deep
  generative models.
\newblock In {\em Proceedings of the 31st International Conference on Machine
  Learning (ICML-14)}, pages 1278--1286, 2014.

\bibitem{stickingthelanding}
Geoffrey Roeder, Yuhuai Wu, and David~K Duvenaud.
\newblock Sticking the landing: Simple, lower-variance gradient estimators for
  variational inference.
\newblock In {\em Advances in Neural Information Processing Systems}, pages
  6925--6934, 2017.

\bibitem{generalreparam_blei}
Francisco Ruiz, Titsias Michalis, and David Blei.
\newblock The generalized reparameterization gradient.
\newblock In {\em Advances in Neural Information Processing Systems}, pages
  460--468, 2016.

\bibitem{overdispersed_blei}
Francisco~JR Ruiz, Michalis~K Titsias, and David~M Blei.
\newblock Overdispersed black-box variational inference.
\newblock {\em arXiv preprint arXiv:1603.01140}, 2016.

\bibitem{sutton2018reinforcement}
Richard~S Sutton and Andrew~G Barto.
\newblock {\em Reinforcement learning: An introduction}.
\newblock MIT press, 2018.

\bibitem{doublystochastic_titsias}
Michalis Titsias and Miguel L{\'a}zaro-Gredilla.
\newblock Doubly stochastic variational bayes for non-conjugate inference.
\newblock In {\em Proceedings of the 31st International Conference on Machine
  Learning (ICML-14)}, pages 1971--1979, 2014.

\bibitem{householderflows}
Jakub~M Tomczak and Max Welling.
\newblock Improving variational auto-encoders using householder flow.
\newblock {\em arXiv preprint arXiv:1611.09630}, 2016.

\bibitem{REBAR}
George Tucker, Andriy Mnih, Chris~J Maddison, John Lawson, and Jascha
  Sohl-Dickstein.
\newblock Rebar: Low-variance, unbiased gradient estimates for discrete latent
  variable models.
\newblock In {\em Advances in Neural Information Processing Systems}, pages
  2627--2636, 2017.

\bibitem{williams_reinforce}
Ronald~J Williams.
\newblock Simple statistical gradient-following algorithms for connectionist
  reinforcement learning.
\newblock {\em Machine learning}, 8(3-4):229--256, 1992.

\bibitem{zhang2017advances}
Cheng Zhang, Judith Butepage, Hedvig Kjellstrom, and Stephan Mandt.
\newblock Advances in variational inference.
\newblock {\em arXiv preprint arXiv:1711.05597}, 2017.

\end{thebibliography}
